\documentclass{article} 

\usepackage{iclr2027_conference,times}

\usepackage[T1]{fontenc}
\usepackage[utf8]{inputenc}
\usepackage{amsmath, amssymb, amsthm, mathtools}
\usepackage{bm}
\usepackage{booktabs}
\usepackage{tabularx}
\usepackage{graphicx}
\usepackage{enumitem}
\usepackage{xcolor}
\usepackage{aliascnt}
\usepackage{hyperref}
\usepackage{url}
\usepackage[nameinlink,capitalize,noabbrev]{cleveref}

\graphicspath{%
  {./}%
  {./figures/}%
  {./qualitative_images/}%
  {./qualitative_images/pickscore/}%
  {./qualitative_images/geneval/}%
}

\hypersetup{
  colorlinks=true,
  linkcolor=blue!60!black,
  urlcolor=blue!60!black,
  citecolor=blue!60!black,
}

\numberwithin{equation}{section}

\newaliascnt{proposition}{theorem}
\newtheorem{proposition}[proposition]{Proposition}
\aliascntresetthe{proposition}

\newaliascnt{lemma}{theorem}

\aliascntresetthe{lemma}

\newaliascnt{corollary}{theorem}
\newtheorem{corollary}[corollary]{Corollary}
\aliascntresetthe{corollary}

\theoremstyle{definition}
\newaliascnt{definition}{theorem}

\aliascntresetthe{definition}

\newaliascnt{assumption}{theorem}
\newtheorem{assumption}[assumption]{Assumption}
\aliascntresetthe{assumption}

\newaliascnt{remark}{theorem}

\aliascntresetthe{remark}

\crefname{theorem}{Theorem}{Theorems}
\Crefname{theorem}{Theorem}{Theorems}
\crefname{proposition}{Proposition}{Propositions}
\Crefname{proposition}{Proposition}{Propositions}
\crefname{lemma}{Lemma}{Lemmas}
\Crefname{lemma}{Lemma}{Lemmas}
\crefname{corollary}{Corollary}{Corollaries}
\Crefname{corollary}{Corollary}{Corollaries}
\crefname{definition}{Definition}{Definitions}
\Crefname{definition}{Definition}{Definitions}
\crefname{assumption}{Assumption}{Assumptions}
\Crefname{assumption}{Assumption}{Assumptions}
\crefname{remark}{Remark}{Remarks}
\Crefname{remark}{Remark}{Remarks}

\newcommand{\E}{\mathbb{E}}

\newcommand{\dd}{\mathrm{d}}
\newcommand{\old}{\mathrm{old}}
\newcommand{\data}{\mathrm{data}}
\newcommand{\clip}{\mathrm{clip}}

\newcommand{\Var}{\mathrm{Var}}
\newcommand{\med}{\mathrm{median}}
\newcommand{\norm}[1]{\left\lVert #1\right\rVert}
\newcommand{\inner}[2]{\left\langle #1, #2 \right\rangle}

\title{$\lambda$-Controlled GRPO: Turning Flow-Matching Ratio Instability into a Budgeted Resource}

\author{
Yufeng Wang \thanks{These authors contributed equally to this work.}\\
Stony Brook University \\
\And
Parivesh Priye $^*$\thanks{Corresponding Author.} \\
Georgia Institute of Technology \\
\And
Meeshawn Marathe\\
Rivian and Volkswagen Group Technologies \\
\And
Ramit Pahwa \\
Rivian and Volkswagen Group Technologies \\}

\iclrfinalcopy
\makeatletter
\g@addto@macro\@maketitle{\lhead{}}
\makeatother

\begin{document}
\maketitle

\begin{abstract}
Reinforcement learning is increasingly used to align image generators with reward signals, and Flow-GRPO recently extended this paradigm to flow-matching models by treating the denoising sampler as a stochastic policy that can be optimized from reward feedback. Training in this setting is unstable in a way specific to multi-step denoising: the policy update changes systematically across denoising steps, with importance ratios drifting below one, becoming increasingly dispersed, clipping at different rates, and leaving fewer usable samples late in training. Prior work treats these effects as separate failure modes and addresses each with a hand-tuned stabilizer. We show instead that they arise from a single per-step quantity, which we call \emph{path variance}. This quantity is determined exactly by the sampler's Gaussian transition kernel and can be estimated cheaply during training. This reframes instability as a resource that can be measured and budgeted rather than a collection of symptoms to repair. Our method, $\lambda$-Controlled GRPO, calibrates importance-ratio behavior from this predicted law rather than from noisy empirical statistics, and allocates gradient effort across denoising steps according to their predicted cost. The two scales governing the update are fixed by standard policy choices rather than introduced as free tuning parameters. On a text-to-image model under two reward settings, rendering difficult target text scored by optical character recognition and matching human preferences scored by a preference model, $\lambda$-Controlled GRPO improves both text accuracy and preference reward over the strongest empirical stabilizer. It also keeps late-step path variance within its intended budget, precisely where the baseline systematically overshoots. The result is a Flow-GRPO update calibrated by its own transition law rather than stabilized after instability appears.
\end{abstract}

\section{Introduction}
\label{sec:intro}

Reinforcement learning has become a standard tool for aligning generative models with reward signals, and has recently extended from language models to the flow-matching and rectified-flow models used for image generation. The key step, Flow-GRPO, converts a deterministic flow-matching sampler into a stochastic policy whose denoising trajectory has a tractable per-step likelihood, enabling policy-gradient updates from the group relative policy optimization (GRPO) family. This makes online reward optimization possible for image generators, but introduces a stability problem specific to multi-step denoising: the policy importance ratio, which measures how much more or less likely the new policy is to take a sampled step than the old one, behaves very differently across denoising timesteps. Recent work reports that ratios shift below one at some steps, become highly dispersed at others, clip asymmetrically in the surrogate objective, and eventually leave few effective samples late in training. These effects have largely been treated as separate empirical symptoms and addressed with hand-designed stabilizers, such as timestep-wise normalization of observed ratios and manually chosen gradient reweighting. What is missing is a common explanation for why these pathologies emerge and which quantity an update should control. We identify such a quantity directly from the finite-grid Gaussian transition kernel used by the sampler. A single per-step scalar, which we call the \emph{path variance},
$$
\lambda_k = \left\lVert \sigma_k^{-1}(b_\theta-b_\old)(X_k,t_k)\right\rVert^2\Delta t_k,
$$
determines the conditional mean and variance of the step log-ratio, and therefore predicts its drift, dispersion, left shift, clipping imbalance, and loss of usable samples within one law. The same scalar is not only diagnostic. It defines the natural cost of a policy update at each denoising step and therefore provides a principled quantity for allocating gradient effort across the trajectory.

A practical subtlety arises because real implementations typically store a mean-reduced log-probability rather than the full coordinate sum. This reduction splits the path-variance effect into two related quantities: a centering term that governs the mean shift of the log-ratio and a variance term that governs its spread. Keeping these roles separate is essential for the theoretical prediction to match the statistics measured during training.

This paper makes four contributions. First, we derive the finite-grid importance-ratio law and its mean-reduced corollary, obtaining an exact per-step prediction for the log-ratio statistics from the path variance. Second, we give a cheap online estimator of this quantity using values already produced by the sampler. Third, we turn the prediction into an algorithm, $\lambda$-Controlled GRPO, which replaces empirical ratio normalization with analytic calibration and allocates gradient effort according to predicted path-variance cost, with both governing scales determined by standard policy choices rather than introduced as free hyperparameters. Fourth, across two reward regimes for a text-to-image model, we show that the analytic update improves task quality over the strongest empirical stabilizer while keeping late-step path-variance spend near its intended budget. Throughout, we first verify that the estimated scalar predicts the measured log-ratio statistics before modifying the optimizer, so the resulting intervention is grounded in a confirmed transition-law prediction rather than an assumed training heuristic.

\section{Related work}
\label{sec:related-work}

\textbf{Reinforcement learning for flow-matching models.} Reinforcement learning has become a standard tool for aligning generative models with rewards, beginning in language with reward models learned from human preferences~\citep{christiano2017,instructgpt} and policy or preference optimization methods including trust-region and proximal policy optimization~\citep{trpo,ppo}, group relative policy optimization (GRPO)~\citep{grpo,deepseekr1}, and direct preference optimization~\citep{dpo}. The same principle was extended to image diffusion models by treating denoising as a sequential decision process optimized with policy gradients~\citep{ddpo,dpok}, differentiating through the sampler under differentiable rewards~\citep{draft,alignprop,reno}, and adapting preference optimization to diffusion models~\citep{diffusiondpo}. Flow-GRPO~\citep{flowgrpo} and DanceGRPO~\citep{dancegrpo} bring GRPO to flow-matching image generators by converting the deterministic sampler into a stochastic policy through an ODE-to-SDE construction with Gaussian transition likelihoods and online updates, making sampled denoising trajectories admit tractable per-step importance ratios. This is the policy class studied here, and it raises a central question: how much policy movement can each denoising step absorb before its likelihood ratio becomes unstable? The closest empirical antecedent is GRPO-Guard~\citep{grpoguard}, which reports the same ratio pathologies studied here, including implicit over-optimization, ratios shifted below one, timestep inconsistency, and clipping imbalance, and stabilizes training through empirical ratio normalization and regulated clipping. That work establishes the failure mode, but its correction is not derived from the transition kernel. It recenters measured ratios after instability appears rather than identifying a quantity that predicts their mean shift and variance before the update, or determines how much gradient a timestep should receive per unit cost. Our finite-grid analysis (\cref{sec:finite-grid-law}) supplies that quantity: conditioned on the current state, the step log-ratio is Gaussian with mean $-\lambda_k/2$ and variance $\lambda_k$. A single $\lambda_k$ therefore determines the typical left shift, ratio dispersion, timestep inconsistency, and clipping imbalance together. Empirical normalization and gradient reweighting can thus be viewed as symptom-level corrections to a quantity that is predictable from the transition law itself.

\textbf{Timestep, sampling, and surrogate-ratio variants.} A rapidly growing literature improves Flow-GRPO from complementary directions. Several methods modify where, when, or how denoising trajectories are sampled or weighted. TempFlow-GRPO~\citep{tempflowgrpo} emphasizes timestep allocation, DenseGRPO~\citep{densegrpo} densifies reward feedback across denoising steps, Smart-GRPO~\citep{smartgrpo} searches over noise perturbations, BranchGRPO~\citep{branchgrpo} introduces structured branching, MixGRPO~\citep{mixgrpo} changes the ODE-to-SDE mixture, CPS~\citep{cps} studies coefficient-preserving sampling, Pref-GRPO~\citep{prefgrpo} replaces pointwise rewards with pairwise preferences to improve optimization stability, OP-GRPO~\citep{opgrpo} targets replay and off-policy efficiency, and Flow-Factory~\citep{flowfactory} organizes the broader design space. These works demonstrate that sampling and credit assignment strongly affect optimization, but none derives a per-step likelihood-ratio resource from the Gaussian transition kernel and uses that quantity to calibrate both ratio statistics and gradient allocation. A $\lambda$ budget is therefore largely complementary to these approaches. FPO~\citep{fpo} takes a different route by replacing exact likelihood computation with a flow-matching surrogate ratio, which is useful when exact transition likelihoods are unavailable. In Flow-GRPO, however, the Euler--Maruyama Gaussian transition is explicit, allowing us to analyze the exact finite-grid ratio and use its path-variance scalar as a predictive optimization budget. Prior work therefore provides the policy class, empirical stabilizers, sampling strategies, and surrogate objectives; our contribution is to connect these phenomena through a single analytically tractable quantity that both predicts per-step ratio behavior and determines how gradient effort should be allocated.

\section{Preliminaries and methods}
\label{sec:prelim-method}

\subsection{Preliminaries}
\label{sec:background}

\textit{Flow matching.} A rectified-flow model~\citep{rectified,flowmatching,stochinterp}
transports data to noise along the linear interpolant
\begin{equation}
  x_t=(1-t)x_0+t x_1,
  \qquad x_0\sim p_{\data},\quad x_1\sim\mathcal N(0,I),\quad t\in[0,1],
  \label{eq:interpolant}
\end{equation}
and learns a velocity field $v_\theta(x,t)$ that predicts $x_1-x_0$. Deterministic
generation then solves the ordinary differential equation
$\dd X_t=v_\theta(X_t,t)\,\dd t$. Rectified flow and closely related diffusion
models~\citep{ddpm,songsde} underlie modern text-to-image generators such as
SD3~\citep{sd3} and FLUX~\citep{flux}, whose transformer backbones follow the
diffusion-transformer line~\citep{dit,sit}.

\textit{Stochastic policy.} GRPO requires stochastic trajectories with tractable
transition densities. Flow-GRPO~\citep{flowgrpo} therefore replaces the deterministic
ODE with a marginal-preserving stochastic differential equation,
$\dd X_t=b_\theta(X_t,t)\,\dd t+\sigma_t\,\dd W_t$, whose drift $b_\theta$ is fixed
through Fokker--Planck matching to the flow-matching score and whose diffusion
$\sigma_t$ follows a prescribed schedule, both derived in \cref{app:prelim}. Training
discretizes this process on a grid $0=t_0<\dots<t_K=1$ with step sizes
$\Delta t_k$, yielding the Euler--Maruyama Gaussian transition kernel
\begin{equation}
  K_\theta(x_{k+1}\mid x_k) =
  \mathcal N\!\bigl(
    x_k+b_\theta(x_k,t_k)\Delta t_k,\;
    \sigma_k\sigma_k^\top\Delta t_k
  \bigr),
  \label{eq:em-kernel}
\end{equation}
which is the central object of our analysis.

\textit{GRPO objective.} For a prompt $c$, a group of $G$ trajectories is sampled
from the old policy and evaluated by a reward model, producing group-relative
advantages $\widehat A^i$~\citep{grpo}, with rewards standardized within each group.
Each sampled transition $x_k^i\to x_{k+1}^i$ carries the per-step importance ratio
\begin{equation}
  r_k^i(\theta) =
  \frac{K_\theta(x_{k+1}^i\mid x_k^i,c)}
       {K_{\old}(x_{k+1}^i\mid x_k^i,c)},
  \label{eq:per-step-ratio}
\end{equation}
and the clipped surrogate~\citep{ppo} optimizes
$\min\!\bigl(r_k^i\widehat A^i,\ \clip(r_k^i,1-\epsilon,1+\epsilon)\widehat A^i\bigr)$
with clipping radius $\epsilon$. The next subsection characterizes how $r_k$ varies
across denoising timesteps.

\subsection{The finite-grid path-ratio law}
\label{sec:finite-grid-law}

The behavior of $r_k$ is determined exactly by the transition kernel under one
assumption satisfied by the fixed-diffusion schedules used in Flow-GRPO.

\begin{assumption}[Fixed diffusion during the policy update]
\label{ass:fixed-diffusion}
At each timestep, the new and old transition kernels share the same diffusion
$\sigma_k\sigma_k^\top$ and differ only in their drift.
\end{assumption}

A direct Gaussian calculation (\cref{app:proofs}) then yields the one-step
log-ratio in closed form.

\begin{proposition}[One-step Gaussian ratio identity]
\label{prop:one-step-ratio}
Let $h_k=\sigma_k^{-1}(b_\theta-b_\old)(X_k,t_k)$ and
$\Delta W_k\sim\mathcal N(0,\Delta t_k I)$. Then
\begin{equation}
  \xi_k \coloneqq
  \log\frac{K_\theta(X_{k+1}\mid X_k)}{K_\old(X_{k+1}\mid X_k)}
  = \inner{h_k}{\Delta W_k}-\tfrac12\norm{h_k}^2\Delta t_k.
  \label{eq:one-step-log-ratio}
\end{equation}
\end{proposition}

The squared coefficient defines the quantity that drives the paper, the
\emph{per-step path variance}
\begin{equation}
  \lambda_k
  \coloneqq
  \norm{h_k}^2\Delta t_k
  =
  \bigl\lVert
  \sigma_k^{-1}(b_\theta-b_\old)(X_k,t_k)
  \bigr\rVert^2\Delta t_k,
  \label{eq:lambda-def}
\end{equation}
which measures the squared policy displacement at timestep $k$ in units of the
transition noise, equivalently the local Mahalanobis distance between the new and old
drifts. This single scalar determines the complete conditional ratio law.

\begin{proposition}[Conditional lognormal ratio law]
\label{prop:lognormal}
Conditioned on $\mathcal F_{t_k}$, the one-step log-ratio is Gaussian,
\begin{equation}
  \xi_k\mid\mathcal F_{t_k}
  \sim
  \mathcal N\!\left(-\tfrac{\lambda_k}{2},\lambda_k\right),
  \label{eq:xi-normal}
\end{equation}
so, with $r_k=\exp(\xi_k)$,
\begin{equation}
  \E[r_k\mid\mathcal F_{t_k}]=1,
  \quad
  \med(r_k\mid\mathcal F_{t_k})=e^{-\lambda_k/2},
  \quad
  \Var(r_k\mid\mathcal F_{t_k})=e^{\lambda_k}-1.
  \label{eq:ratio-moments}
\end{equation}
\end{proposition}

A single $\lambda_k$ therefore controls the negative log-ratio drift, the
left-shifted typical ratio, the ratio variance, the imbalance between the two PPO
clipping tails, and the loss of effective samples. Closed-form clipping-tail
expressions are given in \cref{app:proofs}. Empirical stabilizers from prior work
can therefore be viewed as correcting several symptoms of one analytically
predictable quantity.

One implementation detail is essential. Practical Flow-GRPO code stores a
mean-reduced log-probability rather than the full coordinate sum. Let
$a_{k,m}=(\mu_{\theta,k,m}-\mu_{\old,k,m})/s_{k,m}$ denote the normalized mean shift
for coordinate $m$. The mean-reduced log-ratio is
\begin{equation}
  \bar\xi_k =
  \frac1D\sum_{m=1}^D
  \Bigl(a_{k,m}\epsilon_{k,m}-\tfrac12 a_{k,m}^2\Bigr),
  \qquad
  \epsilon_{k,m}\stackrel{\rm iid}{\sim}\mathcal N(0,1),
  \label{eq:mean-reduced-logratio}
\end{equation}
and the reduction separates the path-variance effect into a centering scale and a
variance scale,
\begin{equation}
  \lambda_k^{\rm center} =
  \frac1D\sum_m a_{k,m}^2,
  \qquad
  \lambda_k^{\rm var} =
  \frac1{D^2}\sum_m a_{k,m}^2.
  \label{eq:lambda-center-var}
\end{equation}

\begin{proposition}[Reduced-log-prob ratio law]
\label{prop:reduced-logprob}
For the mean-reduced log-ratio in
\cref{eq:mean-reduced-logratio}, conditioned on $\mathcal F_{t_k}$,
$
  \E[\bar\xi_k]
  = -\tfrac12\lambda_k^{\rm center},
  \qquad
  \Var(\bar\xi_k)
  = \lambda_k^{\rm var}.
$
\end{proposition}

The two scales differ by the latent dimension. This explains why the predicted
mean-shift law remains pronounced while the raw variance curve can appear nearly flat,
and why centering and variance must be treated separately in the optimizer.

\subsection{$\lambda$-Controlled GRPO}
\label{sec:method}

Our method estimates the path variance online and treats it as a resource to budget
across denoising timesteps. It has two analytic components, and the constants governing
both are determined by standard policy choices rather than tuned as free
hyperparameters. Full derivations are given in
\cref{sec:lambdanorm-t,app:prelim}.

\textit{Estimating the scalar.}
\label{sec:estimate-lambda}
The transition means
$\mu_{\theta,k}=X_k+b_\theta\Delta t_k$ and $\mu_{\old,k}$ are already available
during sampling, so the reduction-matched estimators add negligible cost:
\begin{equation}
  \widehat\lambda_{i,k}^{\rm center}
  =
  \frac1D\sum_{m=1}^D
  \Bigl(
  \frac{\mu_{\theta,k,m}^i-\mu_{\old,k,m}^i}{s_{k,m}}
  \Bigr)^2,
  \qquad
  \widehat\lambda_{i,k}^{\rm var}
  =
  \frac1{D^2}\sum_{m=1}^D
  \Bigl(
  \frac{\mu_{\theta,k,m}^i-\mu_{\old,k,m}^i}{s_{k,m}}
  \Bigr)^2,
  \label{eq:lambda-estimator-reduced-mean-shift}
\end{equation}
where $s_{k,m}=\sigma_{k,m}\sqrt{\Delta t_k}$.
\Cref{app:prelim} gives the endpoint scaling of $\lambda_k$ under the Flow-GRPO
schedule.

\textit{LambdaNorm-T: analytic ratio calibration.}
\label{sec:lambda-gradnorm-main}
Because exponentiating the mean-reduced log-ratio does not preserve the mean-one
property of the full likelihood ratio, we calibrate it from its predicted law instead
of reusing it directly. We center and standardize
$y_{i,k}=\bar\xi_{i,k}$ using the predicted moments and then rescale to a target
variance $\nu$:
\begin{equation}
  z_{i,k}
  = \frac{
  y_{i,k}+\tfrac12\widehat\lambda_{i,k}^{\rm center}
  }{
  \sqrt{\widehat\lambda_{i,k}^{\rm var}+\varepsilon_\lambda}
  },
  \qquad
  \widetilde r_{i,k}
  = \exp\!\bigl(
  \sqrt{\nu}\,z_{i,k}-\tfrac12\nu
  \bigr).
  \label{eq:main-lambdanorm-ratio}
\end{equation}
If $z_{i,k}$ is standard normal, then
$\E[\widetilde r_{i,k}]\approx1$, so $\widetilde r_{i,k}$ becomes a calibrated,
mean-one surrogate ratio. The scale $\nu$ is not introduced as an independent tuning
parameter. PPO already specifies a clipping radius $\epsilon$, so we set
$\nu=\epsilon^2$.

\textit{Damp-only $\lambda$-gradient weighting.}
We next allocate gradient effort according to predicted path-variance cost. A timestep
is damped only when its predicted center exceeds a budget $\tau$:
\begin{equation}
  w_k^\lambda = \min\!\Bigl(1,\;
  \sqrt{\tau/(\widehat\lambda_k^{\rm center}+\varepsilon_\lambda)}\Bigr),
  \qquad
  \mathcal L=-\,\E_{i,k}\!\left[w_k^\lambda\min\!\Bigl(\widetilde r_{i,k}\widehat A_i,\,\clip(\widetilde r_{i,k},1-\epsilon,1+\epsilon)\widehat A_i\Bigr)\right],
  \label{eq:damp-only-lambda-weight}
\end{equation}
where $\widehat\lambda_k^{\rm center}$ is the batch mean of the per-sample estimates
in \cref{eq:lambda-estimator-reduced-mean-shift}, and all estimates are detached so
timesteps already within budget are never amplified. The budget is determined by a
target retained effective-sample fraction $q$~\citep{ess} over the late timesteps
$\mathcal K_{\rm late}$. Under the lognormal path-weight heuristic,
$
  \tau = -\frac{\log q}{|\mathcal K_{\rm late}|}
$
(\cref{sec:lambdanorm-t}). We report realized spend using the mean late-step path
variance
\begin{equation}
  \Lambda_{\rm late} =
  \frac{1}{|\mathcal K_{\rm late}|}
  \sum_{k\in\mathcal K_{\rm late}}
  \widehat\lambda_k^{\rm center},
  \label{eq:Lambda-late}
\end{equation}
which the budget $\tau$ is intended to control. The complete intervention is therefore
specified by two interpretable policy choices, the PPO clipping radius $\epsilon$ and
the retained fraction $q$. We use $\epsilon=0.2$ and $q=0.95$ for OCR, and the
tighter setting $\epsilon=0.1$ and $q=0.99$ for the denser PickScore reward.

\section{Experimental findings}
\label{sec:experimental-findings}

The experiments answer three questions in sequence. The tiny-SD3 diagnostics first test whether the finite-grid law predicts the log-ratio statistics observed during training. The primal-dual experiments then test whether path variance is not only descriptive but also controllable. Finally, the SD3.5 experiments evaluate whether the analytic method, LambdaNorm-T with damp-only $\lambda$-gradient weighting, improves task quality over the empirical stabilizer under matched budgets. The primary held-out comparison is reported in \cref{tab:sd35-hard-ocr-val-selected}.

\subsection{Tiny-SD3 law audit}

The tiny-SD3 audit tests whether the finite-grid path-variance law appears in a working Flow-GRPO implementation. We train a small public SD3 debug pipeline with a compressibility reward and average every statistic over three independent runs from fresh random initializations (configuration in \cref{sec:appendix-config}). Here, $\log r_k$ denotes the mean-reduced log-ratio $\bar\xi_k$ from \cref{eq:mean-reduced-logratio}, which is the quantity stored by the implementation. In the baseline run, the drift-based estimator $\widehat\lambda_k^{\rm center}$ closely tracks the negative log-ratio drift, as predicted by the reduced law $\widehat\lambda_k^{\rm center}\approx-2\,\widehat\E[\log r_k]$. At the final timestep, for example, $\widehat\lambda_7^{\rm center}=0.0054$ predicts a drift of $-0.0027$, exactly matching the measured $-0.0027$. The corresponding $\widehat\lambda_k^{\rm var}$ also lies on the correct scale for the observed variance (\cref{fig:phase0-baseline}). This agreement provides the first direct evidence that the derived law predicts statistics measurable during training.

\begin{figure}[htbp]
  \centering
  \includegraphics[width=0.96\textwidth]{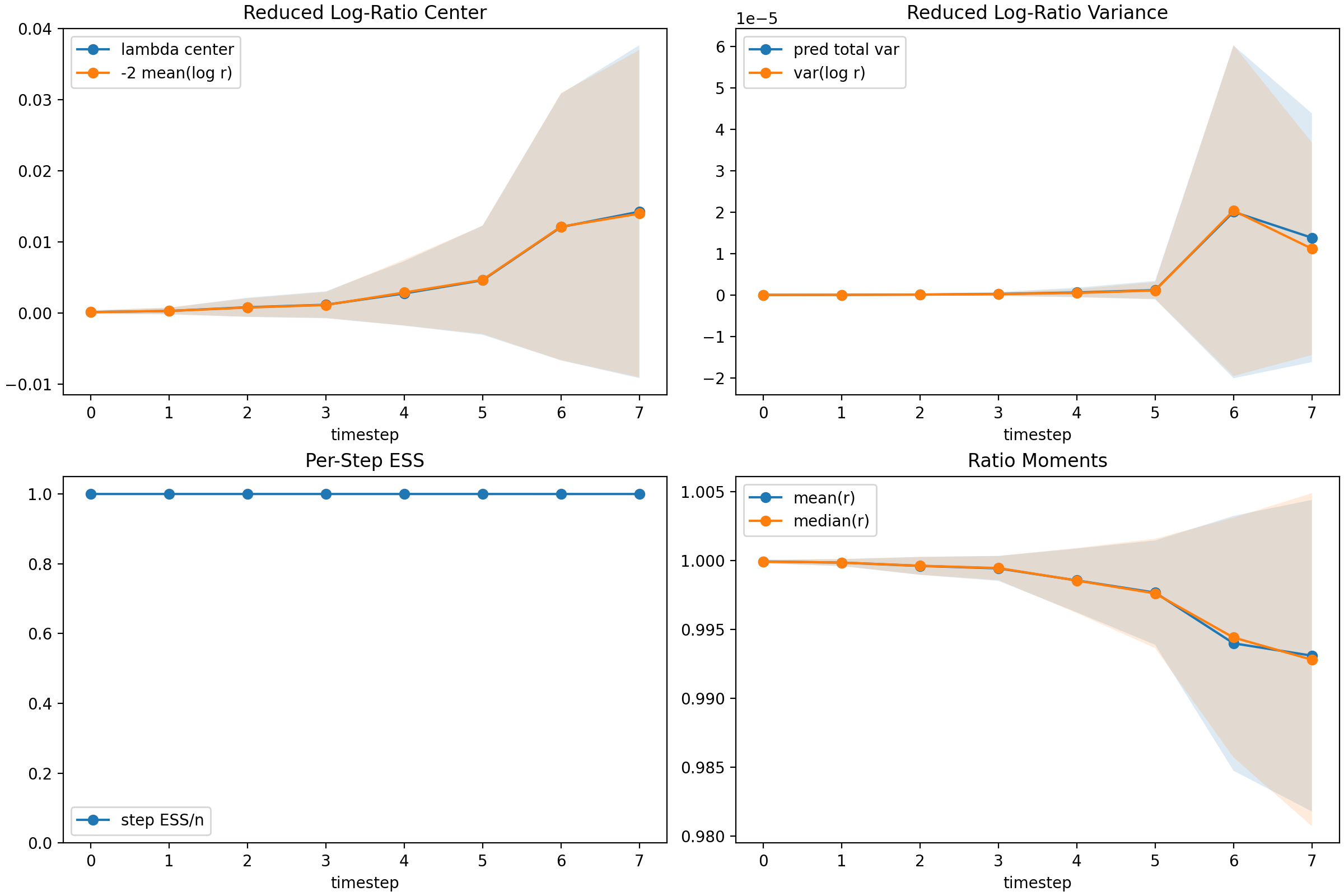}
  \caption{Baseline diagnostics from the tiny-SD3 law audit. The key comparisons are
  between $\widehat\lambda_k^{\rm center}$ and
  $-2\widehat{\E}[\log r_k]$, and between
  $\widehat\lambda_k^{\rm var}$ and $\widehat\Var(\log r_k)$ under the
  reduced-log-prob convention.}
  \label{fig:phase0-baseline}
\end{figure}

We next add $\lambda$-$z$ clipping to test whether controlling the predicted quantity also controls the measured pathology. Averaged over three seeds and the final $50$ updates, it reduces the late-step $\widehat\lambda_7^{\rm center}$ from $0.0054$ to $1.0\times10^{-4}$ and the log-ratio drift from $-0.0027$ to $-4\times10^{-5}$, while leaving reward essentially unchanged (\cref{fig:phase0-baseline-zclip}; full statistics in \cref{tab:tiny-sd3-zclip}). Because this experiment uses a small debug model whose effective sample size remains near one, it is a law audit rather than a performance result. Its purpose is to establish that $\widehat\lambda_k^{\rm center}$ and $\widehat\lambda_k^{\rm var}$ act as accurate online predictors of the reduced log-ratio center and variance, and that directly reducing $\lambda$ suppresses the corresponding pathology.

\begin{figure}[htbp]
  \centering
  \includegraphics[width=0.96\textwidth]{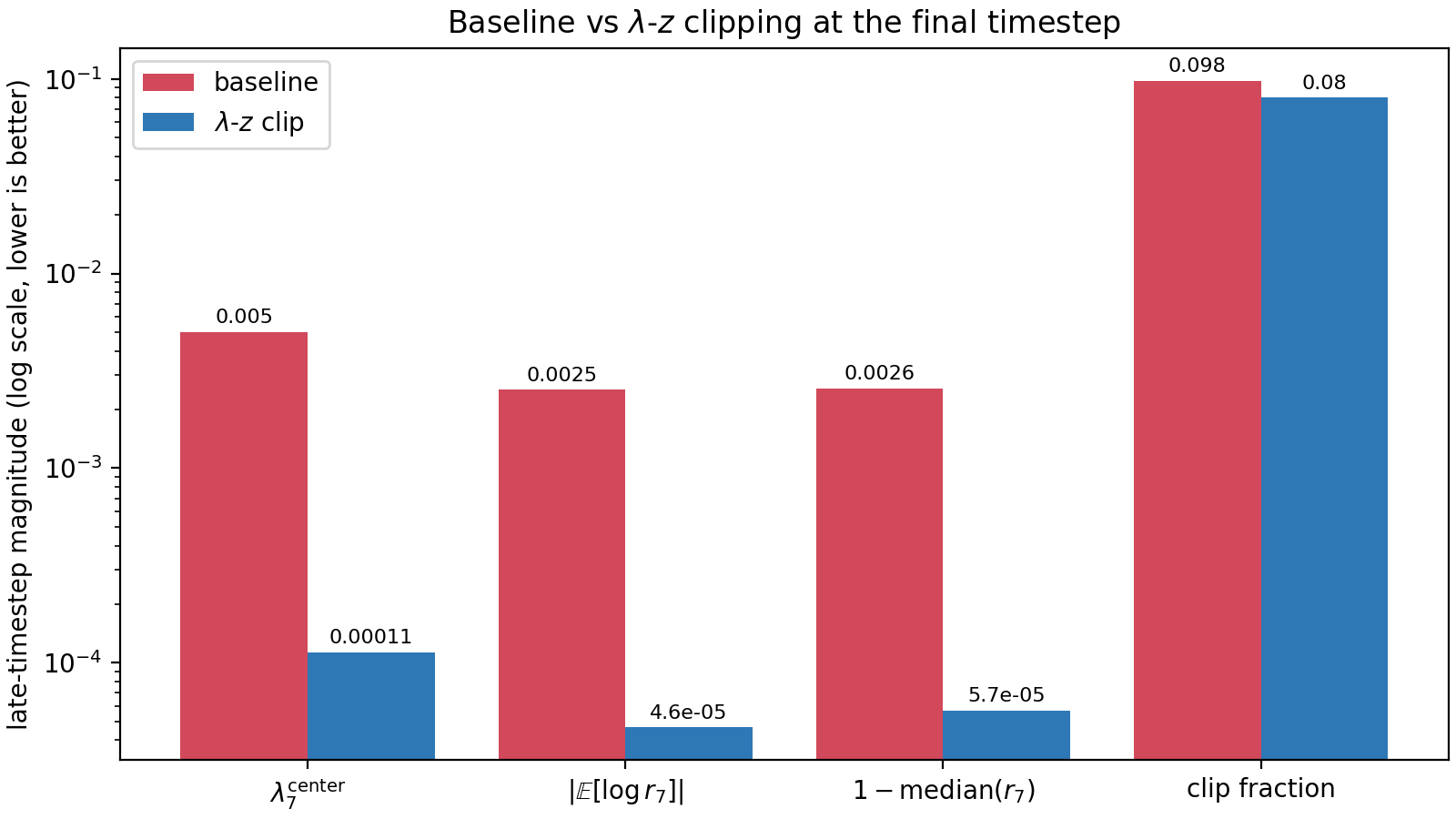}
  \caption{Baseline versus $\lambda$-$z$ clipping in the tiny-SD3 law audit.
  $\lambda$-$z$ clipping sharply reduces the late-step
  $\widehat\lambda_k^{\rm center}$, $\widehat\lambda_k^{\rm var}$, log-ratio mean
  drift, reduced log-ratio variance, and typical-ratio depression.}
  \label{fig:phase0-baseline-zclip}
\end{figure}

\begin{table}[htbp]
  \centering
  \small
  \begin{tabular}{@{}lrr@{}}
  \toprule
  Metric at timestep 7 & Baseline & $\lambda$-$z$ clip \\
  \midrule
  $\widehat\lambda_7^{\rm center}$ & $0.005425$ & $0.000100$ \\
  $\widehat\lambda_7^{\rm var}$ & $3.31{\times}10^{-7}$ & $6.12{\times}10^{-9}$ \\
  $\widehat{\E}[\log r_7]$ & $-0.002693$ & $-0.000041$ \\
  $\widehat\Var(\log r_7)$ & $7.32{\times}10^{-7}$ & $8.83{\times}10^{-9}$ \\
  Median ratio & $0.997264$ & $0.999949$ \\
  Per-step ESS / $n$ & $0.9999993$ & $0.99999996$ \\
  Path ESS / $n$ & $0.999996$ & $1.000000$ \\
  Clip fraction & $0.1169$ & $0.0730$ \\
  Reward average & $-0.008114$ & $-0.008238$ \\
  \bottomrule
  \end{tabular}
  \caption{Tiny-SD3 law audit at the final timestep. Statistics are averaged over
  three seeds and the final $50$ updates. $\lambda$-$z$ clipping reduces predicted
  path variance, log-ratio drift and variance, and typical-ratio depression while
  leaving reward essentially unchanged.}
  \label{tab:tiny-sd3-zclip}
\end{table}

\subsection{Tiny-SD3 stress result: $\lambda$ is controllable}
\label{sec:stress-pilot-dual}

We next test whether $\lambda$ is actionable rather than merely descriptive. As a control ablation, we impose a primal-dual $\lambda$ budget (\cref{sec:candidate-extensions}). Across three paired tiny-SD3 stress seeds, the controller improves the compressibility reward on every seed while reducing the average final-step $\widehat\lambda_7^{\rm center}$ from $0.998$ to $0.166$. This result establishes that the derived scalar can steer optimization. The main method below removes the dual machinery and instead uses the analytic prediction directly through LambdaNorm-T and damp-only weighting.

\subsection{SD3.5 hard-OCR results}
\label{sec:real-sd35-pilots}

Our main real-model experiment fine-tunes SD3.5 Medium with low-rank adaptation (LoRA) on the hard-OCR task. We use the same OCR dataset and reward definition as the public Flow-GRPO and GRPO-Guard configurations, but with smaller update statistics, namely one GPU, smaller groups, fewer batches per epoch, and shorter training. Checkpoints are saved at fixed training-step intervals and selected using a held-out protocol. A random $256$ OCR test prompts are used only for checkpoint selection, while the remaining $762$ prompts are reserved for final evaluation. Alongside the Flow-GRPO-style OCR reward, we report character error rate (CER), defined as the fraction of characters that must be edited to recover the target, together with exact-match and substring-match rates. The latter require the complete target or a contiguous portion of it to appear in the generated image. PickScore~\citep{pickscore} and CLIP score~\citep{clipscore} serve as optional image-quality guardrails. The primal-dual controller from \cref{sec:candidate-extensions} is included only as a baseline, denoted ``Dual,'' alongside empirical RatioNorm.

Under this protocol, empirical RatioNorm selects checkpoint $40$, before the late-training pathology emerges, while the analytic method and primal-dual baseline select checkpoint $80$. On the held-out $762$-prompt test split, LambdaNorm-T with damp-only $\lambda$ weighting outperforms both baselines on OCR reward, CER, exact match, and substring match (\cref{tab:sd35-hard-ocr-val-selected}).

\begin{table}[h]
  \centering
  \begin{tabular}{@{}lrrrrrr@{}}
    \toprule
    Method
      & Ckpt.
      & OCR $\uparrow$
      & CER $\downarrow$
      & Exact $\uparrow$
      & Substr. $\uparrow$
      & $\Lambda_{\rm late}$ $\downarrow$ \\
    \midrule
    RatioNorm
      & $40$ & $0.5632$ & $0.4368$
      & $12.99\%$ & $19.29\%$ & $0.0$ \\
    Dual
      & $80$ & $0.5501$ & $0.4499$
      & $12.07\%$ & $16.80\%$ & $0.000217$ \\
    $\lambda$-Controlled GRPO
      & $80$ & $\mathbf{0.5831}$ & $\mathbf{0.4169}$
      & $\mathbf{14.44\%}$ & $\mathbf{22.83\%}$ & $0.000332$ \\
    \bottomrule
  \end{tabular}
  \caption{SD3.5 hard-OCR held-out results on $762$ test prompts. The analytic
  LambdaNorm-T plus damp-only $\lambda$-gradient method improves over the
  early-stopped empirical RatioNorm checkpoint by $0.0199$ OCR reward,
  $1.44$ exact-match points, and $3.54$ substring-match points, while keeping
  realized late-step spend $\Lambda_{\rm late}$ far below the target budget
  $\tau=-\log(0.95)/5\approx0.01026$.}
  \label{tab:sd35-hard-ocr-val-selected}
\end{table}

On the $256$-prompt validation split used for checkpoint selection, the same method,
LambdaNorm-T with damp-only weighting at $\epsilon=0.2$ and $q=0.95$, improves the
stricter transcript-success metrics much more strongly than the mean OCR reward
(\cref{tab:sd35-hard-ocr-lambda-controlled-val}).

\begin{table}[htbp]
  \centering
  \begin{tabular}{@{}lrrrrrr@{}}
    \toprule
    Method
      & Ckpt.
      & OCR $\uparrow$
      & CER $\downarrow$
      & Exact $\uparrow$
      & Substr. $\uparrow$
      & $\Lambda_{\rm late}$ $\downarrow$ \\
    \midrule
    RatioNorm
      & $40$ & $\mathbf{0.5744}$ & $\mathbf{0.4256}$
      & $10.94\%$ & $14.84\%$ & $\mathbf{0.0}$ \\
    Dual
      & $80$ & $0.5643$ & $0.4357$
      & $8.59\%$ & $14.06\%$ & $0.000217$ \\
    $\lambda$-Controlled GRPO
      & $80$ & $0.5657$ & $0.4343$
      & $\mathbf{15.23\%}$ & $\mathbf{23.83\%}$ & $0.000332$ \\
    \bottomrule
  \end{tabular}
  \caption{Random $256$-prompt SD3.5 hard-OCR validation comparison. The analytic
  $\lambda$-controlled method substantially improves exact-match and substring-match
  rates over empirical RatioNorm and the primal-dual ablation while keeping predicted
  late-step path-variance spend near zero. Relative to RatioNorm, exact match improves
  by approximately $39\%$ and substring match by approximately $61\%$.}
  \label{tab:sd35-hard-ocr-lambda-controlled-val}
\end{table}

The SD3.5 hard-OCR experiment provides the strongest evidence for the main algorithmic claim. Replacing empirical RatioNorm with the finite-grid Gaussian calibration of LambdaNorm-T and shaping gradient allocation using the predicted $\lambda$ cost in \cref{eq:damp-only-lambda-weight} improves the metrics that directly test whether the requested text appears in the image. On the held-out $762$-prompt test split, $\lambda$-Controlled GRPO gains $+0.0199$ OCR reward, $+1.44$ exact-match points, and $+3.54$ substring-match points over RatioNorm. On the $256$-prompt validation split, the strict-metric gains are larger still, approximately $39\%$ for exact match and $61\%$ for substring match. The law audit already established that $\widehat\lambda_k^{\rm center}$ predicts reduced log-ratio drift. Here, using that same prediction to calibrate ratios and allocate gradient effort translates into improved task performance, while the primal-dual runs serve only as a controllability ablation.

\subsection{SD3.5 PickScore results}
\label{sec:real-sd35-pickscore}

To test whether the benefit is specific to OCR, we repeat the protocol with PickScore~\citep{pickscore}, a human-preference reward model in the same broad family as ImageReward~\citep{imagereward} and HPSv2~\citep{hpsv2}. Because this reward is denser, we use the tighter operating point $\epsilon=0.1$ and $q=0.99$, giving $\nu=0.01$ and $\tau\approx2.0\times10^{-3}$. These values are fixed before training. We compare against empirical RatioNorm on the identical training pipeline. Every saved checkpoint at steps $40$, $60$, and $80$ is evaluated on a random $256$-prompt validation split. Each method's best checkpoint is selected by mean PickScore and then evaluated once on the disjoint $762$-prompt test complement.

\begin{table}[htbp]
  \centering
  \small
  \begin{tabular}{@{}lrrrl@{}}
    \toprule
    Method
      & Ckpt.
      & Mean PickScore $\uparrow$
      & Std.
      & Prompt-level comparison \\
    \midrule
    RatioNorm
      & $40$ & $0.8283$ & $0.057$ & --- \\
    $\lambda$-Controlled GRPO
      & $80$ & $\mathbf{0.8358}$ & $0.056$ & wins $467/762$ ($61.3\%$) \\
    \midrule
    \multicolumn{5}{l}{%
      Paired $\Delta$ (Ours $-$ RatioNorm): $+0.0075$,
      $95\%$ bootstrap CI $[+0.0059,\,+0.0092]$, $t\!\approx\!9.1$%
    } \\
    \bottomrule
  \end{tabular}
  \caption{Disjoint $762$-prompt SD3.5 PickScore test split. Representative
  checkpoints are selected using \cref{tab:sd35-pickscore-val} and evaluated once
  on the held-out complement. The paired $95\%$ bootstrap confidence interval excludes
  zero, and $\lambda$-Controlled GRPO beats RatioNorm on $61.3\%$ of prompts with no
  ties. RatioNorm selects checkpoint $40$ on validation; its later checkpoints perform
  worse and are therefore not evaluated on the test split.}
  \label{tab:sd35-pickscore-test}
\end{table}

The PickScore experiment reproduces the same optimization mechanism observed in OCR. Empirical RatioNorm has no explicit path-variance budget, so its late-step $\widehat\lambda^{\rm center}$ rises silently to approximately $3.4\tau$ by step $80$, while mean PickScore falls from $0.8325$ at step $40$ to $0.8126$ at step $80$. By contrast, $\lambda$-Controlled GRPO keeps late-step path variance near $\tau$ and improves monotonically over training. On the held-out test split, the paired bootstrap separates the $+0.0075$ gain from zero across $762$ prompts, and the analytic method wins on a clear majority of prompt-level comparisons.

\subsection{Generalization to a second backbone: FLUX}
\label{sec:flux}

To test whether the analytic method transfers to another flow-matching backbone, we additionally use a broader campaign which trains both SD3.5 and FLUX.1-dev on three tasks, PickScore, hard-OCR, and the GenEval compositional-generation benchmark~\citep{geneval}, and compares $\lambda$-Controlled GRPO against empirical RatioNorm and an additional untuned Flow-GRPO baseline. The held-out FLUX results are summarized in \cref{tab:flux}, and matched held-out GenEval image comparisons for the SD3.5 runs of this campaign are shown in \cref{sec:qualitative-gallery-geneval}.

\begin{table}[htbp]
  \centering
  \small
  \begin{tabular}{@{}lrrr@{\hskip 2em}rrr@{}}
    \toprule
    & \multicolumn{3}{c}{Primary metric $\uparrow$}
    & \multicolumn{3}{c}{Training clip fraction $\downarrow$} \\
    \cmidrule(lr){2-4}\cmidrule(l){5-7}
    FLUX task
      & Flow-GRPO & RatioNorm & $\lambda$-Ctrl
      & Flow-GRPO & RatioNorm & $\lambda$-Ctrl \\
    \midrule
    PickScore  & $\mathbf{0.8430}$ & $0.8383$ & $0.8429$
               & $0.961$ & $0.888$ & $\mathbf{0.089}$ \\
    OCR reward & $0.6434$ & $0.6453$ & $\mathbf{0.6606}$
               & $0.965$ & $0.964$ & $\mathbf{0.060}$ \\
    GenEval    & $0.6545$ & $\mathbf{0.6569}$ & $0.6538$
               & $0.990$ & $0.915$ & $\mathbf{0.100}$ \\
    \bottomrule
  \end{tabular}
  \caption{Held-out FLUX.1-dev results from the broader single-seed $8$-GPU A6000
  reproduction campaign. We compare $\lambda$-Controlled GRPO with untuned Flow-GRPO
  and empirical RatioNorm on PickScore ($762$ paired prompts), hard-OCR ($762$ prompts),
  and GenEval ($2{,}148$ images). The primary metric is mean PickScore, mean OCR reward,
  or mean GenEval composite score. Training clip fraction is averaged over the $50$
  updates preceding the selected checkpoint. The analytic method improves over RatioNorm
  on PickScore and OCR and reduces clipping by approximately an order of magnitude on all
  three tasks; on GenEval, the primary metrics are statistically indistinguishable.}
  \label{tab:flux}
\end{table}

The FLUX experiments extend the SD3.5 findings to a second backbone, although the task-level improvements are not uniform. On FLUX PickScore, $\lambda$-Controlled GRPO essentially ties the untuned Flow-GRPO baseline, with a difference of $-0.0001$ and a $95\%$ confidence interval of $[-0.0013,+0.0011]$, while outperforming empirical RatioNorm by $+0.0046$, with interval $[+0.0036,+0.0057]$ and a $63.4\%$ prompt-level win rate. On FLUX hard-OCR, the analytic method achieves the best observed reward and strict metrics, but the confidence intervals cross zero under this single-seed evaluation, so the improvement remains directional rather than resolved. On GenEval, the three methods are statistically indistinguishable.

The most stable finding across backbones is mechanistic. On all three FLUX tasks, the analytic approach lowers the required training clip fraction by about one order of magnitude, mirroring the stabilization pattern previously seen on SD3.5. This value represents the calibrated surrogate ratio that our method optimizes, not a directly comparable raw ratio computed identically across algorithms. Accordingly, we view FLUX as directional support for cross-backbone generalization, most clearly for the predicted clipping behavior rather than as proof of consistent task-level superiority.

\section{Conclusion and limitations}
\label{sec:conclusion}

In this paper we asked whether the timestep-dependent ratio instability observed in Flow-GRPO has a single underlying cause that can be controlled rather than merely patched. It does. The finite-grid transition kernel shows that negative log-ratio drift, timestep-varying variance, clipping imbalance, and the loss of usable samples are all governed by one per-step scalar, the path variance $\lambda_k$, which fixes the conditional mean and variance of the step log-ratio exactly. Empirical stabilizers from prior work can therefore be interpreted as symptom-level corrections to this quantity. The more direct intervention is to control the cause itself: calibrate ratios from the predicted law and allocate gradient effort according to predicted path-variance cost, with the governing scales set by the PPO clipping radius and a target sample-retention level rather than introduced as free tuning parameters. The central insight is that path variance behaves as a resource already measured by the sampler. It can be estimated online, budgeted across timesteps, and spent through the optimizer, turning Flow-GRPO from an empirically stabilized procedure into one calibrated by its own transition law.

These conclusions hold within a defined operating range. The theory is cleanest when the compared policies share the same diffusion coefficient; policy-dependent diffusion, reversed-time conventions, or modified diffusion schedules require a corresponding finite-grid analysis. The Gaussian theory also predicts the ideal transition law, but does not guarantee that its estimated path variance will exactly match realized statistics once approximate scores, finite discretization, and implementation details enter. This is why we verify the predicted law empirically before using it to control the optimizer. The evidence is similarly bounded. The primary SD3.5 comparisons use a single-GPU pilot protocol, while a broader single-seed campaign shows that task-level gains are not uniform: they are strongest for human-preference optimization, favorable but unresolved for text rendering, and neutral on GenEval. These boundaries define the regime in which the present conclusions should be read. Within that regime, path variance is directly measured and explicitly budgeted, providing a concrete account of when and why the resulting Flow-GRPO updates remain controlled.

\subsection*{Reproducibility statement}

The theoretical claims are stated with their assumptions in
\cref{sec:finite-grid-law} and proved in full in \cref{app:proofs}. The estimator of
the path-variance scalar and the two policy-derived scales that fix the method are
specified in \cref{sec:estimate-lambda,sec:method}, and the exact experimental
configuration, including the models, timestep counts, optimization schedules,
reward definitions, and the per-regime operating points, is collected in
\cref{sec:appendix-config}. Code to reproduce the training and evaluation is
provided as supplementary material.

\subsection*{Ethics statement}

This work studies the optimization stability of reinforcement learning for
text-to-image models and does not involve human subjects or the release of new
data. The image generators and reward models used are existing public artifacts,
and the reward-optimization procedure inherits their known limitations, including
the potential to amplify biases present in the underlying models. We see no
additional ethical concerns specific to the analysis presented here.

\subsection*{AI use statement}

A large language model assisted the authors in writing and polishing the language of the manuscript; the authors verified that all claims, proofs, mathematical formulations, and reported values are valid, checked them against the implementation and results, and take full responsibility for the manuscript.

\bibliographystyle{iclr2027_conference}
\bibliography{iclr2027_conference}

@article{flowgrpo,
  title = {{Flow-GRPO}: Training Flow Matching Models via Online {RL}},
  author = {Liu, Jie and Liu, Gongye and Liang, Jiajun and Li, Yangguang and Liu, Jiaheng and Wang, Xintao and Wan, Pengfei and Zhang, Di and Ouyang, Wanli},
  journal = {arXiv preprint arXiv:2505.05470},
  year = {2025}
}

@article{grpo,
  title = {{DeepSeekMath}: Pushing the Limits of Mathematical Reasoning in Open Language Models},
  author = {Shao, Zhihong and Wang, Peiyi and Zhu, Qihao and Xu, Runxin and Song, Junxiao and Bi, Xiao and Zhang, Haowei and Zhang, Mingchuan and Li, Y. K. and Wu, Y. and Guo, Daya},
  journal = {arXiv preprint arXiv:2402.03300},
  year = {2024}
}

@inproceedings{rectified,
  title = {Flow Straight and Fast: Learning to Generate and Transfer Data with Rectified Flow},
  author = {Liu, Xingchao and Gong, Chengyue and Liu, Qiang},
  booktitle = {International Conference on Learning Representations ({ICLR})},
  year = {2023},
  eprint = {2209.03003},
  archiveprefix = {arXiv}
}

@article{grpoguard,
  title = {{GRPO-Guard}: Mitigating Implicit Over-Optimization in Flow Matching via Regulated Clipping},
  author = {Wang, Jing and Liang, Jiajun and Liu, Jie and Liu, Henglin and Liu, Gongye and Zheng, Jun and Pang, Wanyuan and Ma, Ao and Xie, Zhenyu and Wang, Xintao and Wang, Meng and Wan, Pengfei and Liang, Xiaodan},
  journal = {arXiv preprint arXiv:2510.22319},
  year = {2025}
}

@article{opgrpo,
  title = {{OP-GRPO}: Efficient Off-Policy {GRPO} for Flow-Matching Models},
  author = {Zhang, Liyu and Li, Kehan and Han, Tingrui and Zhao, Tao and Sheng, Yuxuan and He, Shibo and Li, Chao},
  journal = {arXiv preprint arXiv:2604.04142},
  year = {2026}
}

@article{tempflowgrpo,
  title = {{TempFlow-GRPO}: When Timing Matters for {GRPO} in Flow Models},
  author = {He, Xiaoxuan and Fu, Siming and Zhao, Yuke and Li, Wanli and Yang, Jian and Yin, Dacheng and Rao, Fengyun and Zhang, Bo},
  journal = {arXiv preprint arXiv:2508.04324},
  year = {2025}
}

@article{densegrpo,
  title = {{DenseGRPO}: From Sparse to Dense Reward for Flow Matching Model Alignment},
  author = {Deng, Haoyou and Yan, Keyu and Mao, Chaojie and Wang, Xiang and Liu, Yu and Gao, Changxin and Sang, Nong},
  journal = {arXiv preprint arXiv:2601.20218},
  year = {2026}
}

@article{smartgrpo,
  title = {{Smart-GRPO}: Smartly Sampling Noise for Efficient {RL} of Flow-Matching Models},
  author = {Yu, Benjamin and Liu, Jackie and Cui, Justin},
  journal = {arXiv preprint arXiv:2510.02654},
  year = {2025}
}

@article{mixgrpo,
  title = {{MixGRPO}: Unlocking Flow-based {GRPO} Efficiency with Mixed {ODE-SDE}},
  author = {Li, Junzhe and Cui, Yutao and Huang, Tao and Ma, Yinping and Fan, Chun and Cheng, Yiming and Yang, Miles and Zhong, Zhao and Bo, Liefeng},
  journal = {arXiv preprint arXiv:2507.21802},
  year = {2025}
}

@article{branchgrpo,
  title = {{BranchGRPO}: Stable and Efficient {GRPO} with Structured Branching in Diffusion Models},
  author = {Li, Yuming and Wang, Yikai and Zhu, Yuying and Zhao, Zhongyu and Lu, Ming and She, Qi and Zhang, Shanghang},
  journal = {arXiv preprint arXiv:2509.06040},
  year = {2025}
}

@article{cps,
  title = {Coefficients-Preserving Sampling for Reinforcement Learning with Flow Matching},
  author = {Wang, Feng and Yu, Zihao},
  journal = {arXiv preprint arXiv:2509.05952},
  year = {2025}
}

@article{fpo,
  title = {Flow Matching Policy Gradients},
  author = {McAllister, David and Ge, Songwei and Yi, Brent and Kim, Chung Min and Weber, Ethan and Choi, Hongsuk and Feng, Haiwen and Kanazawa, Angjoo},
  journal = {arXiv preprint arXiv:2507.21053},
  year = {2025}
}

@article{flowfactory,
  title = {{Flow-Factory}: A Unified Framework for Reinforcement Learning in Flow-Matching Models},
  author = {Ping, Bowen and Jia, Chengyou and Luo, Minnan and Qian, Hangwei and Tsang, Ivor},
  journal = {arXiv preprint arXiv:2602.12529},
  year = {2026}
}

@inproceedings{ddpm,
  title = {Denoising Diffusion Probabilistic Models},
  author = {Ho, Jonathan and Jain, Ajay and Abbeel, Pieter},
  booktitle = {Advances in Neural Information Processing Systems (NeurIPS)},
  year = {2020},
  eprint = {2006.11239},
  archiveprefix = {arXiv}
}

@inproceedings{songsde,
  title = {Score-Based Generative Modeling through Stochastic Differential Equations},
  author = {Song, Yang and Sohl-Dickstein, Jascha and Kingma, Diederik P. and Kumar, Abhishek and Ermon, Stefano and Poole, Ben},
  booktitle = {International Conference on Learning Representations (ICLR)},
  year = {2021},
  eprint = {2011.13456},
  archiveprefix = {arXiv}
}

@inproceedings{flowmatching,
  title = {Flow Matching for Generative Modeling},
  author = {Lipman, Yaron and Chen, Ricky T. Q. and Ben-Hamu, Heli and Nickel, Maximilian and Le, Matt},
  booktitle = {International Conference on Learning Representations (ICLR)},
  year = {2023},
  eprint = {2210.02747},
  archiveprefix = {arXiv}
}

@article{stochinterp,
  title = {Stochastic Interpolants: A Unifying Framework for Flows and Diffusions},
  author = {Albergo, Michael S. and Boffi, Nicholas M. and Vanden-Eijnden, Eric},
  journal = {Journal of Machine Learning Research (JMLR)},
  volume = {26},
  year = {2025},
  note = {arXiv:2303.08797}
}

@inproceedings{sd3,
  title = {Scaling Rectified Flow Transformers for High-Resolution Image Synthesis},
  author = {Esser, Patrick and Kulal, Sumith and Blattmann, Andreas and Entezari, Rahim and M{\"u}ller, Jonas and Saini, Harry and Levi, Yam and Lorenz, Dominik and Sauer, Axel and Boesel, Frederic and others},
  booktitle = {International Conference on Machine Learning (ICML)},
  year = {2024},
  eprint = {2403.03206},
  archiveprefix = {arXiv}
}

@inproceedings{dit,
  title = {Scalable Diffusion Models with Transformers},
  author = {Peebles, William and Xie, Saining},
  booktitle = {IEEE/CVF International Conference on Computer Vision (ICCV)},
  year = {2023},
  doi = {10.1109/ICCV51070.2023.00387}
}

@inproceedings{sit,
  title = {{SiT}: Exploring Flow and Diffusion-based Generative Models with Scalable Interpolant Transformers},
  author = {Ma, Nanye and Goldstein, Mark and Albergo, Michael S. and Boffi, Nicholas M. and Vanden-Eijnden, Eric and Xie, Saining},
  booktitle = {European Conference on Computer Vision (ECCV)},
  year = {2024},
  eprint = {2401.08740},
  archiveprefix = {arXiv}
}

@misc{flux,
  title         = {FLUX.1 Kontext: Flow Matching for In-Context Image Generation and Editing in Latent Space},
  author        = {Black Forest Labs and Stephen Batifol and Andreas Blattmann and Frederic Boesel and Saksham Consul and Cyril Diagne and Tim Dockhorn and Jack English and Zion English and Patrick Esser and Sumith Kulal and Kyle Lacey and Yam Levi and Cheng Li and Dominik Lorenz and Jonas Müller and Dustin Podell and Robin Rombach and Harry Saini and Axel Sauer and Luke Smith},
  year          = {2025},
  eprint        = {2506.15742},
  archiveprefix = {arXiv},
  primaryclass  = {cs.GR},
  url           = {https://arxiv.org/abs/2506.15742}
}

@article{ppo,
  title = {Proximal Policy Optimization Algorithms},
  author = {Schulman, John and Wolski, Filip and Dhariwal, Prafulla and Radford, Alec and Klimov, Oleg},
  journal = {arXiv preprint arXiv:1707.06347},
  year = {2017}
}

@inproceedings{trpo,
  title = {Trust Region Policy Optimization},
  author = {Schulman, John and Levine, Sergey and Moritz, Philipp and Jordan, Michael I. and Abbeel, Pieter},
  booktitle = {International Conference on Machine Learning (ICML)},
  year = {2015},
  eprint = {1502.05477},
  archiveprefix = {arXiv}
}

@inproceedings{instructgpt,
  title = {Training Language Models to Follow Instructions with Human Feedback},
  author = {Ouyang, Long and Wu, Jeffrey and Jiang, Xu and Almeida, Diogo and Wainwright, Carroll and Mishkin, Pamela and Zhang, Chong and Agarwal, Sandhini and others},
  booktitle = {Advances in Neural Information Processing Systems (NeurIPS)},
  year = {2022},
  doi = {10.52202/068431-2011}
}

@inproceedings{christiano2017,
  title = {Deep Reinforcement Learning from Human Preferences},
  author = {Christiano, Paul F. and Leike, Jan and Brown, Tom B. and Martic, Miljan and Legg, Shane and Amodei, Dario},
  booktitle = {Advances in Neural Information Processing Systems (NeurIPS)},
  year = {2017},
  eprint = {1706.03741},
  archiveprefix = {arXiv}
}

@inproceedings{dpo,
  title = {Direct Preference Optimization: Your Language Model is Secretly a Reward Model},
  author = {Rafailov, Rafael and Sharma, Archit and Mitchell, Eric and Ermon, Stefano and Manning, Christopher D. and Finn, Chelsea},
  booktitle = {Advances in Neural Information Processing Systems (NeurIPS)},
  year = {2023},
  doi = {10.52202/075280-2338}
}

@article{deepseekr1,
  title = {{DeepSeek-R1}: Incentivizing Reasoning Capability in {LLMs} via Reinforcement Learning},
  author = {Guo, Daya and Yang, Dejian and Zhang, Haowei and Song, Junxiao and others},
  journal = {arXiv preprint arXiv:2501.12948},
  year = {2025},
  doi = {10.1038/s41586-025-09422-z}
}

@inproceedings{ddpo,
  title = {Training Diffusion Models with Reinforcement Learning},
  author = {Black, Kevin and Janner, Michael and Du, Yilun and Kostrikov, Ilya and Levine, Sergey},
  booktitle = {International Conference on Learning Representations (ICLR)},
  year = {2024},
  eprint = {2305.13301},
  archiveprefix = {arXiv}
}

@inproceedings{dpok,
  title = {{DPOK}: Reinforcement Learning for Fine-tuning Text-to-Image Diffusion Models},
  author = {Fan, Ying and Watkins, Olivia and Du, Yuqing and Liu, Hao and Ryu, Moonkyung and Boutilier, Craig and Abbeel, Pieter and Ghavamzadeh, Mohammad and others},
  booktitle = {Advances in Neural Information Processing Systems (NeurIPS)},
  year = {2023},
  eprint = {2305.16381},
  archiveprefix = {arXiv}
}

@inproceedings{diffusiondpo,
  title = {Diffusion Model Alignment Using Direct Preference Optimization},
  author = {Wallace, Bram and Dang, Meihua and Rafailov, Rafael and Zhou, Linqi and Lou, Aaron and Purushwalkam, Senthil and Ermon, Stefano and Xiong, Caiming and others},
  booktitle = {IEEE/CVF Conference on Computer Vision and Pattern Recognition (CVPR)},
  year = {2024},
  doi = {10.1109/CVPR52733.2024.00786}
}

@inproceedings{draft,
  title = {Directly Fine-Tuning Diffusion Models on Differentiable Rewards},
  author = {Clark, Kevin and Vicol, Paul and Swersky, Kevin and Fleet, David J.},
  booktitle = {International Conference on Learning Representations (ICLR)},
  year = {2024},
  eprint = {2309.17400},
  archiveprefix = {arXiv}
}

@article{alignprop,
  title = {Aligning Text-to-Image Diffusion Models with Reward Backpropagation},
  author = {Prabhudesai, Mihir and Goyal, Anirudh and Pathak, Deepak and Fragkiadaki, Katerina},
  journal = {arXiv preprint arXiv:2310.03739},
  year = {2023}
}

@inproceedings{reno,
  title = {{ReNO}: Enhancing One-step Text-to-Image Models through Reward-based Noise Optimization},
  author = {Eyring, Luca and Karthik, Shyamgopal and Roth, Karsten and Dosovitskiy, Alexey and Akata, Zeynep},
  booktitle = {Advances in Neural Information Processing Systems (NeurIPS)},
  year = {2024},
  eprint = {2406.04312},
  archiveprefix = {arXiv}
}

@article{dancegrpo,
  title = {{DanceGRPO}: Unleashing {GRPO} on Visual Generation},
  author = {Xue, Zeyue and Wu, Jie and Gao, Yu and Kong, Fangyuan and Zhu, Lingting and Chen, Mengzhao and Liu, Zhiheng and Liu, Wei and Guo, Qiushan and Huang, Weilin and Luo, Ping},
  journal = {arXiv preprint arXiv:2505.07818},
  year = {2025}
}

@article{prefgrpo,
  title = {{Pref-GRPO}: Pairwise Preference Reward-based {GRPO} for Stable Text-to-Image Reinforcement Learning},
  author = {Wang, Yibin and Li, Zhimin and Zang, Yuhang and Zhou, Yujie and Bu, Jiazi and Wang, Chunyu and Lu, Qinglin and Jin, Cheng and Wang, Jiaqi},
  journal = {arXiv preprint arXiv:2508.20751},
  year = {2025}
}

@inproceedings{pickscore,
  title = {Pick-a-Pic: An Open Dataset of User Preferences for Text-to-Image Generation},
  author = {Kirstain, Yuval and Polyak, Adam and Singer, Uriel and Matiana, Shahbuland and Penna, Joe and Levy, Omer},
  booktitle = {Advances in Neural Information Processing Systems (NeurIPS)},
  year = {2023},
  eprint = {2305.01569},
  archiveprefix = {arXiv}
}

@inproceedings{imagereward,
  title = {{ImageReward}: Learning and Evaluating Human Preferences for Text-to-Image Generation},
  author = {Xu, Jiazheng and Liu, Xiao and Wu, Yuchen and Tong, Yuxuan and Li, Qinkai and Ding, Ming and Tang, Jie and Dong, Yuxiao},
  booktitle = {Advances in Neural Information Processing Systems (NeurIPS)},
  year = {2023},
  eprint = {2304.05977},
  archiveprefix = {arXiv}
}

@article{hpsv2,
  title = {Human Preference Score v2: A Solid Benchmark for Evaluating Human Preferences of Text-to-Image Synthesis},
  author = {Wu, Xiaoshi and Hao, Yiming and Sun, Keqiang and Chen, Yixiong and Zhao, Feng and Zhao, Rui and Li, Hongsheng},
  journal = {arXiv preprint arXiv:2306.09341},
  year = {2023}
}

@inproceedings{clipscore,
  title = {{CLIPScore}: A Reference-free Evaluation Metric for Image Captioning},
  author = {Hessel, Jack and Holtzman, Ari and Forbes, Maxwell and Le Bras, Ronan and Choi, Yejin},
  booktitle = {Conference on Empirical Methods in Natural Language Processing (EMNLP)},
  year = {2021},
  doi = {10.18653/v1/2021.emnlp-main.595}
}

@article{ess,
  title = {Rethinking the Effective Sample Size},
  author = {Elvira, V{\'i}ctor and Martino, Luca and Robert, Christian P.},
  journal = {International Statistical Review},
  volume = {90},
  number = {3},
  pages = {525--550},
  year = {2018},
  note = {arXiv:1809.04129},
  doi = {10.1111/insr.12500}
}

@article{geneval,
  title   = {Geneval: An object-focused framework for evaluating text-to-image alignment},
  author  = {Ghosh, Dhruba and Hajishirzi, Hannaneh and Schmidt, Ludwig},
  journal = {Advances in Neural Information Processing Systems},
  volume  = {36},
  pages   = {52132--52152},
  year    = {2023}
}

\appendix

\section{Overview of the appendices}
\label{sec:appendix-overview}

The appendices are grouped by purpose. The first group covers methodology and
theory: the experimental configuration, the flow-matching preliminaries and drift, the derivations and proofs behind the
finite-grid law, the full derivation of the analytic ratio calibration, and the
diagnostic and alternative estimators of the path-variance scalar. The next section gives the
primal-dual baseline and the diagnostic clipping rule used in the law audit, which
appear in the main text only as comparisons. The last three are experimental
evidence, presenting the qualitative held-out comparisons for the OCR, PickScore,
and GenEval reward regimes.

\section{Experimental configuration}
\label{sec:appendix-config}

The tiny SD3 law audit was deliberately small. It used a tiny public SD3 debug
pipeline with a JPEG-compressibility reward and $K=8$ denoising timesteps, trained
for 50 epochs with 32 samples per epoch, 5 inner epochs, and a batch size of 8.
The two matched configurations were a baseline diagnostic run and a
$\lambda$-$z$-clipping run, and every reported statistic is the average of the
last 50 updates over 3 seeds, where a seed is an independent training run from a
fresh random initialization.

The SD3.5 experiments fine-tuned SD3.5 Medium with low-rank adaptation on a single
GPU, using the same OCR and PickScore datasets and reward definitions as the
public Flow-GRPO and GRPO-Guard configurations, but with smaller groups, fewer
batches per epoch, and shorter training than the original multi-GPU regime. The
two policy choices that define the method were fixed before training in each
reward regime: the PPO step radius $\epsilon$ and the effective-sample-size
retention target $q$, from which $\nu=\epsilon^2$ and
$\tau=-\log(q)/|\mathcal K_{\rm late}|$ follow. The OCR runs used
$\epsilon=0.2$ and $q=0.95$; the denser PickScore reward used the tighter
$\epsilon=0.1$ and $q=0.99$.

\section{Flow-matching preliminaries and drift}
\label{app:prelim}

\Cref{sec:prelim-method} states the marginal-preserving SDE and its Gaussian
transition kernel. This appendix gives the drift, the diffusion schedule, and the
endpoint behavior of the path-variance scalar.

The deterministic flow-matching sampler solves
\begin{equation}
  \dd X_t=v_\theta(X_t,t)\,\dd t.
  \label{eq:flow-ode}
\end{equation}
GRPO needs stochastic trajectories with a Gaussian transition density, so
Flow-GRPO turns \cref{eq:flow-ode} into the marginal-preserving SDE
$\dd X_t=b_\theta\,\dd t+\sigma_t\,\dd W_t$. If $p_t$ is the ODE marginal density,
Fokker--Planck matching gives
\begin{equation}
  b_\theta(x,t)=v_\theta(x,t)+\frac{\sigma_t^2}{2}\nabla\log p_t(x).
  \label{eq:marginal-preserve-general}
\end{equation}
For the linear interpolant \cref{eq:interpolant}, the score identity is
\begin{equation}
  \nabla\log p_t(x)=-\frac{x+(1-t)v_t(x)}{t},
  \label{eq:score-identity}
\end{equation}
so substituting the estimated velocity yields the working drift
\begin{equation}
  b_\theta(x,t)=v_\theta(x,t)-\frac{\sigma_t^2}{2t}\bigl(x+(1-t)v_\theta(x,t)\bigr).
  \label{eq:flowgrpo-drift}
\end{equation}
The minus sign comes from the negative score in \cref{eq:score-identity}.
Flow-GRPO commonly uses the schedule
\begin{equation}
  \sigma_t=a\sqrt{\tfrac{t}{1-t}},
  \qquad
  \sigma_t^2=a^2\tfrac{t}{1-t},
  \label{eq:flowgrpo-sigma}
\end{equation}
which is small near the data endpoint $t=0$ and large near the noise endpoint
$t=1$.

\subsection{Endpoint scaling of the path variance}

Let $\Delta(t_k)=\E[\norm{b_\theta(X_k,t_k)-b_\old(X_k,t_k)}^2]$. For scalar
$\sigma_k$, the mean path variance is
\begin{equation}
  \bar\lambda_k=\frac{\Delta(t_k)}{\sigma_k^2}\Delta t_k,
  \label{eq:bar-lambda}
\end{equation}
and with the schedule \cref{eq:flowgrpo-sigma},
\begin{equation}
  \bar\lambda_k=\frac{\Delta(t_k)(1-t_k)}{a^2 t_k}\Delta t_k.
  \label{eq:endpoint-scaling}
\end{equation}
The low-noise endpoint is therefore problematic when $\Delta(t)/t$ does not
vanish: if $\Delta(t)$ is bounded below near $t=0$ then $\bar\lambda_k$ diverges as
$t_k\to0$, whereas $\Delta(t)=O(t)$ keeps the endpoint contribution finite. The
governing diagnostic is the variance-normalized drift scale $\Delta(t)/\sigma_t^2$.

\section{Derivations and proofs}
\label{app:proofs}

This appendix collects the derivations and proofs for the finite-grid path-ratio
law of \cref{sec:finite-grid-law}.

\begin{proof}[Proof of \cref{prop:one-step-ratio}]
The kernels are Gaussian with the same covariance
$\Sigma_k\Delta t_k=\sigma_k\sigma_k^\top\Delta t_k$ and means
$x_k+b_\theta\Delta t_k$ and $x_k+b_\old\Delta t_k$. The determinant terms in
the Gaussian densities cancel. Expanding the difference of the two quadratic
forms and using $X_{k+1}-X_k-b_\old\Delta t_k=\sigma_k\Delta W_k$ gives
\begin{equation}
  \xi_k
  =
  \inner{\sigma_k^{-1}\delta b_k}{\Delta W_k}
  -
  \frac12\norm{\sigma_k^{-1}\delta b_k}^2\Delta t_k,
\end{equation}
which is \cref{eq:one-step-log-ratio}.
\end{proof}

\subsection{Why the path-variance scalar takes this form}

Start with the scalar case. Let the old transition be
$p_{\old}=\mathcal N(\mu_{\old},s^2)$ and the new transition be
$p_\theta=\mathcal N(\mu_\theta,s^2)$ with the same variance. Write the mean shift
in units of the standard deviation as
\begin{equation}
  a
  =
  \frac{\mu_\theta-\mu_{\old}}{s}.
  \label{eq:scalar-normalized-shift}
\end{equation}
If $X\sim p_{\old}$, then $X=\mu_{\old}+sZ$ with $Z\sim\mathcal N(0,1)$, and
direct substitution into the Gaussian densities gives
\begin{equation}
  \log\frac{p_\theta(X)}{p_{\old}(X)}
  =
  aZ-\frac12a^2.
  \label{eq:scalar-gaussian-ratio}
\end{equation}
Thus the same number $a^2$ controls both statistics of the log-ratio:
\begin{equation}
  \E_{\old}\!\left[
    \log\frac{p_\theta(X)}{p_{\old}(X)}
  \right]
  =
  -\frac12a^2,
  \qquad
  \Var_{\old}\!\left(
    \log\frac{p_\theta(X)}{p_{\old}(X)}
  \right)
  =
  a^2.
  \label{eq:scalar-ratio-mean-var}
\end{equation}
This equality is the reason for defining a path-variance scalar: it is the squared
policy move measured in noise units.

The multivariate case is the same calculation after whitening the noise. Suppose
$p_{\old}=\mathcal N(\mu_{\old},\Sigma)$ and
$p_\theta=\mathcal N(\mu_\theta,\Sigma)$ share covariance $\Sigma$. Let
\begin{equation}
  a
  =
  \Sigma^{-1/2}(\mu_\theta-\mu_{\old}).
  \label{eq:multivariate-normalized-shift}
\end{equation}
Under $p_{\old}$, write $X=\mu_{\old}+\Sigma^{1/2}Z$ with
$Z\sim\mathcal N(0,I)$. Expanding the two quadratic forms gives
\begin{equation}
  \log\frac{p_\theta(X)}{p_{\old}(X)}
  =
  \inner{a}{Z}
  -
  \frac12\norm{a}^2,
  \label{eq:multivariate-gaussian-ratio-simple}
\end{equation}
so that
\begin{equation}
  \E_{\old}\!\left[
    \log\frac{p_\theta(X)}{p_{\old}(X)}
  \right]
  =
  -\frac12\norm{a}^2,
  \qquad
  \Var_{\old}\!\left(
    \log\frac{p_\theta(X)}{p_{\old}(X)}
  \right)
  =
  \norm{a}^2.
  \label{eq:multivariate-ratio-mean-var}
\end{equation}
The scalar that appears from the likelihood ratio is therefore the Mahalanobis
distance between the two means,
\begin{equation}
  \norm{a}^2
  =
  (\mu_\theta-\mu_{\old})^\top
  \Sigma^{-1}
  (\mu_\theta-\mu_{\old}).
  \label{eq:mahalanobis-policy-move}
\end{equation}
For one Euler--Maruyama Flow-GRPO step, the old and new kernels share covariance
$\Sigma_k\Delta t_k=\sigma_k\sigma_k^\top\Delta t_k$ and differ only in mean by
$\mu_{\theta,k}-\mu_{\old,k}=(b_\theta-b_\old)(X_k,t_k)\Delta t_k$. Substituting
these into \cref{eq:mahalanobis-policy-move} gives
\begin{equation}
  (\mu_{\theta,k}-\mu_{\old,k})^\top
  (\Sigma_k\Delta t_k)^{-1}
  (\mu_{\theta,k}-\mu_{\old,k})
  =
  \left\lVert
  \sigma_k^{-1}\bigl(b_\theta-b_\old\bigr)(X_k,t_k)
  \right\rVert^2\Delta t_k
  =
  \lambda_k,
  \label{eq:em-mahalanobis-policy-move}
\end{equation}
so the one-step log-ratio has mean $-\lambda_k/2$ and variance $\lambda_k$, and
$\lambda_k$ is the scalar the Gaussian likelihood ratio itself uses for its drift
and variance.

\begin{proof}[Proof of \cref{prop:lognormal}]
Conditioned on $\mathcal F_{t_k}$, $h_k$ is fixed and
$\Delta W_k\sim\mathcal N(0,\Delta t_k I)$. Hence
$\inner{h_k}{\Delta W_k}\sim\mathcal N(0,\norm{h_k}^2\Delta t_k)
=\mathcal N(0,\lambda_k)$. Subtracting $\lambda_k/2$ gives \cref{eq:xi-normal}.
Since $r_k=\exp(\xi_k)$, $r_k$ is lognormal. If $Z\sim\mathcal N(\mu,s^2)$, then
$\E[e^Z]=e^{\mu+s^2/2}$, $\med(e^Z)=e^\mu$, and
$\Var(e^Z)=e^{2\mu+s^2}(e^{s^2}-1)$. Plugging $\mu=-\lambda_k/2$ and
$s^2=\lambda_k$ gives \cref{eq:ratio-moments}.
\end{proof}

\begin{proof}[Proof of \cref{prop:reduced-logprob}]
Taking the conditional expectation of \cref{eq:mean-reduced-logratio} removes the
noise term and leaves $-(2D)^{-1}\sum_m a_{k,m}^2=-\lambda_k^{\rm center}/2$.
Since the $\epsilon_{k,m}$ are independent with unit variance,
\begin{equation}
  \Var(\bar\xi_k\mid\mathcal F_{t_k})
  =
  \frac{1}{D^2}\sum_{m=1}^{D}a_{k,m}^2
  =
  \lambda_k^{\rm var}.
\end{equation}
\end{proof}

\begin{corollary}[Total variance for reduced log-ratios]
\label{cor:reduced-total-variance}
For mean-reduced latent log-probabilities, the batch variance combines
conditional transition noise with heterogeneity in the conditional centers,
$\Var(\bar\xi_k)=\E[\lambda_k^{\rm var}]+\tfrac14\Var(\lambda_k^{\rm center})$.
\end{corollary}

\begin{proof}[Proof of \cref{cor:reduced-total-variance}]
Apply the law of total variance to \cref{prop:reduced-logprob}:
\begin{align}
  \Var(\bar\xi_k)
  &=
  \E[\Var(\bar\xi_k\mid\mathcal F_{t_k})]
  +
  \Var(\E[\bar\xi_k\mid\mathcal F_{t_k}]) \\
  &=
  \E[\lambda_k^{\rm var}]
  +
  \Var(-\lambda_k^{\rm center}/2).
\end{align}
\end{proof}

The same argument gives the unconditional moments of the full log-ratio, which are
useful when comparing against batch statistics that mix many conditional states.

\begin{corollary}[Unconditional first and second moments of log-ratios]
\label{cor:unconditional}
Let $\bar\lambda_k=\E[\lambda_k]$. Then
\begin{equation}
  \E[\xi_k]= -\frac12\bar\lambda_k,
  \qquad
  \Var(\xi_k)=\bar\lambda_k+\frac14\Var(\lambda_k)\geq \bar\lambda_k.
  \label{eq:unconditional-logratio}
\end{equation}
\end{corollary}

\begin{proof}
The expectation follows by the tower property. For the variance, use the law of
total variance,
\begin{equation}
  \Var(\xi_k)
  =
  \E[\Var(\xi_k\mid\mathcal F_{t_k})]
  +
  \Var(\E[\xi_k\mid\mathcal F_{t_k}])
  =
  \E[\lambda_k]+\Var(-\lambda_k/2).
\end{equation}
\end{proof}

\section{Analytic ratio calibration}
\label{sec:lambdanorm-t}

This section gives the longer derivation of LambdaNorm-T, the analytic ratio
calibration used by the main method in \cref{sec:lambda-gradnorm-main}. The
question is whether the reduced-log-ratio law can replace empirical RatioNorm
itself. We call the construction \emph{LambdaNorm-T}, where ``T'' denotes a
target log-ratio variance.

Centering alone is not enough. If
\begin{equation}
  y_{i,k}
  \coloneqq
  \overline{\log r}_{i,k},
\end{equation}
then the reduced law gives
\begin{equation}
  \E[y_{i,k}\mid \mathcal F_{t_k}]
  =
  -\frac{1}{2}\lambda^{\rm center}_{i,k},
  \qquad
  \Var(y_{i,k}\mid \mathcal F_{t_k})
  =
  \lambda^{\rm var}_{i,k}.
  \label{eq:lambdanorm-reduced-law}
\end{equation}
The reduced ratio $\exp(y_{i,k})$ is therefore a surrogate ratio rather than the
full path Radon--Nikodym derivative, and it generally does not preserve the
mean-one likelihood-ratio convention:
\begin{equation}
  \E[\exp(y_{i,k})\mid\mathcal F_{t_k}]
  =
  \exp\!\left(
    -\frac12\lambda_{i,k}^{\rm center}
    +
    \frac12\lambda_{i,k}^{\rm var}
  \right).
  \label{eq:lambdanorm-raw-reduced-mean}
\end{equation}
The analytically standardized coordinate is therefore
\begin{equation}
  z_{i,k}
  =
  \frac{
    y_{i,k}+\frac{1}{2}\lambda^{\rm center}_{i,k}
  }{
    \sqrt{\lambda^{\rm var}_{i,k}+\varepsilon}
  }.
  \label{eq:lambdanorm-z}
\end{equation}
Under the Gaussian transition approximation, $z_{i,k}$ is approximately standard
normal. Using $\exp(z_{i,k})$ directly would impose an arbitrary
and often too-large surrogate ratio scale. LambdaNorm-T instead chooses a target
log-ratio variance $\nu_k>0$ and reconstructs a calibrated surrogate log-ratio:
\begin{equation}
  \widetilde\ell_{i,k}
  =
  \sqrt{\nu_k}\,z_{i,k}
  -
  \frac{1}{2}\nu_k,
  \qquad
  \widetilde r_{i,k}
  =
  \exp(\widetilde\ell_{i,k}).
  \label{eq:lambdanorm-t-ratio}
\end{equation}
If $z_{i,k}\sim \mathcal N(0,1)$, then
$\widetilde\ell_{i,k}\sim \mathcal N(-\nu_k/2,\nu_k)$ and
$\E[\widetilde r_{i,k}]\approx 1$. The subtraction $-\nu_k/2$ is important: it
preserves the mean-one ratio convention after choosing the PPO-scale variance.

The target variance $\nu_k$ is tied to the PPO clipping scale rather than tuned
as an unconstrained normalization constant. If the clipped objective uses a
ratio window $[1-\epsilon,1+\epsilon]$, the corresponding positive log-ratio
boundary is $\log(1+\epsilon)$. We therefore use
        \begin{equation}
  \nu_k
  =
  \log(1+\epsilon)^2
  \approx
  \epsilon^2
  \label{eq:lambdanorm-target-var-from-clip}
\end{equation}
as the default scale, with small sensitivity sweeps around it. For the standard
PPO-scale choice $\epsilon=0.2$, this gives $\nu\approx0.04$, i.e. a target
surrogate log-ratio standard deviation of about $0.2$.

Equivalently, define the centered reduced log-ratio
\begin{equation}
  \ell^c_{i,k}
  =
  y_{i,k}
  +
  \frac{1}{2}\lambda^{\rm center}_{i,k}.
\end{equation}
Then LambdaNorm-T can be written as
\begin{equation}
  \widetilde\ell_{i,k}
  =
  T_{i,k}\ell^c_{i,k}
  -
  \frac{1}{2}T_{i,k}^2\lambda^{\rm var}_{i,k},
          \qquad
  T_{i,k}
  =
  \sqrt{
    \frac{\nu_k}{\lambda^{\rm var}_{i,k}+\varepsilon}
  }.
  \label{eq:lambdanorm-temperature-form}
\end{equation}
This form makes clear that empirical RatioNorm is being replaced by analytic
centering plus analytic variance calibration. In implementation we stop-gradient
through $\lambda^{\rm center}$ and $\lambda^{\rm var}$ inside the normalizer, so
the model cannot win by gaming the coordinate system. The budget penalty still
uses differentiable $\lambda^{\rm center}$, because that term must shape the
update.

The PPO objective uses $\widetilde r_{i,k}$ in place of the empirical RatioNorm
ratio:
\begin{equation}
  \mathcal L_{\rm PPO}^{\rm LambdaNormT}
  =
  -\E_{i,k}
  \left[
    \min\left(
      \widetilde r_{i,k}A_i,
      \clip(\widetilde r_{i,k},1-\epsilon,1+\epsilon)A_i
    \right)
  \right].
  \label{eq:lambdanorm-t-ppo}
\end{equation}
The main method combines this calibrated surrogate ratio with the damp-only
$\lambda$-gradient weight from \cref{eq:damp-only-lambda-weight}. The earlier
primal-dual ablation instead adds the exponential-moving-average (EMA) smoothed
budget penalty from \cref{sec:dual-lambda-budget}:
\begin{equation}
  \mathcal L
  =
  \mathcal L_{\rm PPO}^{\rm LambdaNormT}
  +
  \beta_{\lambda,t}
  \phi_\alpha\!\left(
    \Lambda_{\rm late}-\tau
  \right),
          \qquad
  \Lambda_{\rm late}
  =
  \frac{1}{|\mathcal K_{\rm late}|}
  \sum_{k\in\mathcal K_{\rm late}}
  \E_i[\lambda^{\rm center}_{i,k}].
  \label{eq:lambdanorm-t-dual-objective}
        \end{equation}

There is one further scale issue. For a Gaussian transition
$X_{k+1}\mid X_k\sim\mathcal N(\mu_{\theta,k},s_k^2I)$, the score with respect
to the mean scales like
        \begin{equation}
  \nabla_\mu \log p(X_{k+1}\mid X_k)
  =
  \frac{X_{k+1}-\mu_{\theta,k}}{s_k^2},
  \qquad
  \| \nabla_\mu \log p\| \text{ scales like } s_k^{-1}.
        \end{equation}
Thus low-noise timesteps can receive disproportionately large gradients. A
simple analytic compensation is an optional per-step loss weight
\begin{equation}
  w_k^{\rm score}
  =
  \clip\left(
    \frac{s_k}{s_{\rm ref}},
    w_{\min},
    w_{\max}
  \right),
  \qquad
  s_k
  =
  \E[\sigma_k\sqrt{\Delta t_k}],
  \label{eq:lambdanorm-score-weight}
\end{equation}
which is applied to the per-step PPO loss. This is an ablation, not part of the
minimal LambdaNorm-T definition.

The initial LambdaNorm-T ablation grid varies around the clipping-derived
central value:
\begin{equation}
  \nu
  \in
  \{0.01,\;0.025,\;0.04,\;0.09\},
  \qquad
  \sqrt{\nu}
  \in
  \{0.10,\;0.158,\;0.20,\;0.30\}.
  \label{eq:lambdanorm-target-var-sweep}
\end{equation}
The central setting is $\nu=0.04$, with tighter and looser alternatives included
to test sensitivity. We also include a score-scale-weighted variant at the same
target variance.

LambdaNorm-T answers the sharper analytic question: can the finite-grid
likelihood law replace most or all of empirical RatioNorm? The primary comparison
reports reward together with path-variance spend:
\begin{equation}
  R_{\rm LambdaNormT+dual}
  \gtrsim
  R_{\rm RatioNorm}
  \qquad\text{and}\qquad
  \Lambda_{\rm late,LambdaNormT+dual}
  <
  \Lambda_{\rm late,RatioNorm}.
\end{equation}
This comparison tests whether analytic centering and variance calibration recover
the stabilizing effect of empirical RatioNorm while preserving the same
$\lambda$-budgeted control variable.

\section{Diagnostic and alternative estimators of the path-variance scalar}
\label{sec:appendix-lambda-estimators}

The main method uses the transition mean-shift estimator in
\cref{sec:estimate-lambda}. The estimators below are useful for diagnostics,
sanity checks, and implementation variants, but they are not used to train the
main SD3.5 hard-OCR result.

\subsection{Velocity-only simplification of the drift}
Often the model naturally exposes $v_\theta$, not $b_\theta$. If both policies
use the same scalar schedule $\sigma_t$ and the marginal-preserving drift
\cref{eq:flowgrpo-drift}, then
\begin{align}
  b_\theta(x,t)-b_\old(x,t)
  &=
  \left[1-\frac{\sigma_t^2(1-t)}{2t}\right]
  \left(v_\theta(x,t)-v_\old(x,t)\right).
  \label{eq:drift-diff-velocity}
\end{align}
Therefore
\begin{equation}
  \widehat\lambda_{i,k}^{\rm vel}
  =
  \left[1-\frac{\sigma_k^2(1-t_k)}{2t_k}\right]^2
  \frac{\norm{v_\theta(x_k^i,t_k,c_i)-v_\old(x_k^i,t_k,c_i)}^2}{\sigma_k^2}
  \Delta t_k.
  \label{eq:lambda-velocity}
\end{equation}
For the schedule $\sigma_t^2=a^2t/(1-t)$, the bracket simplifies to
$1-a^2/2$, so
\begin{equation}
  \widehat\lambda_{i,k}^{\rm vel}
  =
  \left(1-\frac{a^2}{2}\right)^2
  \frac{\norm{v_\theta(x_k^i,t_k,c_i)-v_\old(x_k^i,t_k,c_i)}^2}{a^2t_k/(1-t_k)}
  \Delta t_k.
  \label{eq:lambda-flowgrpo-schedule}
\end{equation}
This identity is useful only when applied with the same time convention and same
SDE drift used by the sampler.

\subsection{Ratio-based validation}
The theorem also predicts
\begin{equation}
  \E[\log r_k]\approx-\widehat\lambda_k/2,
  \qquad
  \Var(\log r_k)\approx\widehat\lambda_k.
  \label{eq:ratio-validation}
\end{equation}
Thus one can estimate $\lambda_k$ from empirical log-ratios:
\begin{equation}
  \widehat\lambda_k^{\rm var}=\widehat\Var_i(\log r_{i,k}),
  \qquad
  \widehat\lambda_k^{\rm mean}=-2\widehat\E_i[\log r_{i,k}].
  \label{eq:lambda-ratio-estimators}
\end{equation}
For mean-reduced latent log-probabilities, use the two-component validation from
\cref{prop:reduced-logprob} instead:
\begin{equation}
  -2\widehat\E_i[\bar{\log r}_{i,k}]
  \approx
  \widehat\lambda_k^{\rm center},
  \qquad
  \widehat\Var_i(\bar{\log r}_{i,k})
  \approx
  \widehat\lambda_k^{\rm var}.
  \label{eq:reduced-ratio-validation}
\end{equation}
These ratio-based estimators are the convention used in the law-audit plots in
\cref{sec:experimental-findings}. They are noisier than the
transition mean-shift estimator, but they are a stringent test of the theory: if
the finite-grid law is explanatory, then the mean-shift estimate and the observed
log-ratio statistics should align up to sampling noise.

\section{Primal-dual baseline and diagnostic clipping}
\label{sec:candidate-extensions}

This appendix formalizes the two constructions used outside the main method:
the primal-dual ablation reported as a baseline in
\cref{tab:sd35-hard-ocr-val-selected,tab:sd35-hard-ocr-lambda-controlled-val},
and the diagnostic $z$-space normalization used in the Tiny-SD3 law-audit
intervention in \cref{sec:experimental-findings}. Both are referenced from the
main text; neither is part of the $\lambda$-Controlled GRPO algorithm.

\subsection{Exponentially smoothed primal-dual budget}
\label{sec:dual-lambda-budget}
The primal-dual ablation constrains the late-step path-variance mean
\begin{equation}
  \Lambda_{\rm late}
  =
  \frac{1}{|\mathcal K_{\rm late}|}
  \sum_{k\in\mathcal K_{\rm late}}
  \widehat\lambda_k^{\rm center},
  \label{eq:lambda-late-mean}
\end{equation}
through
\begin{equation}
  \max_\theta\ J(\theta)
  \qquad
  \text{subject to}
  \qquad
  \Lambda_{\rm late}\leq\tau.
  \label{eq:lambda-budget-constraint}
\end{equation}
With empirical RatioNorm as $J(\theta)$, the ablation uses the smooth hinge
\begin{equation}
  \phi_\alpha(u)
  =
  \left[
    \frac{\operatorname{softplus}(\alpha u)}{\alpha}
  \right]^2
  \label{eq:smooth-lambda-hinge}
\end{equation}
and primal loss
\begin{equation}
  \mathcal L_t(\theta)
  =
  \mathcal L_{\rm GRPO/RatioNorm}
  +
  \beta_{\lambda,t}
  \phi_\alpha\!\left(
    \Lambda_{\rm late}(\theta)-\tau
  \right).
  \label{eq:dual-budget-loss}
\end{equation}
The multiplier is updated outside autograd from a detached EMA:
\begin{equation}
  \bar\lambda_{{\rm late},t}
  =
  (1-\gamma)\bar\lambda_{{\rm late},t-1}
  +
  \gamma\,\Lambda_{\rm late}(\theta_t)^{\rm sg},
  \label{eq:dual-lambda-ema}
\end{equation}
\begin{equation}
  \beta_{\lambda,t+1}
  =
  \Pi_{[0,\beta_{\max}]}
  \left[
    \beta_{\lambda,t}
    +
    \eta_\beta
    \left(
      \bar\lambda_{{\rm late},t}
      -
      \tau
      -
      \delta
    \right)
  \right].
  \label{eq:dual-beta-update}
\end{equation}
This ablation established that $\lambda$ can actively control optimization, but
the main method replaces the multiplier dynamics with direct damp-only
$\lambda$-gradient weighting.

\subsection{Diagnostic normalization and clipping}
The reduced law gives the analytic diagnostic
\begin{equation}
  z_{i,k}
  =
  \frac{\overline{\log r}_{i,k}
        +\widehat\lambda_{i,k}^{\rm center}/2}
       {\sqrt{\widehat\lambda_{i,k}^{\rm var}+\varepsilon}}.
  \label{eq:analytic-rationorm}
\end{equation}
One can either log $z_{i,k}$ as a law audit or clip it,
\begin{equation}
  \tilde z_{i,k}=\clip(z_{i,k},-c,c),
  \label{eq:zclip}
\end{equation}
and reconstruct
\begin{equation}
  \widetilde{\overline{\log r}}_{i,k}
  =
  -\frac12\widehat\lambda_{i,k}^{\rm center}
  +\tilde z_{i,k}\sqrt{\widehat\lambda_{i,k}^{\rm var}+\varepsilon}.
  \label{eq:zclip-reconstruct}
\end{equation}
This implies the adaptive interval
\begin{equation}
  \overline{\log r}_{i,k}
  \in
  \left[
    -\frac12\widehat\lambda_{i,k}^{\rm center}
    -c\sqrt{\widehat\lambda_{i,k}^{\rm var}},
    -\frac12\widehat\lambda_{i,k}^{\rm center}
    +c\sqrt{\widehat\lambda_{i,k}^{\rm var}}
  \right].
  \label{eq:adaptive-logclip}
\end{equation}

\clearpage
\section{Qualitative held-out OCR comparisons}
\label{sec:qualitative-gallery}

The aggregate held-out metrics in \cref{tab:sd35-hard-ocr-val-selected} are
supported by a qualitative sweep on the same 762-prompt test complement. The
examples below use matched prompt indices and matched random seeds across
methods: empirical RatioNorm at step 40, the primal-dual ablation at step 80, and
LambdaNorm-T with damp-only $\lambda$ weights at step 80. All examples are drawn
from the disjoint 762-prompt
held-out complement and were not used for checkpoint selection. The first block
shows clear transcript-level wins for $\lambda$-Controlled GRPO. The second
block deliberately includes mixed and failure cases, because the method improves
the distribution of outcomes but does not solve text rendering uniformly.

\begin{figure}[h]
  \centering
  \begin{minipage}[t]{0.32\textwidth}
    \centering
    \textbf{Ours}\\[0.25em]
    \includegraphics[width=\linewidth]{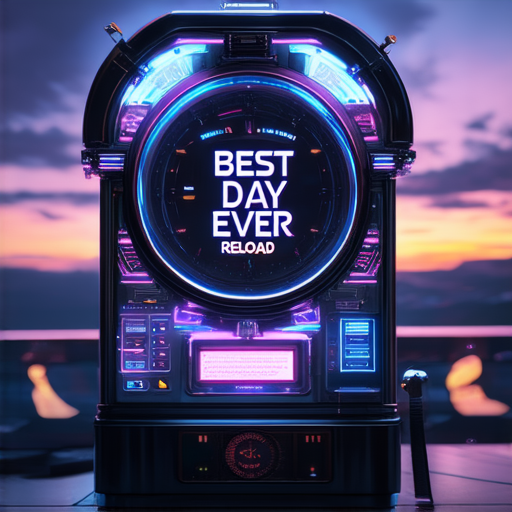}
  \end{minipage}\hfill
  \begin{minipage}[t]{0.32\textwidth}
    \centering
    \textbf{RatioNorm}\\[0.25em]
    \includegraphics[width=\linewidth]{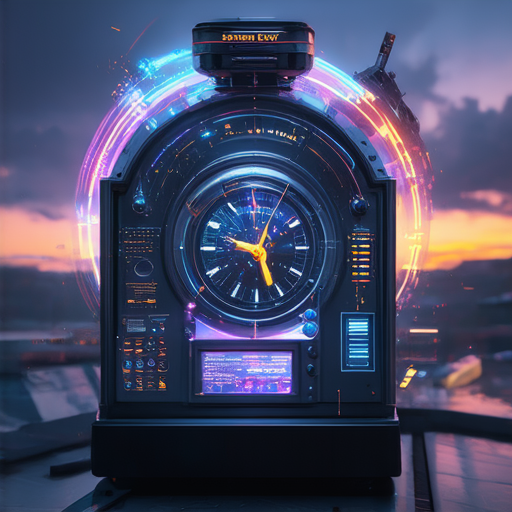}
  \end{minipage}\hfill
  \begin{minipage}[t]{0.32\textwidth}
    \centering
    \textbf{Dual}\\[0.25em]
    \includegraphics[width=\linewidth]{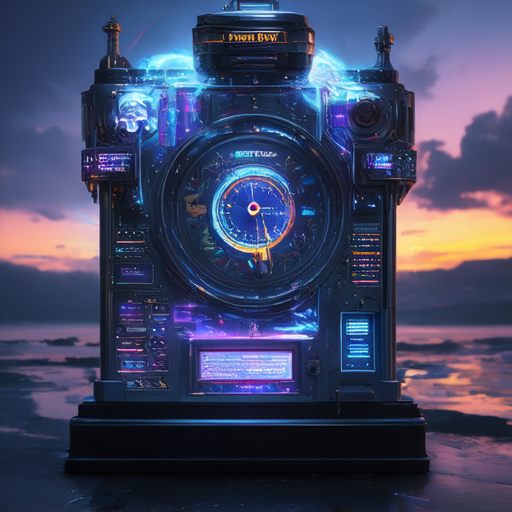}
  \end{minipage}
  \par\smallskip
  {\footnotesize\emph{Prompt.} A futuristic time machine with a sleek, metallic
  finish, its dial set to ``Best Day Ever Reload''. The machine is surrounded by
  a glowing aura, with digital readouts and holographic interfaces displaying
  vibrant colors, set against a backdrop of a twilight sky.}
  \caption{Qualitative held-out test comparison for prompt 319. Target text:
  \texttt{Best Day Ever Reload}. $\lambda$-Controlled GRPO renders the complete target in the
  central display, while RatioNorm and the primal-dual ablation produce the
  object and lighting but fail to place readable target text.}
  \label{fig:qual-p319}
\end{figure}

\begin{figure}[h]
  \centering
  \begin{minipage}[t]{0.32\textwidth}
    \centering
    \textbf{Ours}\\[0.25em]
    \includegraphics[width=\linewidth]{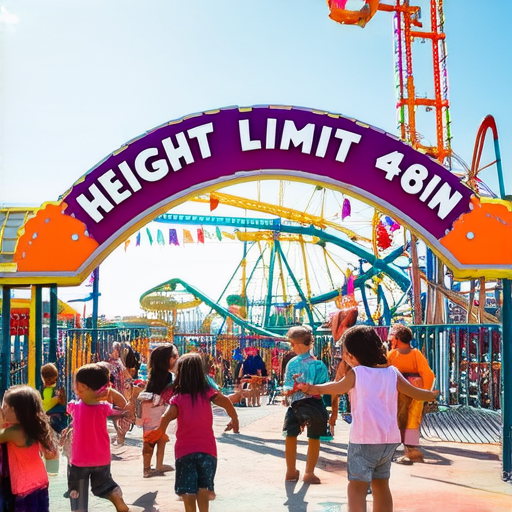}
  \end{minipage}\hfill
  \begin{minipage}[t]{0.32\textwidth}
    \centering
    \textbf{RatioNorm}\\[0.25em]
    \includegraphics[width=\linewidth]{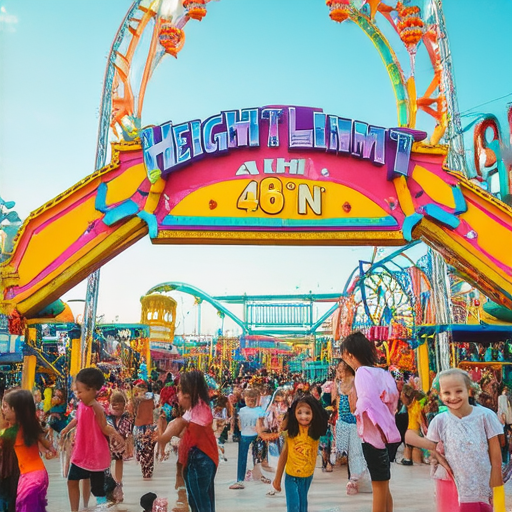}
  \end{minipage}\hfill
  \begin{minipage}[t]{0.32\textwidth}
    \centering
    \textbf{Dual}\\[0.25em]
    \includegraphics[width=\linewidth]{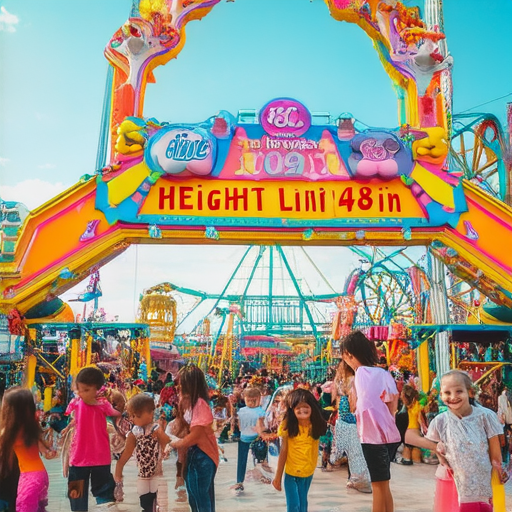}
  \end{minipage}
  \par\smallskip
  {\footnotesize\emph{Prompt.} A bustling amusement park with a vibrant entrance
  arch prominently displaying ``Height Limit 48in'', surrounded by excited
  children and their parents, with colorful banners and playful music in the
  background.}
  \caption{Qualitative held-out test comparison for prompt 107. Target text:
  \texttt{Height Limit 48in}. The analytic method renders the full target on
  the amusement-park arch. RatioNorm and the primal-dual ablation retain the
  colorful park scene, but their sign text corrupts the numerical suffix or
  middle word.}
  \label{fig:qual-p107}
\end{figure}

\begin{figure}[h]
  \centering
  \begin{minipage}[t]{0.32\textwidth}
    \centering
    \textbf{Ours}\\[0.25em]
    \includegraphics[width=\linewidth]{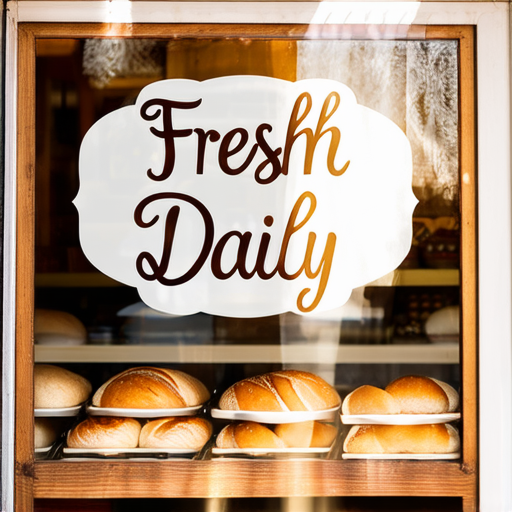}
  \end{minipage}\hfill
  \begin{minipage}[t]{0.32\textwidth}
    \centering
    \textbf{RatioNorm}\\[0.25em]
    \includegraphics[width=\linewidth]{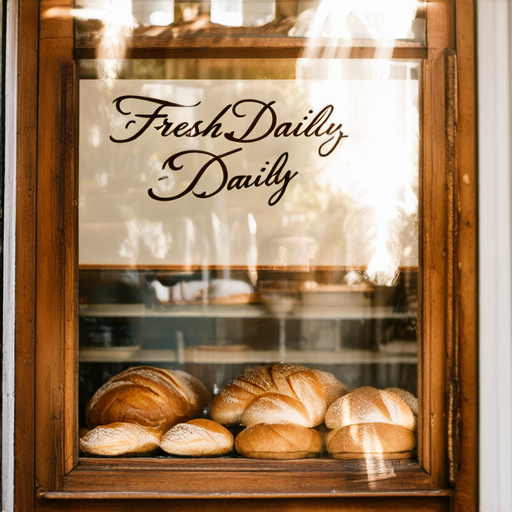}
  \end{minipage}\hfill
  \begin{minipage}[t]{0.32\textwidth}
    \centering
    \textbf{Dual}\\[0.25em]
    \includegraphics[width=\linewidth]{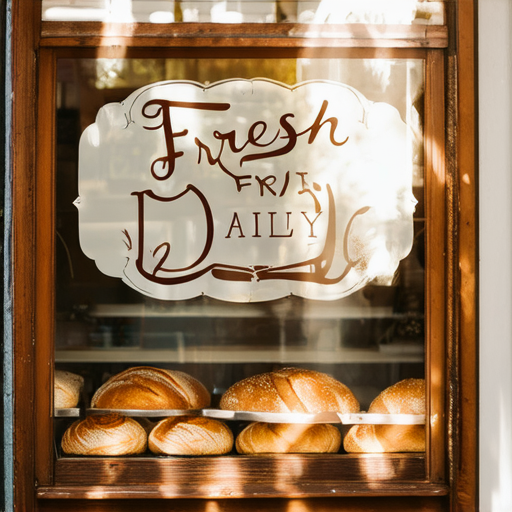}
  \end{minipage}
  \par\smallskip
  {\footnotesize\emph{Prompt.} A charming bakery window with a vintage wooden
  frame, adorned with a decal that reads ``Fresh Daily'' in elegant cursive.
  Sunlight streams through, casting a warm glow on the display of freshly baked
  bread and pastries.}
  \caption{Qualitative held-out test comparison for prompt 145. Target text:
  \texttt{Fresh Daily}. The analytic method gives the cleanest bakery window
  composition and a readable decal, with only a small character-level artifact
  in ``Fresh.'' RatioNorm repeats the text, and the primal-dual output is more
  typographically fragmented.}
  \label{fig:qual-p145}
\end{figure}

\begin{figure}[h]
  \centering
  \begin{minipage}[t]{0.32\textwidth}
    \centering
    \textbf{Ours}\\[0.25em]
    \includegraphics[width=\linewidth]{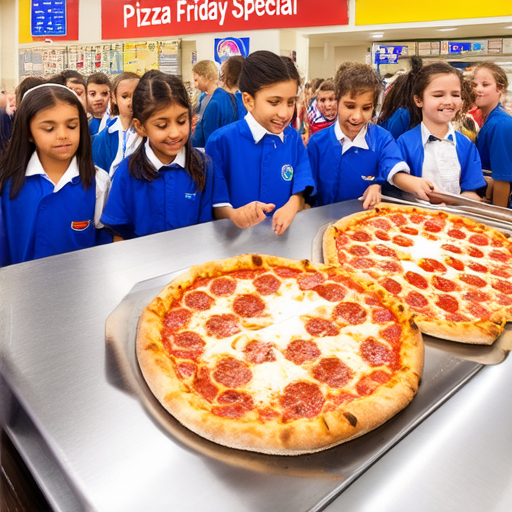}
  \end{minipage}\hfill
  \begin{minipage}[t]{0.32\textwidth}
    \centering
    \textbf{RatioNorm}\\[0.25em]
    \includegraphics[width=\linewidth]{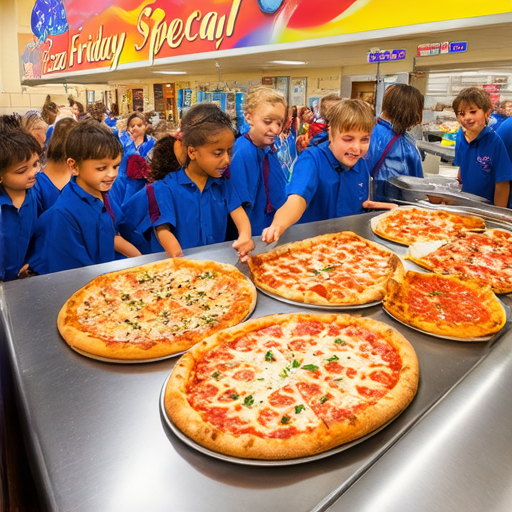}
  \end{minipage}\hfill
  \begin{minipage}[t]{0.32\textwidth}
    \centering
    \textbf{Dual}\\[0.25em]
    \includegraphics[width=\linewidth]{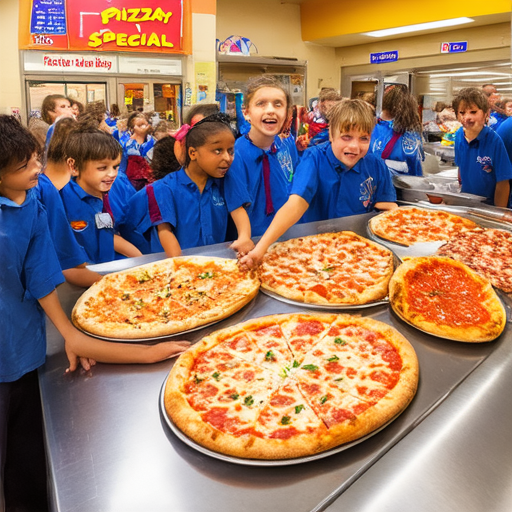}
  \end{minipage}
  \par\smallskip
  {\footnotesize\emph{Prompt.} A bustling school cafeteria on a Friday, with a
  large, colorful sign displaying the ``Pizza Friday Special'' menu. Students in
  vibrant uniforms gather excitedly, pointing at the mouth-watering pizzas
  arranged on the serving counter.}
  \caption{Qualitative held-out test comparison for prompt 366. Target text:
  \texttt{Pizza Friday Special}. $\lambda$-Controlled GRPO writes the full cafeteria banner
  cleanly. RatioNorm and the primal-dual ablation preserve the school-lunch
  scene, but their banner text is partial or corrupted.}
  \label{fig:qual-p366}
\end{figure}

\begin{figure}[h]
  \centering
  \begin{minipage}[t]{0.32\textwidth}
    \centering
    \textbf{Ours}\\[0.25em]
    \includegraphics[width=\linewidth]{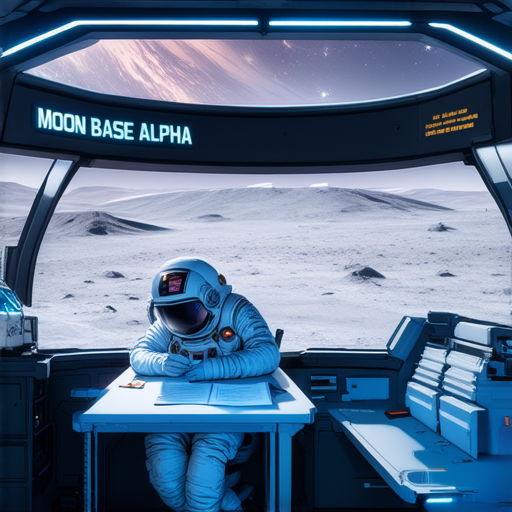}
  \end{minipage}\hfill
  \begin{minipage}[t]{0.32\textwidth}
    \centering
    \textbf{RatioNorm}\\[0.25em]
    \includegraphics[width=\linewidth]{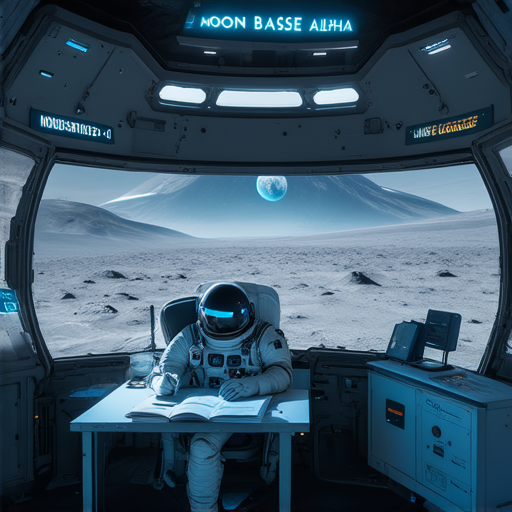}
  \end{minipage}\hfill
  \begin{minipage}[t]{0.32\textwidth}
    \centering
    \textbf{Dual}\\[0.25em]
    \includegraphics[width=\linewidth]{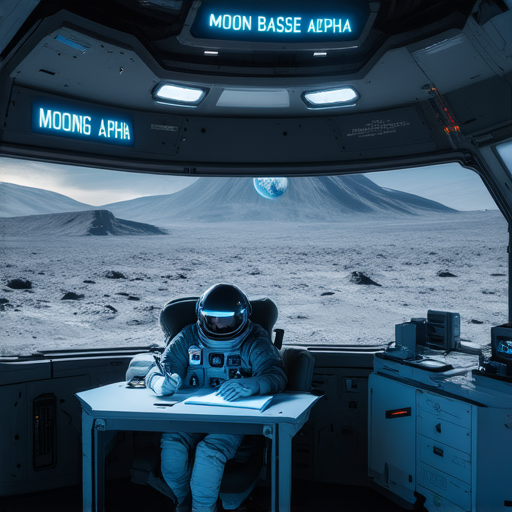}
  \end{minipage}
  \par\smallskip
  {\footnotesize\emph{Prompt.} An astronaut sits at a desk inside Moon Base
  Alpha, writing in a journal. The base's futuristic interior is illuminated by
  soft blue lights, and a large window behind the astronaut showcases the barren
  lunar landscape and Earth rising above the horizon. ``Moon Base Alpha'' is
  prominently displayed on a plaque nearby.}
  \caption{Qualitative held-out test comparison for prompt 747. Target text:
  \texttt{Moon Base Alpha}. $\lambda$-Controlled GRPO keeps the full phrase readable on the base
  sign. The baselines generate coherent lunar-base scenes, but their signage
  repeats or corrupts the target phrase.}
  \label{fig:qual-p747}
\end{figure}

\begin{figure}[h]
  \centering
  \begin{minipage}[t]{0.32\textwidth}
    \centering
    \textbf{Ours}\\[0.25em]
    \includegraphics[width=\linewidth]{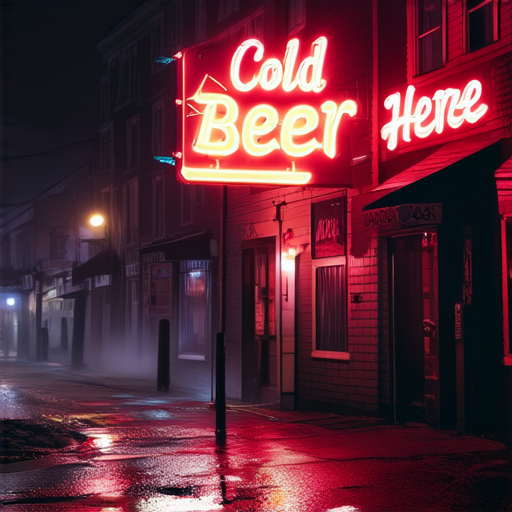}
  \end{minipage}\hfill
  \begin{minipage}[t]{0.32\textwidth}
    \centering
    \textbf{RatioNorm}\\[0.25em]
    \includegraphics[width=\linewidth]{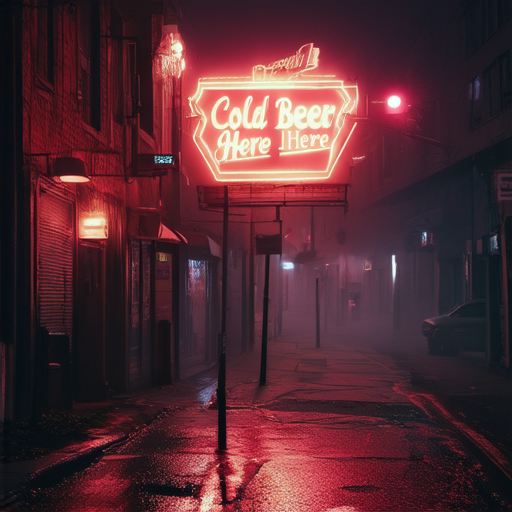}
  \end{minipage}\hfill
  \begin{minipage}[t]{0.32\textwidth}
    \centering
    \textbf{Dual}\\[0.25em]
    \includegraphics[width=\linewidth]{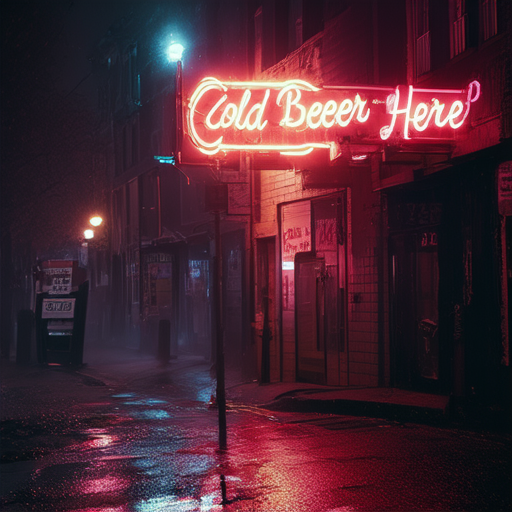}
  \end{minipage}
  \par\smallskip
  {\footnotesize\emph{Prompt.} A gritty urban street at night, with a dive bar
  neon sign glowing brightly, reading ``Cold Beer Here'', casting a warm,
  inviting glow through the misty air, reflecting off the wet pavement.}
  \caption{Qualitative held-out test comparison for prompt 168. Target text:
  \texttt{Cold Beer Here}. The analytic method produces the most legible target
  phrase and strongest neon-street atmosphere. The remaining error is spatial:
  ``Here'' is readable but spills onto the adjacent facade rather than remaining
  entirely within the main sign.}
  \label{fig:qual-p168}
\end{figure}

\begin{figure}[h]
  \centering
  \begin{minipage}[t]{0.32\textwidth}
    \centering
    \textbf{Ours}\\[0.25em]
    \includegraphics[width=\linewidth]{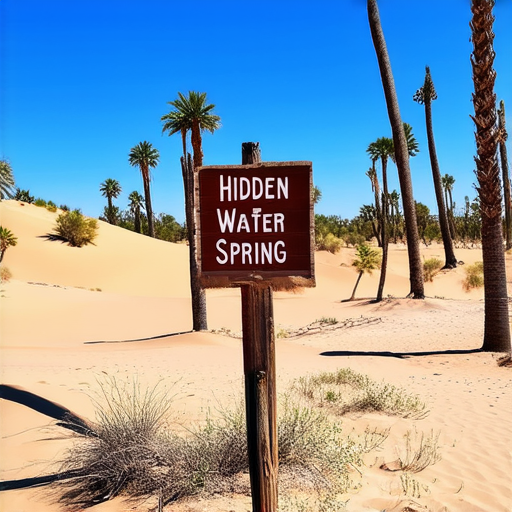}
  \end{minipage}\hfill
  \begin{minipage}[t]{0.32\textwidth}
    \centering
    \textbf{RatioNorm}\\[0.25em]
    \includegraphics[width=\linewidth]{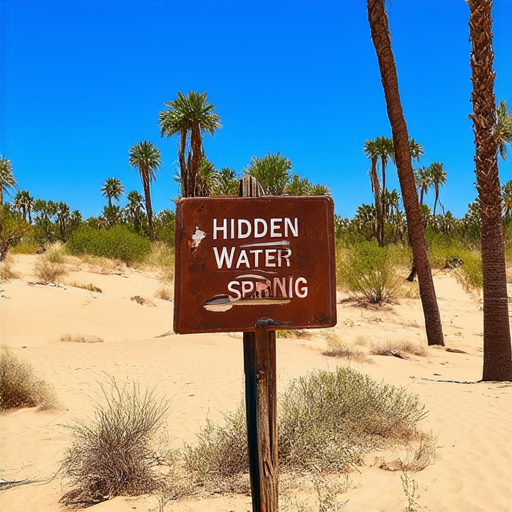}
  \end{minipage}\hfill
  \begin{minipage}[t]{0.32\textwidth}
    \centering
    \textbf{Dual}\\[0.25em]
    \includegraphics[width=\linewidth]{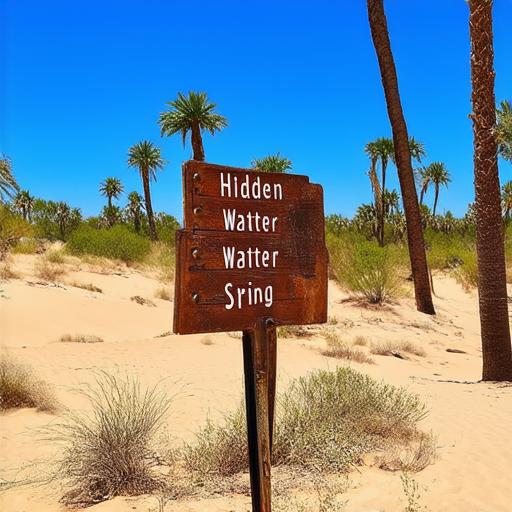}
  \end{minipage}
  \par\smallskip
  {\footnotesize\emph{Prompt.} A serene desert scene with a weathered signpost
  pointing towards ``Hidden Water Spring'', surrounded by tall palm trees and
  golden sand dunes, under a clear blue sky.}
  \caption{Qualitative held-out test comparison for prompt 176. Target text:
  \texttt{Hidden Water Spring}. The analytic method gives an exact, centered,
  readable sign in the intended desert scene. The comparison highlights the kind
  of sample counted by the stricter exact and substring metrics rather than only
  by mean OCR reward.}
  \label{fig:qual-p176}
\end{figure}

\begin{figure}[h]
  \centering
  \begin{minipage}[t]{0.32\textwidth}
    \centering
    \textbf{Ours}\\[0.25em]
    \includegraphics[width=\linewidth]{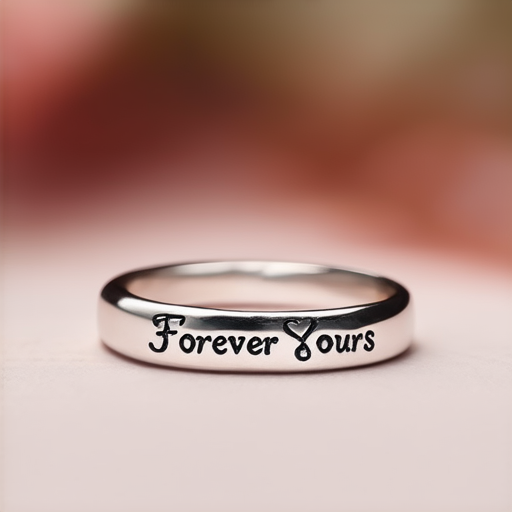}
  \end{minipage}\hfill
  \begin{minipage}[t]{0.32\textwidth}
    \centering
    \textbf{RatioNorm}\\[0.25em]
    \includegraphics[width=\linewidth]{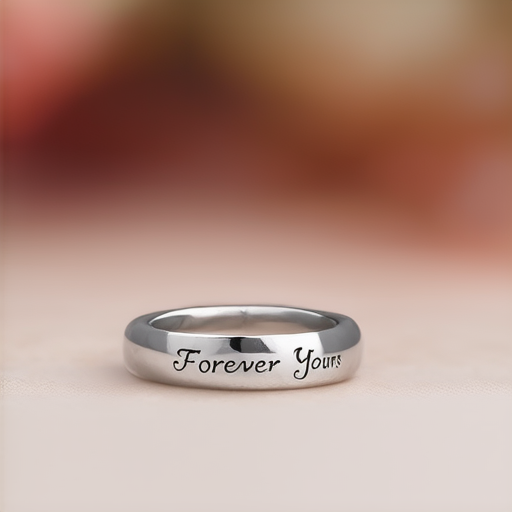}
  \end{minipage}\hfill
  \begin{minipage}[t]{0.32\textwidth}
    \centering
    \textbf{Dual}\\[0.25em]
    \includegraphics[width=\linewidth]{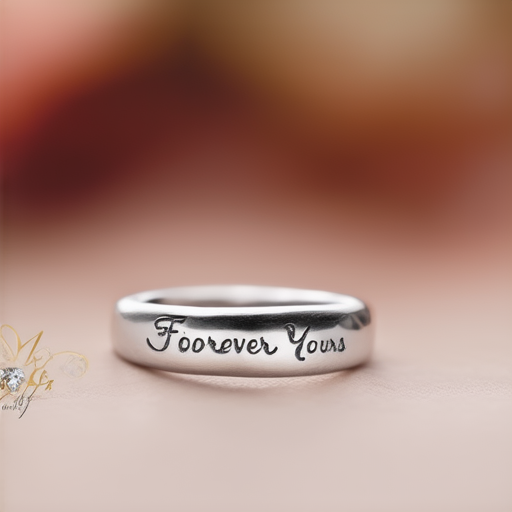}
  \end{minipage}
  \par\smallskip
  {\footnotesize\emph{Prompt.} A close-up photograph of an engraved silver ring
  with the inscription ``Forever Yours'' delicately etched into its surface, set
  against a soft, blurred background of romantic, warm tones.}
  \caption{Qualitative held-out test comparison for prompt 002. Target text:
  \texttt{Forever Yours}. Both $\lambda$-Controlled GRPO and RatioNorm render this
  short phrase cleanly; the primal-dual ablation introduces a character-level
  error.}
  \label{fig:qual-p002}
\end{figure}

\begin{figure}[h]
  \centering
  \begin{minipage}[t]{0.32\textwidth}
    \centering
    \textbf{Ours}\\[0.25em]
    \includegraphics[width=\linewidth]{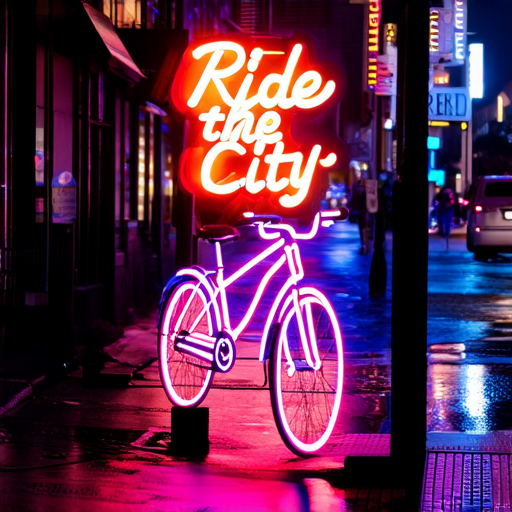}
  \end{minipage}\hfill
  \begin{minipage}[t]{0.32\textwidth}
    \centering
    \textbf{RatioNorm}\\[0.25em]
    \includegraphics[width=\linewidth]{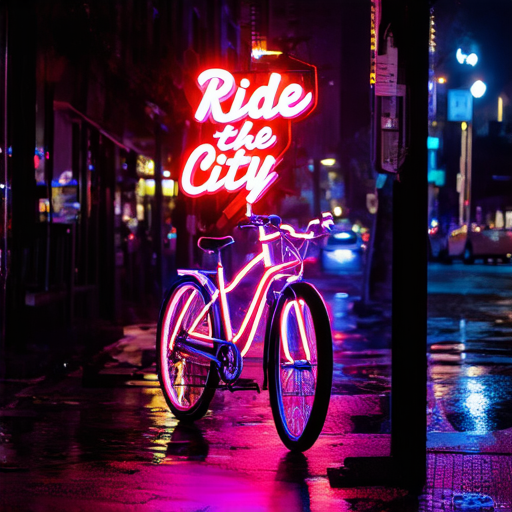}
  \end{minipage}\hfill
  \begin{minipage}[t]{0.32\textwidth}
    \centering
    \textbf{Dual}\\[0.25em]
    \includegraphics[width=\linewidth]{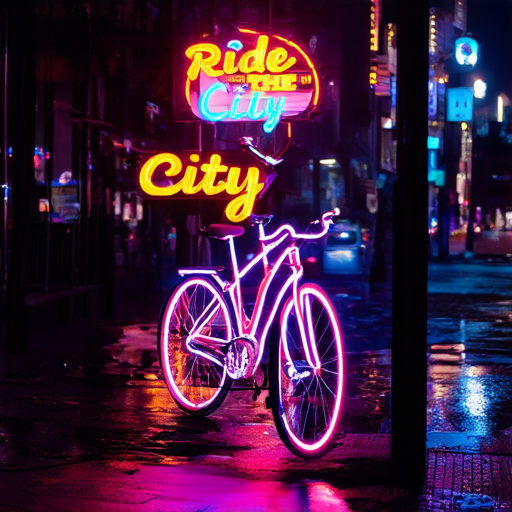}
  \end{minipage}
  \par\smallskip
  {\footnotesize\emph{Prompt.} A neon bike rental sign glowing ``Ride the City''
  stands out against a dark urban backdrop, its vibrant colors reflecting off wet
  pavements in a bustling night scene.}
  \caption{Qualitative held-out test comparison for prompt 047. Target text:
  \texttt{Ride the City}. $\lambda$-Controlled GRPO preserves the complete phrase,
  while the baselines tend to omit or repeat words.}
  \label{fig:qual-p047}
\end{figure}

\begin{figure}[h]
  \centering
  \begin{minipage}[t]{0.32\textwidth}
    \centering
    \textbf{Ours}\\[0.25em]
    \includegraphics[width=\linewidth]{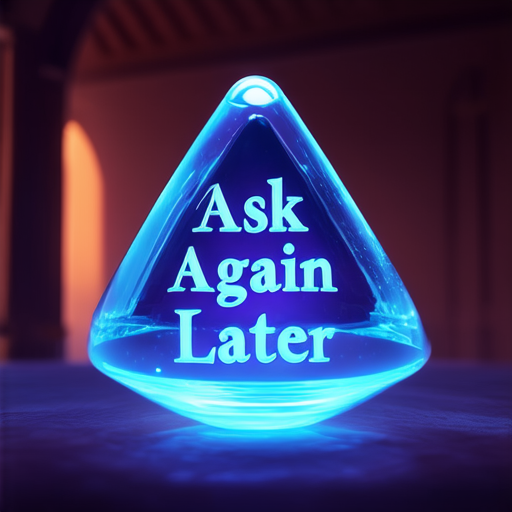}
  \end{minipage}\hfill
  \begin{minipage}[t]{0.32\textwidth}
    \centering
    \textbf{RatioNorm}\\[0.25em]
    \includegraphics[width=\linewidth]{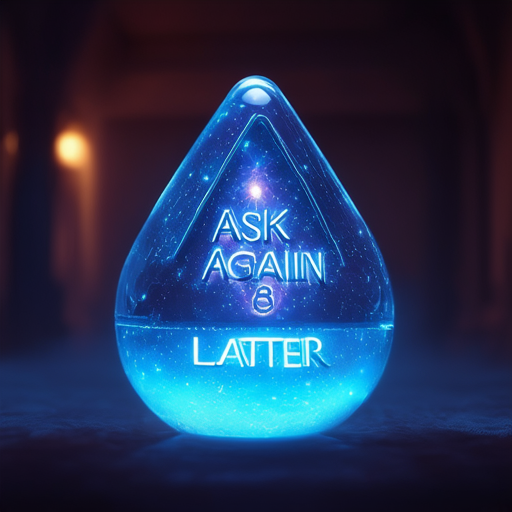}
  \end{minipage}\hfill
  \begin{minipage}[t]{0.32\textwidth}
    \centering
    \textbf{Dual}\\[0.25em]
    \includegraphics[width=\linewidth]{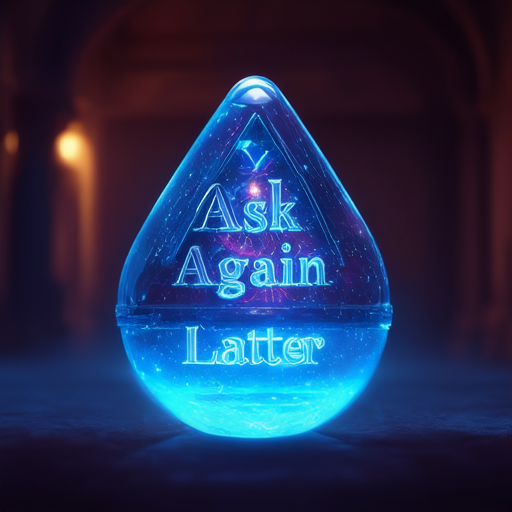}
  \end{minipage}
  \par\smallskip
  {\footnotesize\emph{Prompt.} A glowing Magic 8 Ball floats in a dimly lit
  room, its triangular window displaying the answer ``Ask Again Later'' in
  shimmering, ethereal blue text, surrounded by a soft, mystical aura.}
  \caption{Qualitative held-out test comparison for prompt 130. Target text:
  \texttt{Ask Again Later}. The proposed method keeps the three-word answer
  legible inside the triangular display; the baselines introduce spelling noise.}
  \label{fig:qual-p130}
\end{figure}

\begin{figure}[h]
  \centering
  \begin{minipage}[t]{0.32\textwidth}
    \centering
    \textbf{Ours}\\[0.25em]
    \includegraphics[width=\linewidth]{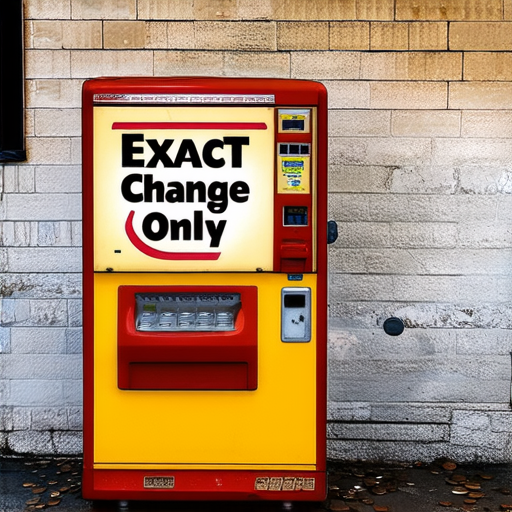}
  \end{minipage}\hfill
  \begin{minipage}[t]{0.32\textwidth}
    \centering
    \textbf{RatioNorm}\\[0.25em]
    \includegraphics[width=\linewidth]{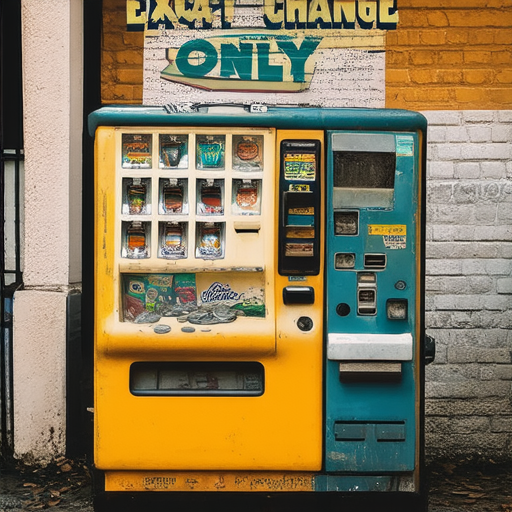}
  \end{minipage}\hfill
  \begin{minipage}[t]{0.32\textwidth}
    \centering
    \textbf{Dual}\\[0.25em]
    \includegraphics[width=\linewidth]{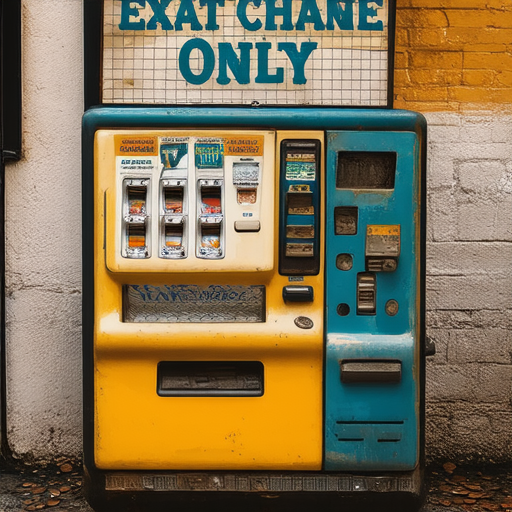}
  \end{minipage}
  \par\smallskip
  {\footnotesize\emph{Prompt.} A realistic photograph of a vintage vending
  machine with a prominent sign that reads ``Exact Change Only'', set against a
  slightly worn brick wall, with a few coins scattered at its base.}
  \caption{Qualitative held-out test comparison for prompt 321. Target text:
  \texttt{Exact Change Only}. $\lambda$-Controlled GRPO places the phrase on the
  vending machine sign, while the baselines produce less recognizable signage.}
  \label{fig:qual-p321}
\end{figure}

\begin{figure}[h]
  \centering
  \begin{minipage}[t]{0.32\textwidth}
    \centering
    \textbf{Ours}\\[0.25em]
    \includegraphics[width=\linewidth]{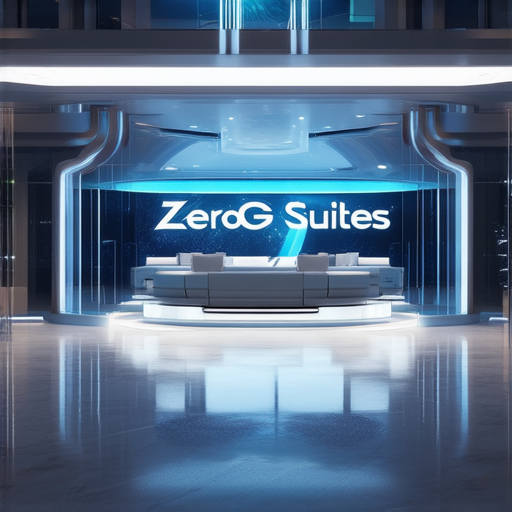}
  \end{minipage}\hfill
  \begin{minipage}[t]{0.32\textwidth}
    \centering
    \textbf{RatioNorm}\\[0.25em]
    \includegraphics[width=\linewidth]{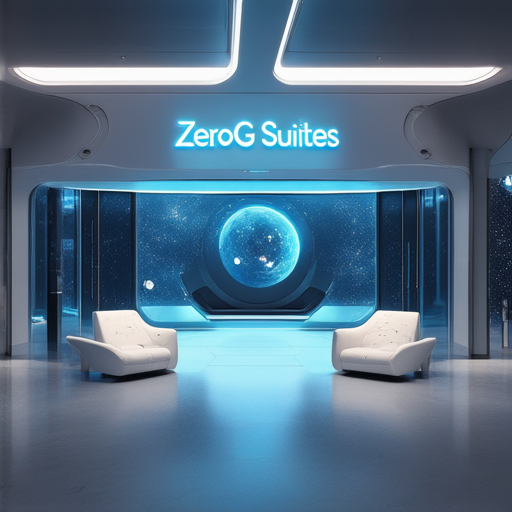}
  \end{minipage}\hfill
  \begin{minipage}[t]{0.32\textwidth}
    \centering
    \textbf{Dual}\\[0.25em]
    \includegraphics[width=\linewidth]{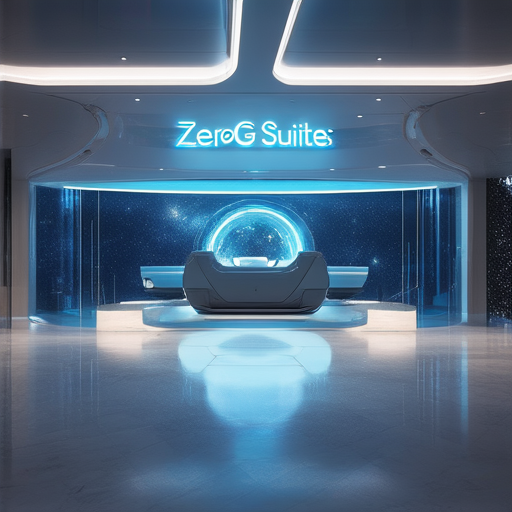}
  \end{minipage}
  \par\smallskip
  {\footnotesize\emph{Prompt.} A futuristic space hotel lobby with a sleek,
  glowing sign that reads ``ZeroG Suites'', surrounded by minimalist decor and
  large windows showcasing the vastness of space outside.}
  \caption{Qualitative held-out test comparison for prompt 456. Target text:
  \texttt{ZeroG Suites}. The proposed method renders the hotel sign cleanly; the
  baselines are close but introduce extra or incorrect letters.}
  \label{fig:qual-p456}
\end{figure}

\begin{figure}[h]
  \centering
  \begin{minipage}[t]{0.32\textwidth}
    \centering
    \textbf{Ours}\\[0.25em]
    \includegraphics[width=\linewidth]{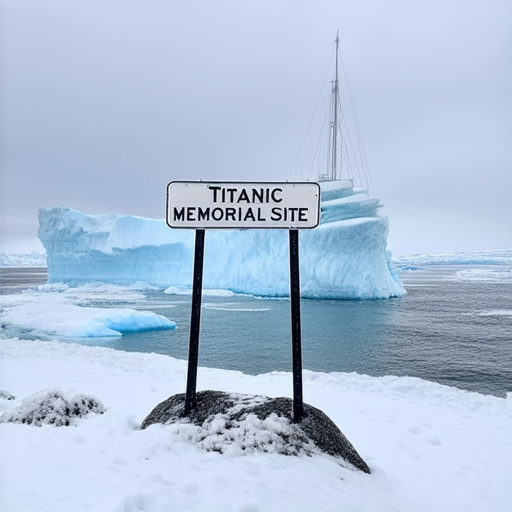}
  \end{minipage}\hfill
  \begin{minipage}[t]{0.32\textwidth}
    \centering
    \textbf{RatioNorm}\\[0.25em]
    \includegraphics[width=\linewidth]{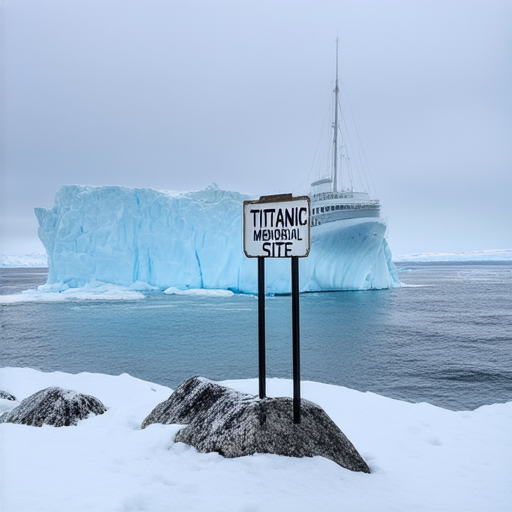}
  \end{minipage}\hfill
  \begin{minipage}[t]{0.32\textwidth}
    \centering
    \textbf{Dual}\\[0.25em]
    \includegraphics[width=\linewidth]{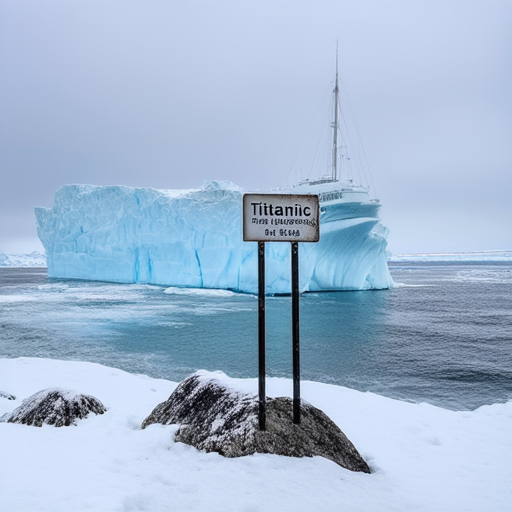}
  \end{minipage}
  \par\smallskip
  {\footnotesize\emph{Prompt.} A vast, icy landscape with a stark, metal sign
  reading ``Titanic Memorial Site'' standing firmly on a snow-covered rock near a
  towering iceberg, the cold waters of the North Atlantic stretching endlessly
  into the horizon.}
  \caption{Qualitative held-out test comparison for prompt 654. Target text:
  \texttt{Titanic Memorial Site}. $\lambda$-Controlled GRPO preserves the full
  memorial sign text in the icy scene, while the baselines retain only fragments.}
  \label{fig:qual-p654}
\end{figure}

\begin{figure}[h]
  \centering
  \begin{minipage}[t]{0.32\textwidth}
    \centering
    \textbf{Ours}\\[0.25em]
    \includegraphics[width=\linewidth]{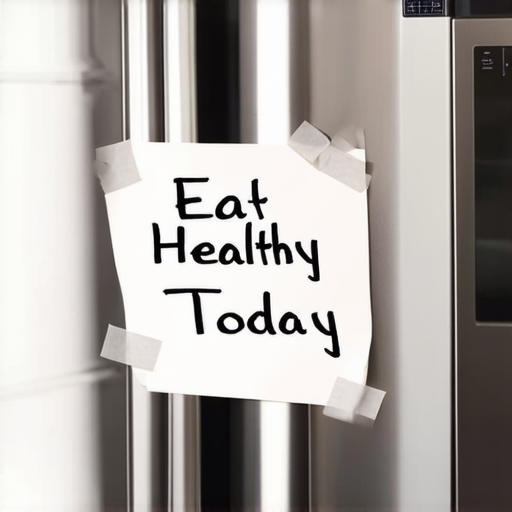}
  \end{minipage}\hfill
  \begin{minipage}[t]{0.32\textwidth}
    \centering
    \textbf{RatioNorm}\\[0.25em]
    \includegraphics[width=\linewidth]{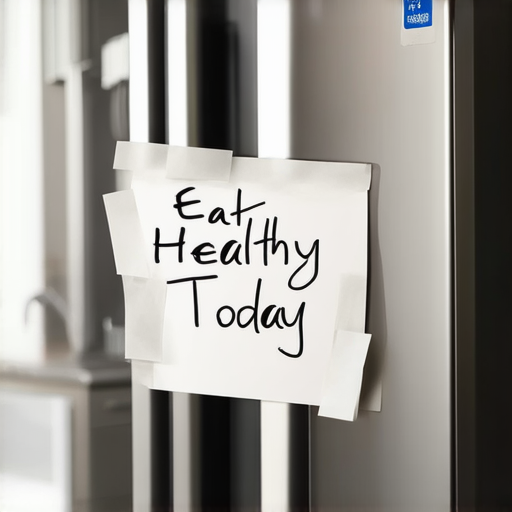}
  \end{minipage}\hfill
  \begin{minipage}[t]{0.32\textwidth}
    \centering
    \textbf{Dual}\\[0.25em]
    \includegraphics[width=\linewidth]{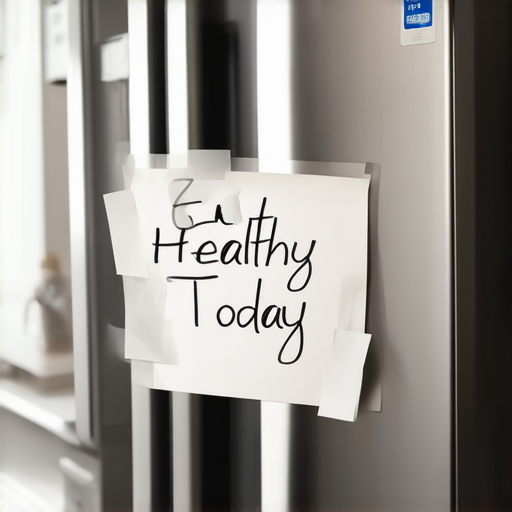}
  \end{minipage}
  \par\smallskip
  {\footnotesize\emph{Prompt.} A realistic photograph of a handwritten note,
  with the words ``Eat Healthy Today'' clearly visible, taped to the door of a
  modern refrigerator in a well-lit kitchen.}
  \caption{Qualitative held-out test comparison for prompt 706. Target text:
  \texttt{Eat Healthy Today}. The proposed method keeps the note text readable;
  the baselines drop or corrupt parts of the phrase.}
  \label{fig:qual-p706}
\end{figure}

\clearpage
\subsection{Mixed and failure cases}
\label{sec:qualitative-mixed-failures}

The examples above show where the analytic intervention is most visually clear.
For calibration, we also scan the disjoint held-out test complement for prompts
where the proposed method is comparable to RatioNorm, worse than RatioNorm, or
where all methods remain far from the requested text. These cases are useful
scientifically: they separate an average improvement in transcript success from
a claim of uniform dominance.

\begin{figure}[h]
  \centering
  \begin{minipage}[t]{0.32\textwidth}
    \centering
    \textbf{Ours}\\[0.25em]
    \includegraphics[width=\linewidth]{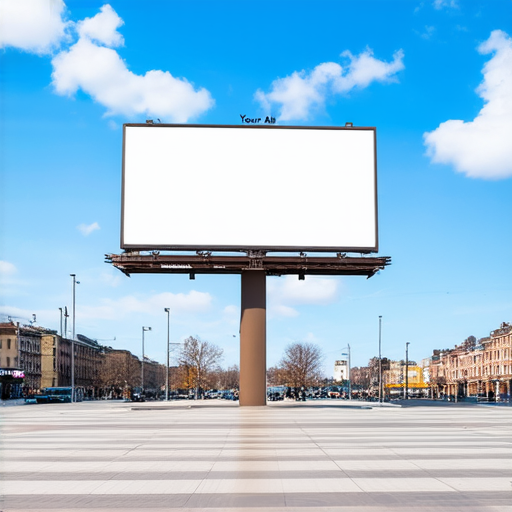}
  \end{minipage}\hfill
  \begin{minipage}[t]{0.32\textwidth}
    \centering
    \textbf{RatioNorm}\\[0.25em]
    \includegraphics[width=\linewidth]{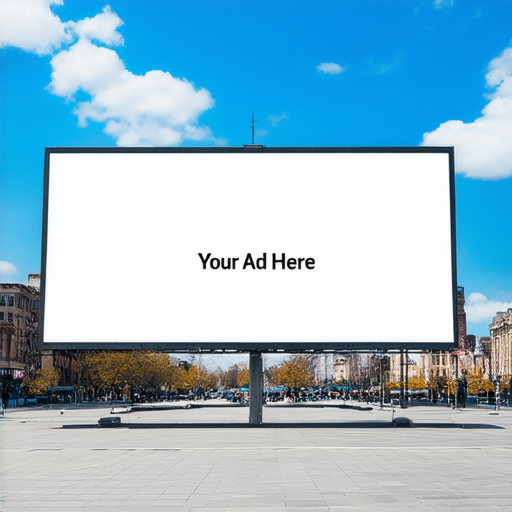}
  \end{minipage}\hfill
  \begin{minipage}[t]{0.32\textwidth}
    \centering
    \textbf{Dual}\\[0.25em]
    \includegraphics[width=\linewidth]{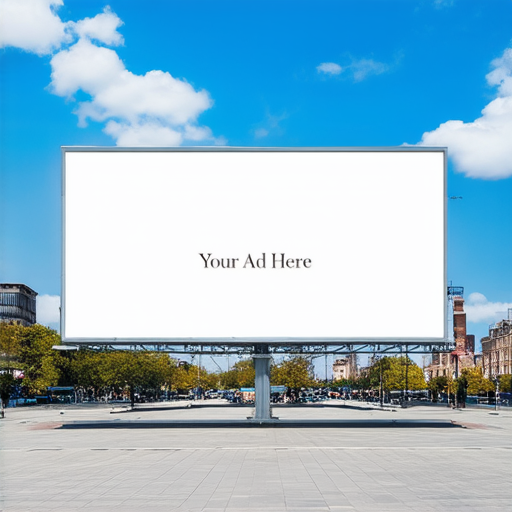}
  \end{minipage}
  \par\smallskip
  {\footnotesize\emph{Prompt.} A vast, empty city square with a large, blank
  billboard standing prominently, awaiting the words ``Your Ad Here'' to be
  filled in, under a clear blue sky with a few fluffy clouds.}
  \caption{Mixed held-out test comparison for prompt 718. Target text:
  \texttt{Your Ad Here}. This is a clear negative case for the proposed method:
  $\lambda$-Controlled GRPO leaves the billboard effectively blank, while both
  empirical RatioNorm and the primal-dual ablation place the intended phrase
  cleanly in the center.}
  \label{fig:qual-mixed-p718}
\end{figure}

\begin{figure}[h]
  \centering
  \begin{minipage}[t]{0.32\textwidth}
    \centering
    \textbf{Ours}\\[0.25em]
    \includegraphics[width=\linewidth]{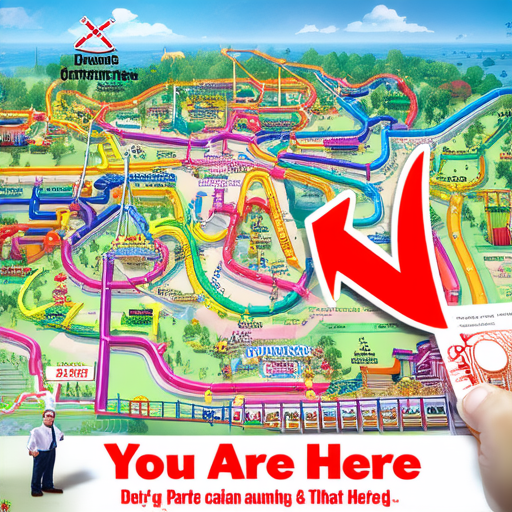}
  \end{minipage}\hfill
  \begin{minipage}[t]{0.32\textwidth}
    \centering
    \textbf{RatioNorm}\\[0.25em]
    \includegraphics[width=\linewidth]{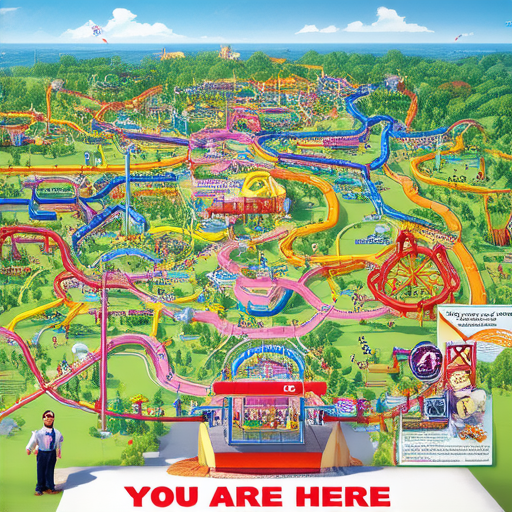}
  \end{minipage}\hfill
  \begin{minipage}[t]{0.32\textwidth}
    \centering
    \textbf{Dual}\\[0.25em]
    \includegraphics[width=\linewidth]{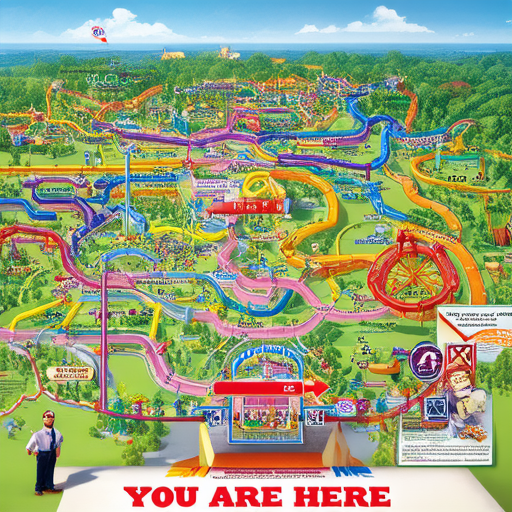}
  \end{minipage}
  \par\smallskip
  {\footnotesize\emph{Prompt.} A detailed amusement park map with colorful
  pathways and attractions, prominently marking ``You Are Here'' with a large,
  red arrow. The map is held by a friendly park employee, standing in front of a
  vibrant, bustling entrance.}
  \caption{Mixed held-out test comparison for prompt 617. Target text:
  \texttt{You Are Here}. The proposed method renders the main phrase but adds
  extra corrupted text, which hurts the transcript metric. RatioNorm and the
  primal-dual ablation give cleaner target text in this particular sample.}
  \label{fig:qual-mixed-p617}
\end{figure}

\begin{figure}[h]
  \centering
  \begin{minipage}[t]{0.32\textwidth}
    \centering
    \textbf{Ours}\\[0.25em]
    \includegraphics[width=\linewidth]{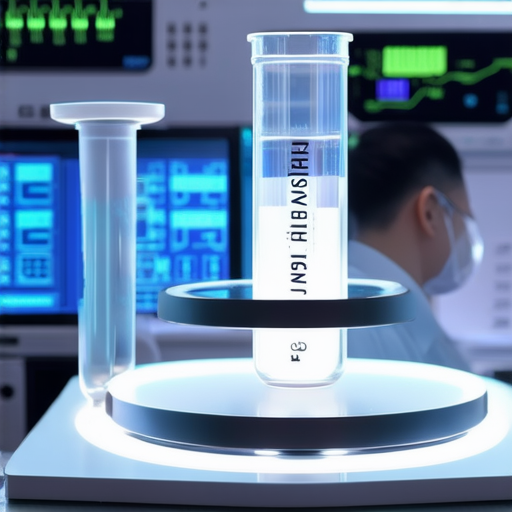}
  \end{minipage}\hfill
  \begin{minipage}[t]{0.32\textwidth}
    \centering
    \textbf{RatioNorm}\\[0.25em]
    \includegraphics[width=\linewidth]{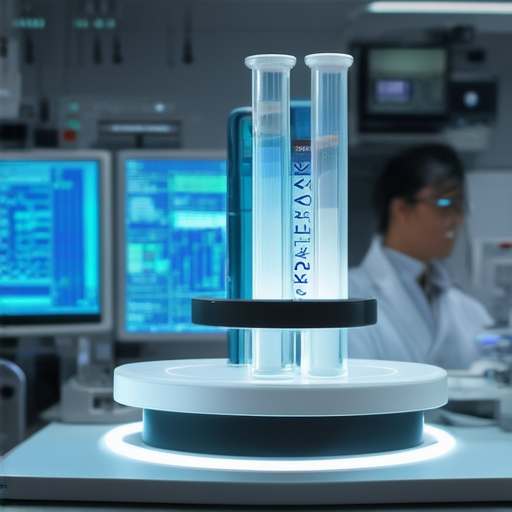}
  \end{minipage}\hfill
  \begin{minipage}[t]{0.32\textwidth}
    \centering
    \textbf{Dual}\\[0.25em]
    \includegraphics[width=\linewidth]{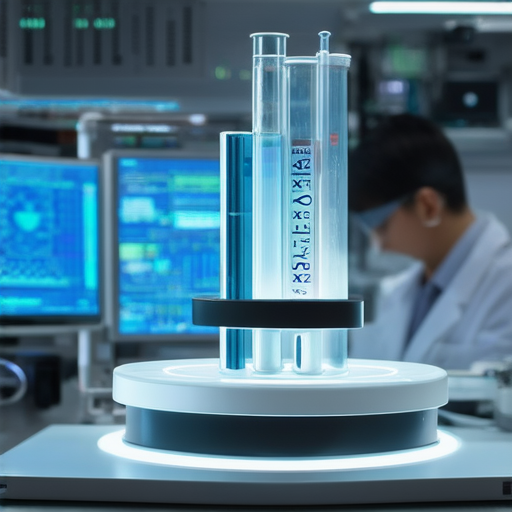}
  \end{minipage}
  \par\smallskip
  {\footnotesize\emph{Prompt.} A high-tech laboratory setting with a test tube
  labeled ``Sample XZ42'' on a sleek, illuminated stand, surrounded by advanced
  scientific equipment and glowing monitors displaying complex data. The scene
  is modern and sterile, with a scientist in the background observing through a
  protective visor.}
  \caption{Failure held-out test comparison for prompt 026. Target text:
  \texttt{Sample XZ42}. All three methods produce plausible laboratory imagery,
  but none renders the alphanumeric label correctly. This illustrates a residual
  failure mode for short technical strings and small curved surfaces.}
  \label{fig:qual-failure-p026}
\end{figure}

\begin{figure}[h]
  \centering
  \begin{minipage}[t]{0.32\textwidth}
    \centering
    \textbf{Ours}\\[0.25em]
    \includegraphics[width=\linewidth]{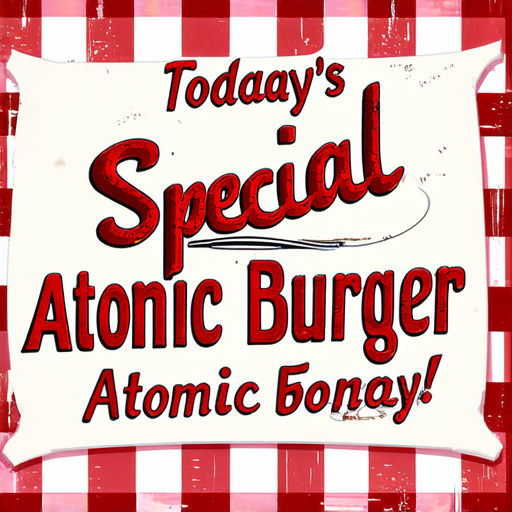}
  \end{minipage}\hfill
  \begin{minipage}[t]{0.32\textwidth}
    \centering
    \textbf{RatioNorm}\\[0.25em]
    \includegraphics[width=\linewidth]{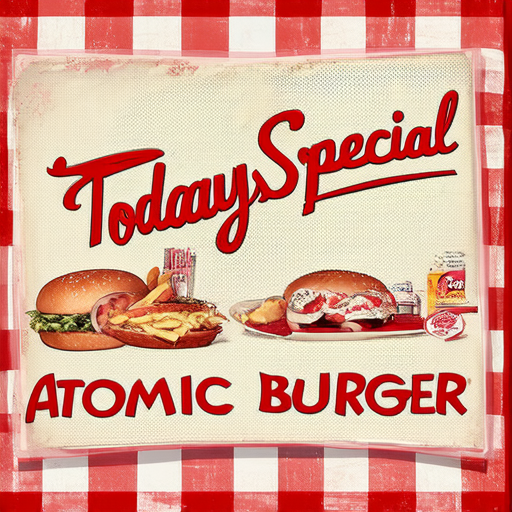}
  \end{minipage}\hfill
  \begin{minipage}[t]{0.32\textwidth}
    \centering
    \textbf{Dual}\\[0.25em]
    \includegraphics[width=\linewidth]{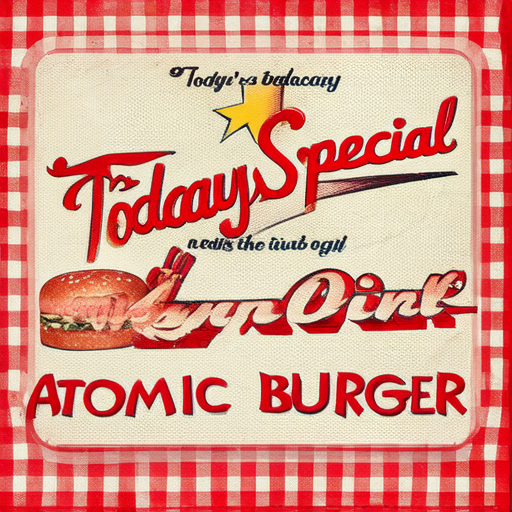}
  \end{minipage}
  \par\smallskip
  {\footnotesize\emph{Prompt.} Retro diner scene with a red and white checkered
  placemat featuring the text ``Todays Special Atomic Burger'' in bold, vintage
  font. The placemat is slightly worn, with a classic 1950s diner background.}
  \caption{Mixed held-out test comparison for prompt 159. Target text:
  \texttt{Todays Special Atomic Burger}. This is a case where empirical RatioNorm
  is better: it keeps the main phrase nearly intact, while $\lambda$-Controlled
  GRPO produces a plausible diner placemat but scrambles the target into extra
  and reordered words.}
  \label{fig:qual-mixed-p159}
\end{figure}

\begin{figure}[h]
  \centering
  \begin{minipage}[t]{0.32\textwidth}
    \centering
    \textbf{Ours}\\[0.25em]
    \includegraphics[width=\linewidth]{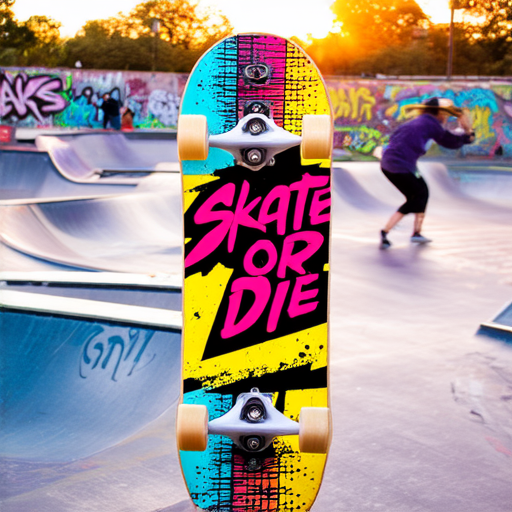}
  \end{minipage}\hfill
  \begin{minipage}[t]{0.32\textwidth}
    \centering
    \textbf{RatioNorm}\\[0.25em]
    \includegraphics[width=\linewidth]{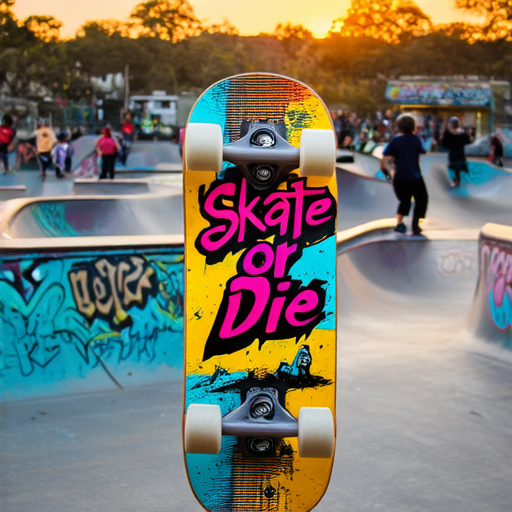}
  \end{minipage}\hfill
  \begin{minipage}[t]{0.32\textwidth}
    \centering
    \textbf{Dual}\\[0.25em]
    \includegraphics[width=\linewidth]{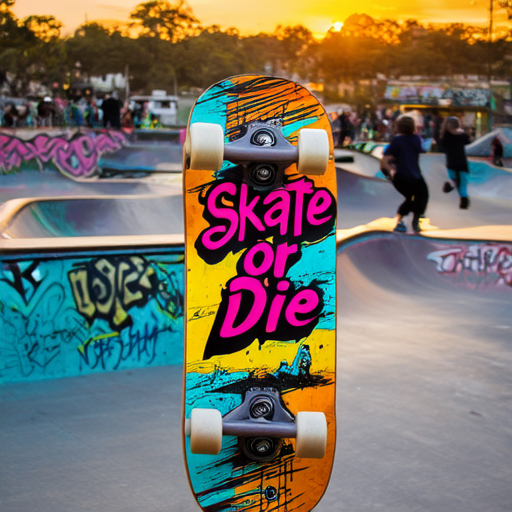}
  \end{minipage}
  \par\smallskip
  {\footnotesize\emph{Prompt.} ``Skate or Die'' slogan prominently displayed on a
  vibrant, colorful skateboard deck, set against a backdrop of a bustling urban
  skate park at sunset, with skaters in motion and graffiti-covered walls,
  capturing the rebellious spirit and dynamic energy of skate culture.}
  \caption{Mixed held-out test comparison for prompt 054. Target text:
  \texttt{Skate or Die}. All three samples are visually plausible skateboard-deck
  renderings, but this is not a clean win for the proposed method: the OCR system
  fails on the stylized version, while the RatioNorm and primal-dual samples are
  closer to the normalized transcript.}
  \label{fig:qual-mixed-p054}
\end{figure}

\begin{figure}[h]
  \centering
  \begin{minipage}[t]{0.32\textwidth}
    \centering
    \textbf{Ours}\\[0.25em]
    \includegraphics[width=\linewidth]{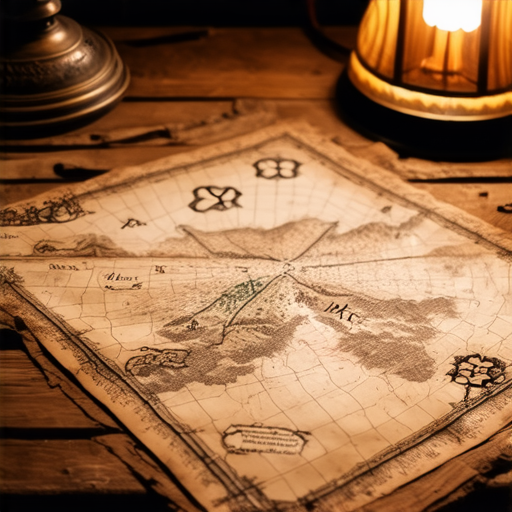}
  \end{minipage}\hfill
  \begin{minipage}[t]{0.32\textwidth}
    \centering
    \textbf{RatioNorm}\\[0.25em]
    \includegraphics[width=\linewidth]{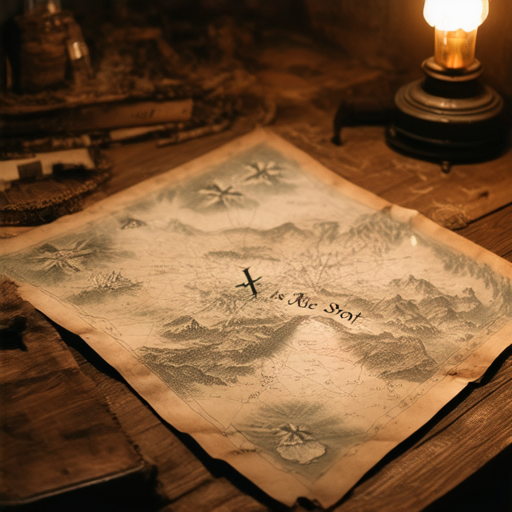}
  \end{minipage}\hfill
  \begin{minipage}[t]{0.32\textwidth}
    \centering
    \textbf{Dual}\\[0.25em]
    \includegraphics[width=\linewidth]{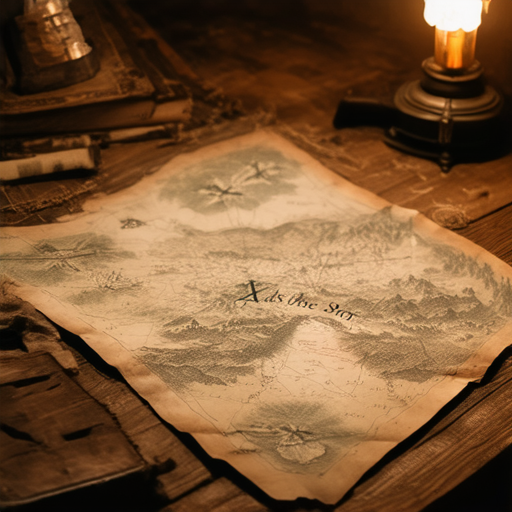}
  \end{minipage}
  \par\smallskip
  {\footnotesize\emph{Prompt.} A weathered treasure map laid out on an old wooden
  table, with ``X Marks the Spot'' clearly visible in the center, surrounded by
  intricate illustrations of mountains, forests, and a distant coastline, all
  under the warm glow of a vintage lamp.}
  \caption{Failure held-out test comparison for prompt 040. Target text:
  \texttt{X Marks the Spot}. All methods capture the map-and-lamp composition,
  but none renders the complete target phrase. This suggests that some failures
  are driven by scene/text placement difficulty rather than only by ratio
  instability.}
  \label{fig:qual-failure-p040}
\end{figure}

\clearpage
\Cref{tab:sd35-pickscore-val} reports the per-checkpoint PickScore validation
trajectory that underlies the winner selection in \cref{tab:sd35-pickscore-test}.

\begin{table}[htbp]
  \centering
  \small
  \begin{tabular}{@{}lrrrr@{}}
    \toprule
    Method
      & Ckpt.
      & Mean PickScore $\uparrow$
      & Std.
      & $\Lambda_{\rm late}$ $\downarrow$ \\
    \midrule
    RatioNorm
      & $40$ & $0.8325$ & $0.057$ & $3.8\!\times\!10^{-5}$ \\
    RatioNorm
      & $60$ & $0.8266$ & $0.057$ & $6.4\!\times\!10^{-5}$ \\
    RatioNorm
      & $80$ & $0.8126$ & $0.059$ & $6.80\!\times\!10^{-3}$ \\
    $\lambda$-Controlled GRPO
      & $40$ & $0.8365$ & $0.057$ & $4.65\!\times\!10^{-3}$ \\
    $\lambda$-Controlled GRPO
      & $60$ & $0.8373$ & $0.055$ & $1.9\!\times\!10^{-4}$ \\
    $\lambda$-Controlled GRPO
      & $80$ & $\mathbf{0.8380}$ & $0.057$ & $3.34\!\times\!10^{-3}$ \\
    \bottomrule
  \end{tabular}
  \caption{Random 256-prompt SD3.5 PickScore validation comparison. Every
  $\lambda$-Controlled GRPO checkpoint beats every RatioNorm checkpoint.
  RatioNorm peaks at step 40 and then regresses as its
  late-step $\lambda$ spend jumps past $\tau\approx 2.0\!\times\!10^{-3}$ by
  step 80, an overspend of about $3.4\tau$. The analytic method improves
  PickScore monotonically across $40\!\to\!60\!\to\!80$ and keeps
  $\Lambda_{\rm late}$ at most comparable to $\tau$.}
  \label{tab:sd35-pickscore-val}
\end{table}

\section{Qualitative held-out PickScore comparisons}
\label{sec:qualitative-gallery-pickscore}

The PickScore aggregate result in \cref{tab:sd35-pickscore-test} is illustrated
by a qualitative sweep on the same $762$-prompt held-out test complement.
Comparisons are matched by prompt index and random seed across methods:
empirical RatioNorm at step 40 (the validation-favored early stop) and
$\lambda$-Controlled GRPO at step 80 (the validation-favored late stop). Every
example below is from the disjoint
$762$-prompt held-out complement and was not used for checkpoint selection.
Captions report the prompt and the raw PickScore values verbatim, so the reader
can judge each pair. The first block shows PickScore wins for the analytic
method, the second block shows near-ties or PickScore losses, and the third
block shows cases where both methods score low on the reward model.

\subsection{PickScore hits (analytic method wins)}

\begin{figure}[h]
  \centering
  \begin{minipage}[t]{0.48\textwidth}
    \centering
    \textbf{$\lambda$-Controlled GRPO}\\[0.25em]
    \includegraphics[width=\linewidth]{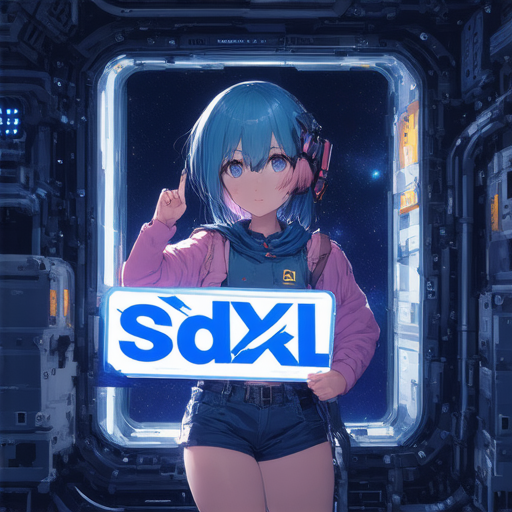}
  \end{minipage}\hfill
  \begin{minipage}[t]{0.48\textwidth}
    \centering
    \textbf{RatioNorm}\\[0.25em]
    \includegraphics[width=\linewidth]{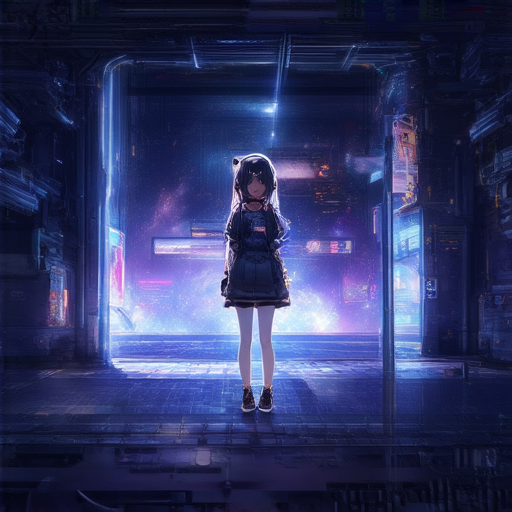}
  \end{minipage}
  \par\smallskip
  {\footnotesize\emph{Prompt.} anime girl in the space with a sign saying
  `sdxl' as text,4k}
  \caption{Held-out test prompt $274$. PickScore: Ours $=0.9137$, RatioNorm
  $=0.8235$, $\Delta=+0.0902$.}
  \label{fig:qual-pickscore-p0274}
\end{figure}

\begin{figure}[h]
  \centering
  \begin{minipage}[t]{0.48\textwidth}
    \centering
    \textbf{$\lambda$-Controlled GRPO}\\[0.25em]
    \includegraphics[width=\linewidth]{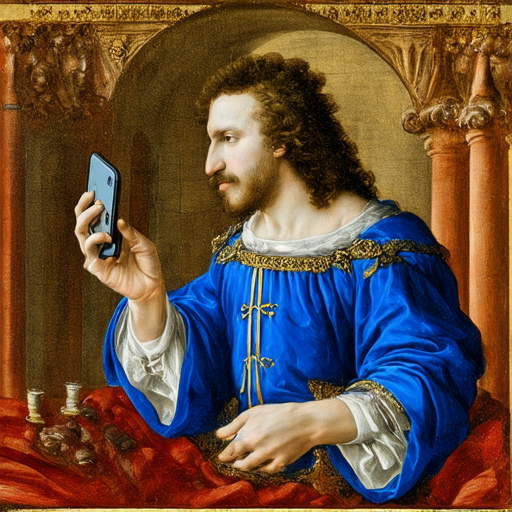}
  \end{minipage}\hfill
  \begin{minipage}[t]{0.48\textwidth}
    \centering
    \textbf{RatioNorm}\\[0.25em]
    \includegraphics[width=\linewidth]{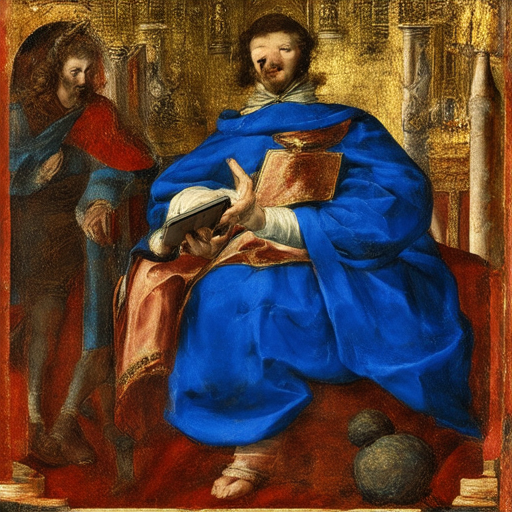}
  \end{minipage}
  \par\smallskip
  {\footnotesize\emph{Prompt.} medieval painting of a man in a blue gown using
  a cellphone}
  \caption{Held-out test prompt $521$. PickScore: Ours $=0.9014$, RatioNorm
  $=0.8180$, $\Delta=+0.0834$.}
  \label{fig:qual-pickscore-p0521}
\end{figure}

\begin{figure}[h]
  \centering
  \begin{minipage}[t]{0.48\textwidth}
    \centering
    \textbf{$\lambda$-Controlled GRPO}\\[0.25em]
    \includegraphics[width=\linewidth]{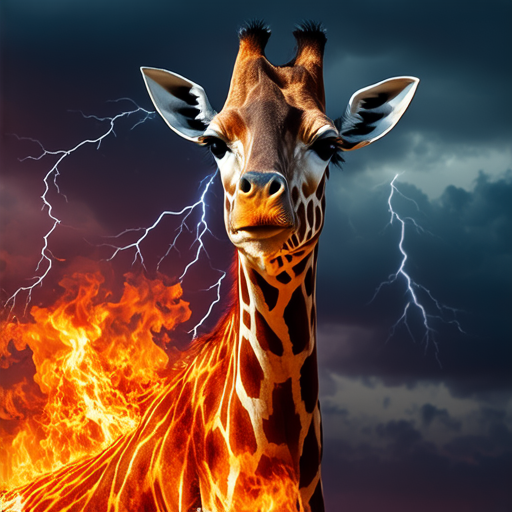}
  \end{minipage}\hfill
  \begin{minipage}[t]{0.48\textwidth}
    \centering
    \textbf{RatioNorm}\\[0.25em]
    \includegraphics[width=\linewidth]{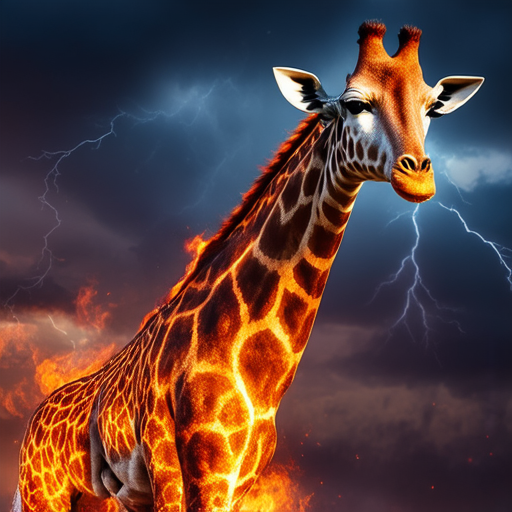}
  \end{minipage}
  \par\smallskip
  {\footnotesize\emph{Prompt.} portrait of a giraffe in a fiery thunderstorm,
  digital art, hyper detailed}
  \caption{Held-out test prompt $853$. PickScore: Ours $=0.9170$, RatioNorm
  $=0.8512$, $\Delta=+0.0658$.}
  \label{fig:qual-pickscore-p0853}
\end{figure}

\begin{figure}[h]
  \centering
  \begin{minipage}[t]{0.48\textwidth}
    \centering
    \textbf{$\lambda$-Controlled GRPO}\\[0.25em]
    \includegraphics[width=\linewidth]{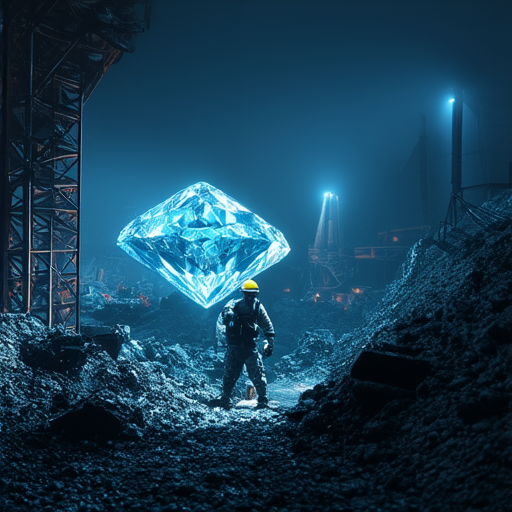}
  \end{minipage}\hfill
  \begin{minipage}[t]{0.48\textwidth}
    \centering
    \textbf{RatioNorm}\\[0.25em]
    \includegraphics[width=\linewidth]{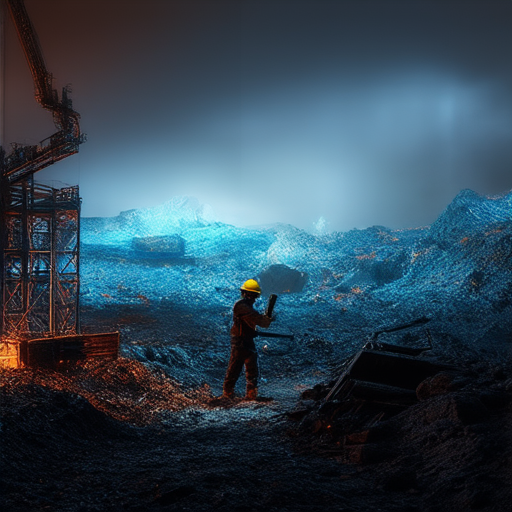}
  \end{minipage}
  \par\smallskip
  {\footnotesize\emph{Prompt.} Working in a diamond mine, Midjourney v5 style,
  insanely detailed, photorealistic, 8k, volumetric lighting}
  \caption{Held-out test prompt $223$. PickScore: Ours $=0.8623$, RatioNorm
  $=0.7737$, $\Delta=+0.0886$.}
  \label{fig:qual-pickscore-p0223}
\end{figure}

\begin{figure}[h]
  \centering
  \begin{minipage}[t]{0.48\textwidth}
    \centering
    \textbf{$\lambda$-Controlled GRPO}\\[0.25em]
    \includegraphics[width=\linewidth]{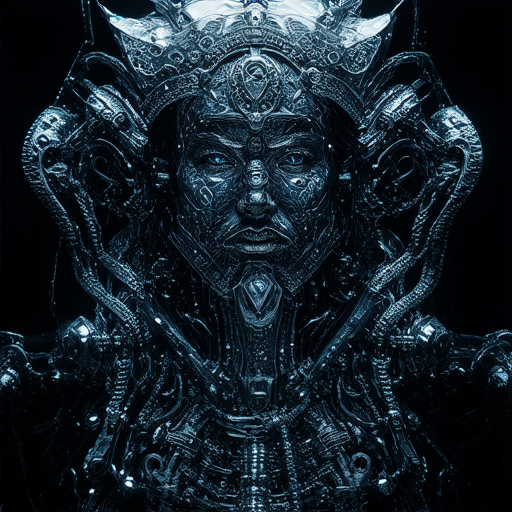}
  \end{minipage}\hfill
  \begin{minipage}[t]{0.48\textwidth}
    \centering
    \textbf{RatioNorm}\\[0.25em]
    \includegraphics[width=\linewidth]{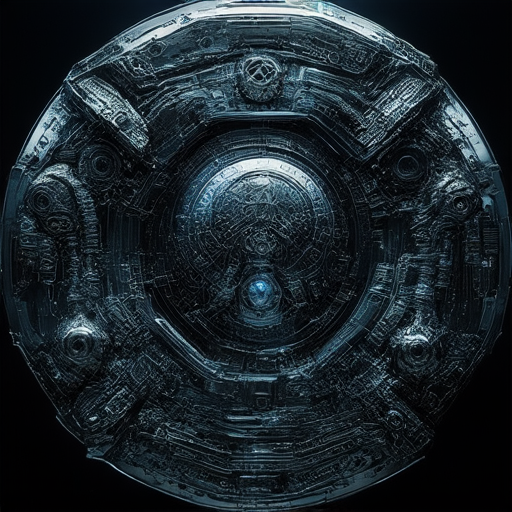}
  \end{minipage}
  \par\smallskip
  {\footnotesize\emph{Prompt.} Cyberpunk, Ancient India style, sci fi, silver
  on a black background, bas-relief, cyborgs, neon lighting, contrasting
  shadows, three-dimensional sculpture, high resolution, 8k detail, baroque,
  clear edges, technology, mechanisms}
  \caption{Held-out test prompt $682$. PickScore: Ours $=0.8182$, RatioNorm
  $=0.7329$, $\Delta=+0.0853$.}
  \label{fig:qual-pickscore-p0682}
\end{figure}

\begin{figure}[h]
  \centering
  \begin{minipage}[t]{0.48\textwidth}
    \centering
    \textbf{$\lambda$-Controlled GRPO}\\[0.25em]
    \includegraphics[width=\linewidth]{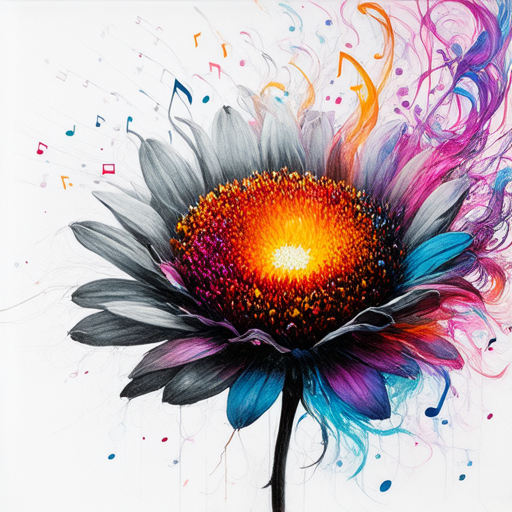}
  \end{minipage}\hfill
  \begin{minipage}[t]{0.48\textwidth}
    \centering
    \textbf{RatioNorm}\\[0.25em]
    \includegraphics[width=\linewidth]{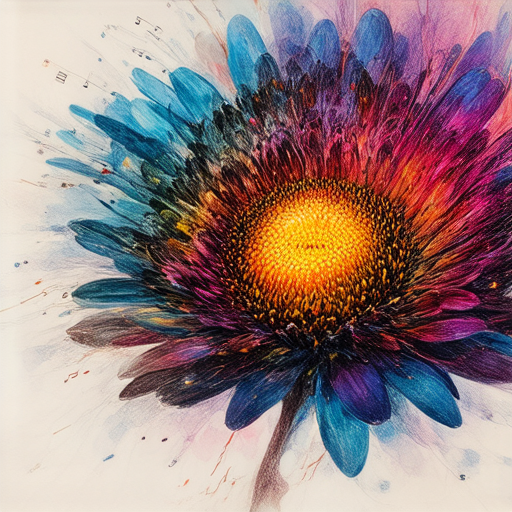}
  \end{minipage}
  \par\smallskip
  {\footnotesize\emph{Prompt.} Synesthesia, musical notes flying away from a
  music flower, charcoal hyperrealistic colourful art print by robert longo and
  alicexz}
  \caption{Held-out test prompt $626$. PickScore: Ours $=0.8725$, RatioNorm
  $=0.7951$, $\Delta=+0.0774$.}
  \label{fig:qual-pickscore-p0626}
\end{figure}

\begin{figure}[h]
  \centering
  \begin{minipage}[t]{0.48\textwidth}
    \centering
    \textbf{$\lambda$-Controlled GRPO}\\[0.25em]
    \includegraphics[width=\linewidth]{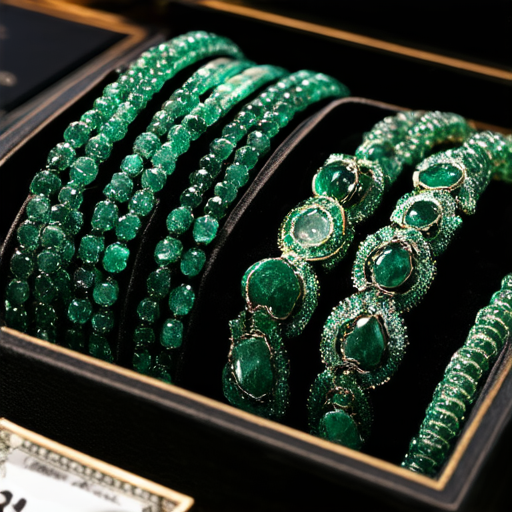}
  \end{minipage}\hfill
  \begin{minipage}[t]{0.48\textwidth}
    \centering
    \textbf{RatioNorm}\\[0.25em]
    \includegraphics[width=\linewidth]{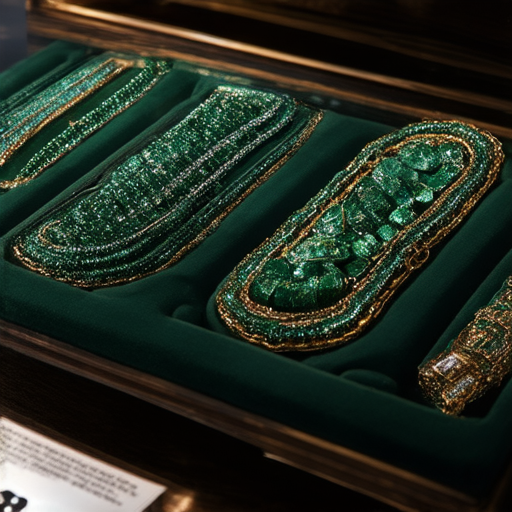}
  \end{minipage}
  \par\smallskip
  {\footnotesize\emph{Prompt.} A set of museum-quality emerald bracelets and
  beads in green in a display box at the auction, 32k, highest resolution,
  hyper realistic}
  \caption{Held-out test prompt $27$. PickScore: Ours $=0.8773$, RatioNorm
  $=0.8050$, $\Delta=+0.0723$.}
  \label{fig:qual-pickscore-p0027}
\end{figure}

\begin{figure}[h]
  \centering
  \begin{minipage}[t]{0.48\textwidth}
    \centering
    \textbf{$\lambda$-Controlled GRPO}\\[0.25em]
    \includegraphics[width=\linewidth]{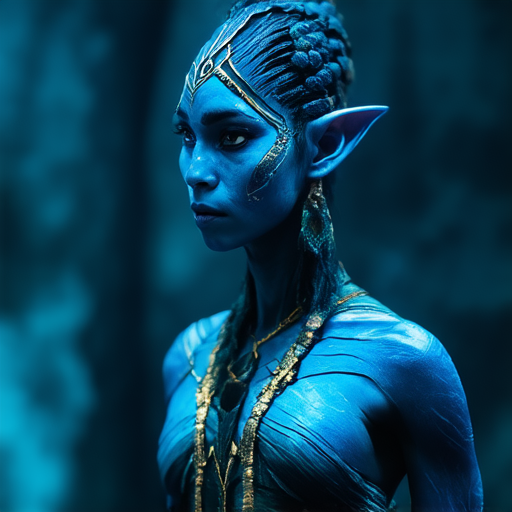}
  \end{minipage}\hfill
  \begin{minipage}[t]{0.48\textwidth}
    \centering
    \textbf{RatioNorm}\\[0.25em]
    \includegraphics[width=\linewidth]{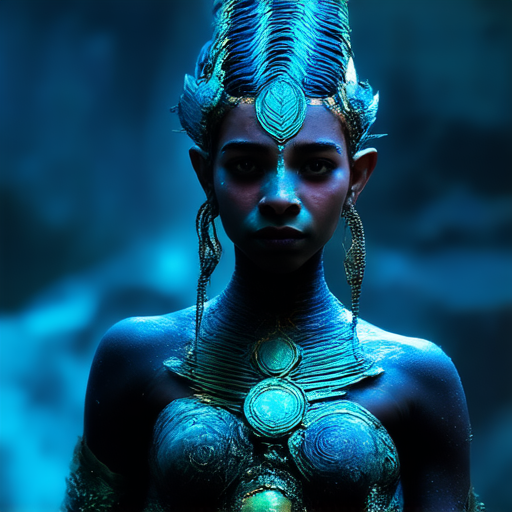}
  \end{minipage}
  \par\smallskip
  {\footnotesize\emph{Prompt.} film still of Neytiri from Avatar}
  \caption{Held-out test prompt $118$. PickScore: Ours $=0.8441$, RatioNorm
  $=0.7749$, $\Delta=+0.0692$.}
  \label{fig:qual-pickscore-p0118}
\end{figure}

\begin{figure}[h]
  \centering
  \begin{minipage}[t]{0.48\textwidth}
    \centering
    \textbf{$\lambda$-Controlled GRPO}\\[0.25em]
    \includegraphics[width=\linewidth]{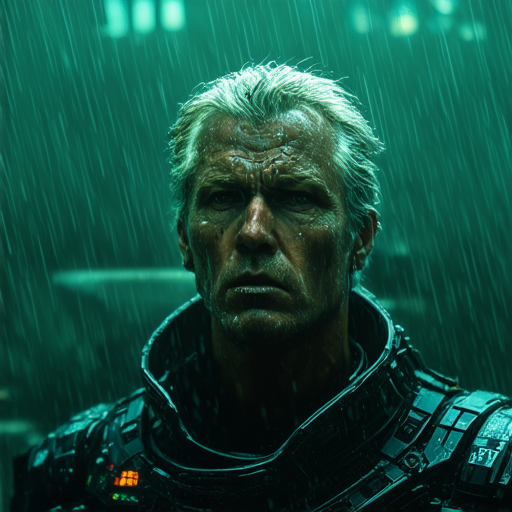}
  \end{minipage}\hfill
  \begin{minipage}[t]{0.48\textwidth}
    \centering
    \textbf{RatioNorm}\\[0.25em]
    \includegraphics[width=\linewidth]{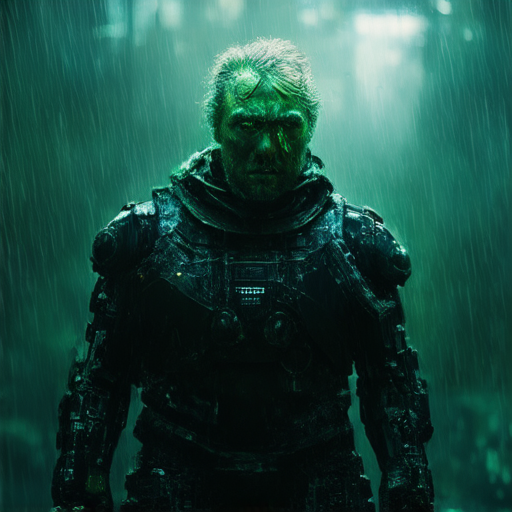}
  \end{minipage}
  \par\smallskip
  {\footnotesize\emph{Prompt.} rutger hauer from blade runner standing in the
  rain, green light, vhs quality, film grain, wet, very sad and reluctant
  expression, wearing a biomechanical suit, scifi, digital painting, artstation,
  concept art}
  \caption{Held-out test prompt $835$. PickScore: Ours $=0.8748$, RatioNorm
  $=0.8071$, $\Delta=+0.0677$.}
  \label{fig:qual-pickscore-p0835}
\end{figure}

\begin{figure}[h]
  \centering
  \begin{minipage}[t]{0.48\textwidth}
    \centering
    \textbf{$\lambda$-Controlled GRPO}\\[0.25em]
    \includegraphics[width=\linewidth]{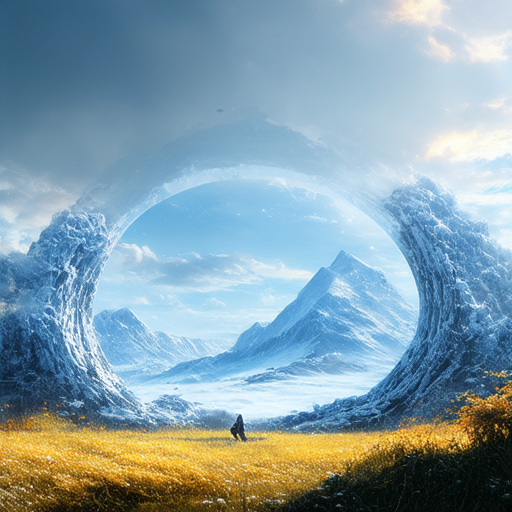}
  \end{minipage}\hfill
  \begin{minipage}[t]{0.48\textwidth}
    \centering
    \textbf{RatioNorm}\\[0.25em]
    \includegraphics[width=\linewidth]{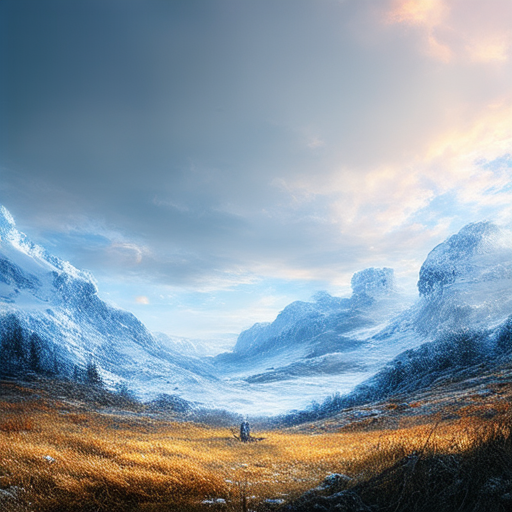}
  \end{minipage}
  \par\smallskip
  {\footnotesize\emph{Prompt.} a portal to snowy mountain, standing in a warm
  summer field}
  \caption{Held-out test prompt $1003$. PickScore: Ours $=0.9609$, RatioNorm
  $=0.8936$, $\Delta=+0.0673$.}
  \label{fig:qual-pickscore-p1003}
\end{figure}

\begin{figure}[h]
  \centering
  \begin{minipage}[t]{0.48\textwidth}
    \centering
    \textbf{$\lambda$-Controlled GRPO}\\[0.25em]
    \includegraphics[width=\linewidth]{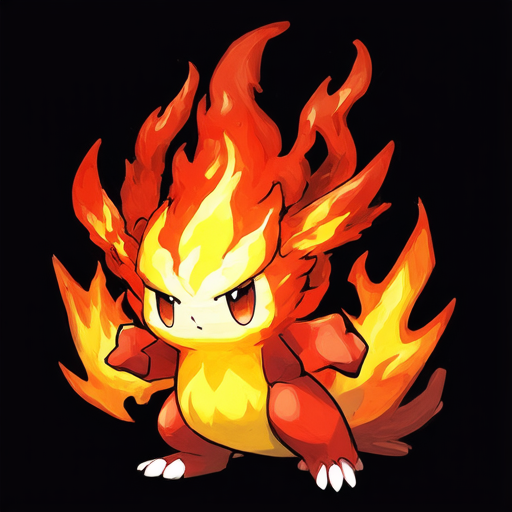}
  \end{minipage}\hfill
  \begin{minipage}[t]{0.48\textwidth}
    \centering
    \textbf{RatioNorm}\\[0.25em]
    \includegraphics[width=\linewidth]{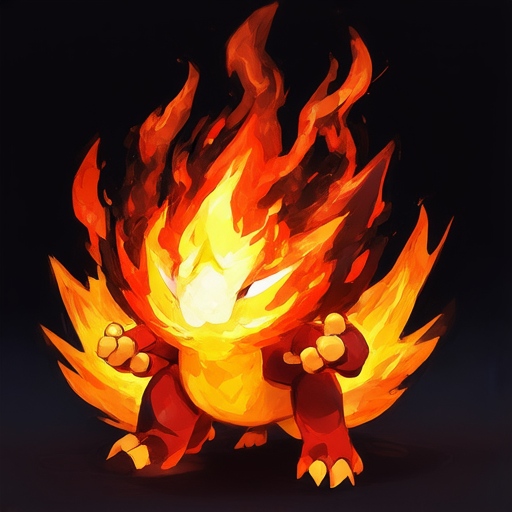}
  \end{minipage}
  \par\smallskip
  {\footnotesize\emph{Prompt.} art print of a cute fire elemental pokemon by
  league of legends. finally evolutionary stage}
  \caption{Held-out test prompt $514$. PickScore: Ours $=0.9264$, RatioNorm
  $=0.8594$, $\Delta=+0.0670$.}
  \label{fig:qual-pickscore-p0514}
\end{figure}

\begin{figure}[h]
  \centering
  \begin{minipage}[t]{0.48\textwidth}
    \centering
    \textbf{$\lambda$-Controlled GRPO}\\[0.25em]
    \includegraphics[width=\linewidth]{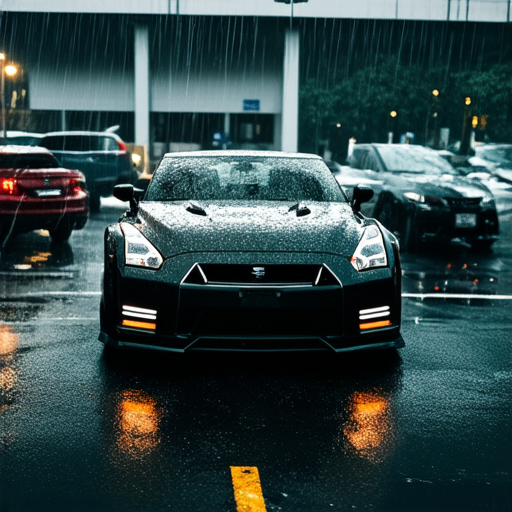}
  \end{minipage}\hfill
  \begin{minipage}[t]{0.48\textwidth}
    \centering
    \textbf{RatioNorm}\\[0.25em]
    \includegraphics[width=\linewidth]{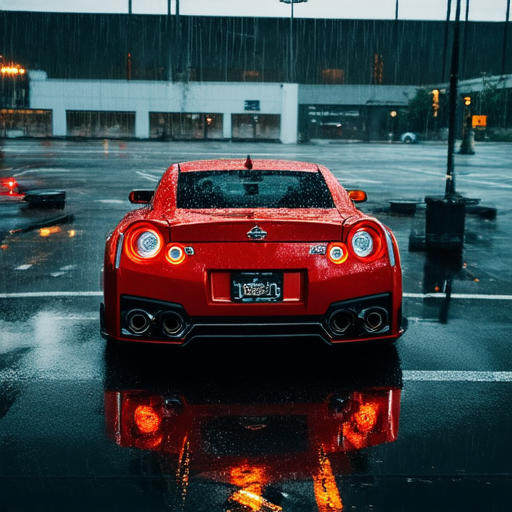}
  \end{minipage}
  \par\smallskip
  {\footnotesize\emph{Prompt.} Nissan GT-R in a parking lot, raining, film
  grain, moody}
  \caption{Held-out test prompt $149$. PickScore: Ours $=0.8943$, RatioNorm
  $=0.8288$, $\Delta=+0.0655$.}
  \label{fig:qual-pickscore-p0149}
\end{figure}

\begin{figure}[h]
  \centering
  \begin{minipage}[t]{0.48\textwidth}
    \centering
    \textbf{$\lambda$-Controlled GRPO}\\[0.25em]
    \includegraphics[width=\linewidth]{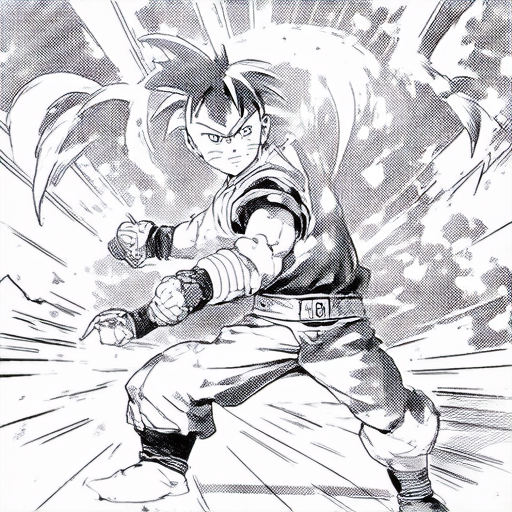}
  \end{minipage}\hfill
  \begin{minipage}[t]{0.48\textwidth}
    \centering
    \textbf{RatioNorm}\\[0.25em]
    \includegraphics[width=\linewidth]{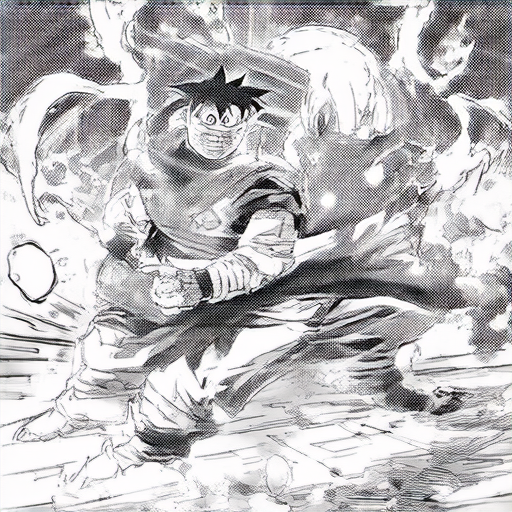}
  \end{minipage}
  \par\smallskip
  {\footnotesize\emph{Prompt.} a manga drawing of naruto fighting goku}
  \caption{Held-out test prompt $130$. PickScore: Ours $=0.8538$, RatioNorm
  $=0.7916$, $\Delta=+0.0622$.}
  \label{fig:qual-pickscore-p0130}
\end{figure}

\begin{figure}[h]
  \centering
  \begin{minipage}[t]{0.48\textwidth}
    \centering
    \textbf{$\lambda$-Controlled GRPO}\\[0.25em]
    \includegraphics[width=\linewidth]{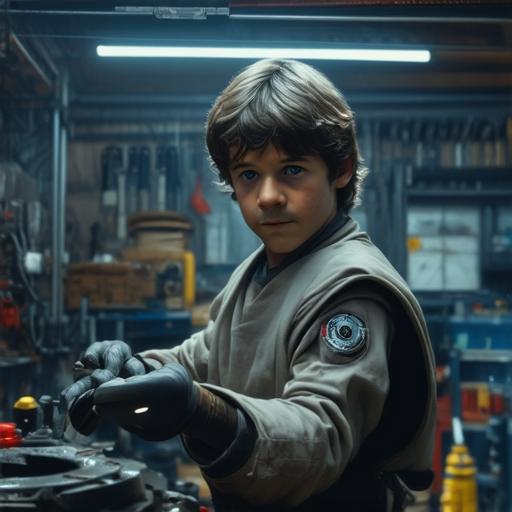}
  \end{minipage}\hfill
  \begin{minipage}[t]{0.48\textwidth}
    \centering
    \textbf{RatioNorm}\\[0.25em]
    \includegraphics[width=\linewidth]{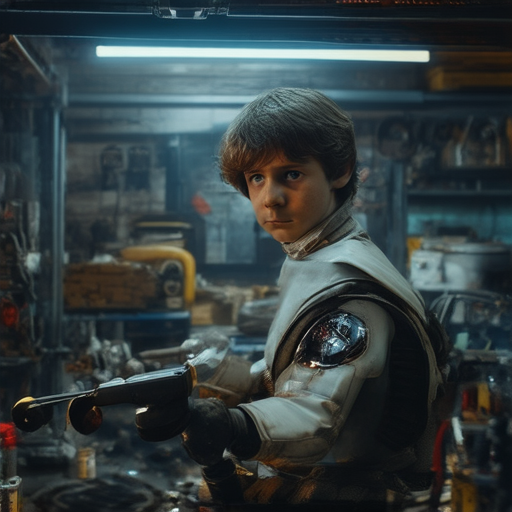}
  \end{minipage}
  \par\smallskip
  {\footnotesize\emph{Prompt.} Movie still of star wars young luke skywalker
  working as mechanic in a garage, extremely detailed, intricate, high
  resolution, hdr, trending on artstation}
  \caption{Held-out test prompt $852$. PickScore: Ours $=0.8465$, RatioNorm
  $=0.7915$, $\Delta=+0.0550$.}
  \label{fig:qual-pickscore-p0852}
\end{figure}

\begin{figure}[h]
  \centering
  \begin{minipage}[t]{0.48\textwidth}
    \centering
    \textbf{$\lambda$-Controlled GRPO}\\[0.25em]
    \includegraphics[width=\linewidth]{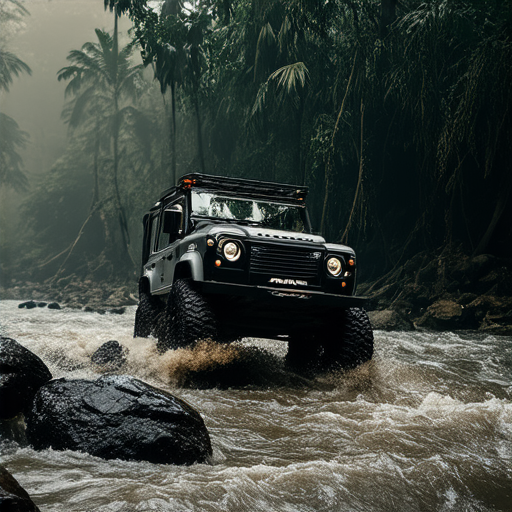}
  \end{minipage}\hfill
  \begin{minipage}[t]{0.48\textwidth}
    \centering
    \textbf{RatioNorm}\\[0.25em]
    \includegraphics[width=\linewidth]{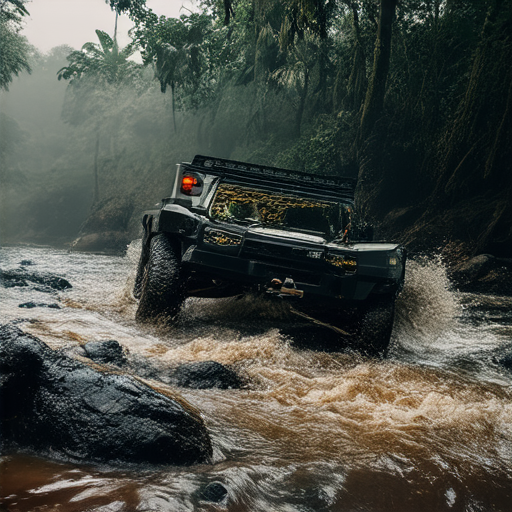}
  \end{minipage}
  \par\smallskip
  {\footnotesize\emph{Prompt.} photo of king kong lifting a landrover defender
  in the jungle river, misty mud rocks, headlights chrome detailing}
  \caption{Held-out test prompt $307$. PickScore: Ours $=0.8618$, RatioNorm
  $=0.8076$, $\Delta=+0.0542$.}
  \label{fig:qual-pickscore-p0307}
\end{figure}

\clearpage
\subsection{PickScore near-ties and RatioNorm wins}

\begin{figure}[h]
  \centering
  \begin{minipage}[t]{0.48\textwidth}
    \centering
    \textbf{$\lambda$-Controlled GRPO}\\[0.25em]
    \includegraphics[width=\linewidth]{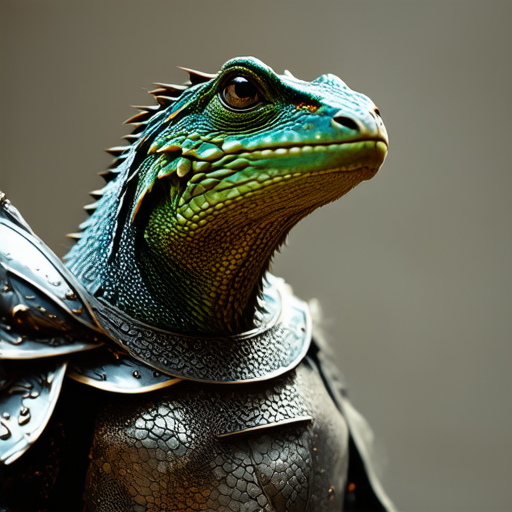}
  \end{minipage}\hfill
  \begin{minipage}[t]{0.48\textwidth}
    \centering
    \textbf{RatioNorm}\\[0.25em]
    \includegraphics[width=\linewidth]{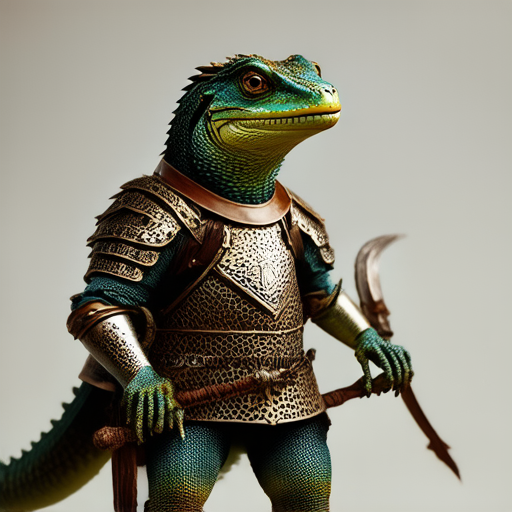}
  \end{minipage}
  \par\smallskip
  {\footnotesize\emph{Prompt.} analog style picture of a lizard dressed as a
  knight in armour}
  \caption{Held-out test prompt $942$. PickScore: Ours $=0.8598$, RatioNorm
  $=0.9018$, $\Delta=-0.0420$.}
  \label{fig:qual-pickscore-p0942}
\end{figure}

\begin{figure}[h]
  \centering
  \begin{minipage}[t]{0.48\textwidth}
    \centering
    \textbf{$\lambda$-Controlled GRPO}\\[0.25em]
    \includegraphics[width=\linewidth]{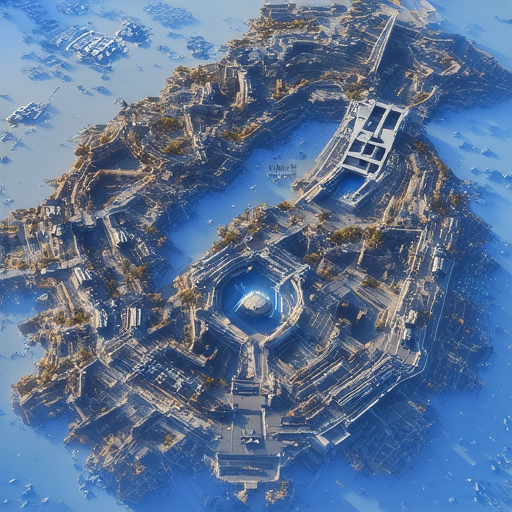}
  \end{minipage}\hfill
  \begin{minipage}[t]{0.48\textwidth}
    \centering
    \textbf{RatioNorm}\\[0.25em]
    \includegraphics[width=\linewidth]{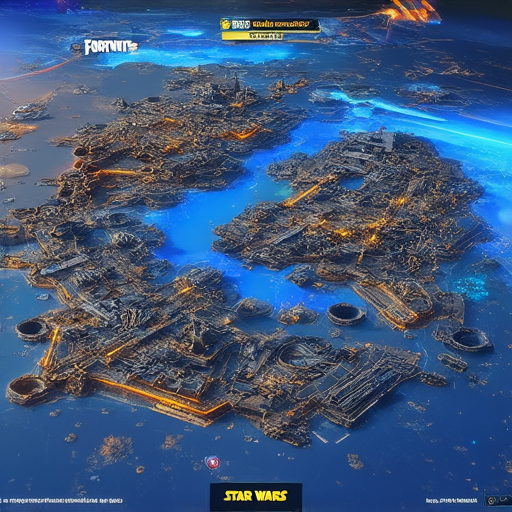}
  \end{minipage}
  \par\smallskip
  {\footnotesize\emph{Prompt.} A fortnite map inspired by Star Wars}
  \caption{Held-out test prompt $597$. PickScore: Ours $=0.8291$, RatioNorm
  $=0.8680$, $\Delta=-0.0389$.}
  \label{fig:qual-pickscore-p0597}
\end{figure}

\clearpage
\subsection{PickScore joint failures (both methods score low)}

\begin{figure}[h]
  \centering
  \begin{minipage}[t]{0.48\textwidth}
    \centering
    \textbf{$\lambda$-Controlled GRPO}\\[0.25em]
    \includegraphics[width=\linewidth]{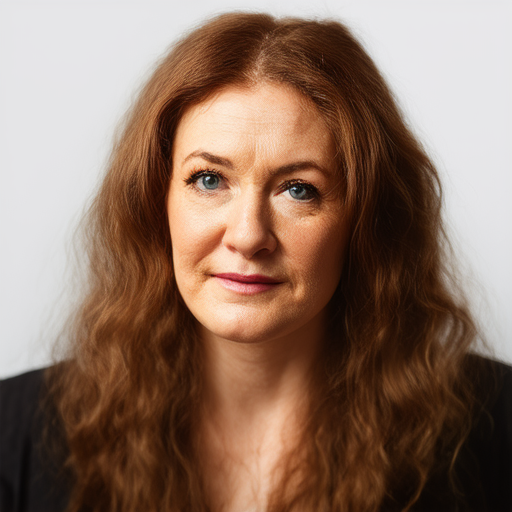}
  \end{minipage}\hfill
  \begin{minipage}[t]{0.48\textwidth}
    \centering
    \textbf{RatioNorm}\\[0.25em]
    \includegraphics[width=\linewidth]{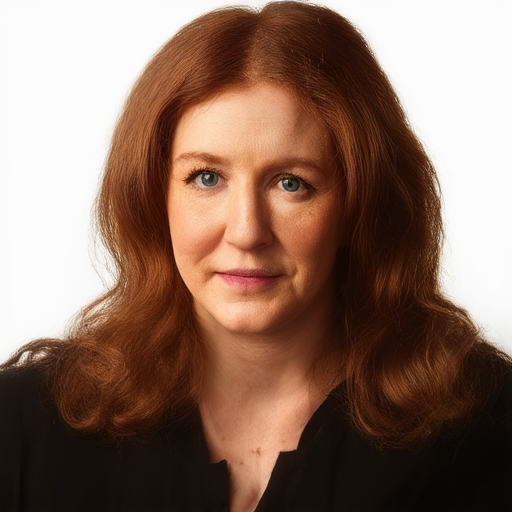}
  \end{minipage}
  \par\smallskip
  {\footnotesize\emph{Prompt.} Young Victoria coren-mitchell, insanely detailed,
  photorealistic, 8k, ultra high resolution, volumetric lighting, taken with
  canon eos,}
  \caption{Held-out test prompt $391$. PickScore: Ours $=0.6642$, RatioNorm
  $=0.6773$, $\Delta=-0.0131$. Both methods land well below the aggregate mean;
  this prompt class (photorealistic named-individual likeness) is a known hard
  case for SD3.5-LoRA regardless of the ratio controller.}
  \label{fig:qual-pickscore-p0391}
\end{figure}

\begin{figure}[h]
  \centering
  \begin{minipage}[t]{0.48\textwidth}
    \centering
    \textbf{$\lambda$-Controlled GRPO}\\[0.25em]
    \includegraphics[width=\linewidth]{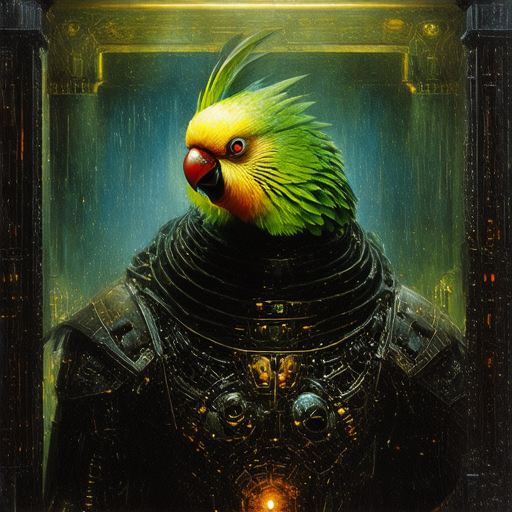}
  \end{minipage}\hfill
  \begin{minipage}[t]{0.48\textwidth}
    \centering
    \textbf{RatioNorm}\\[0.25em]
    \includegraphics[width=\linewidth]{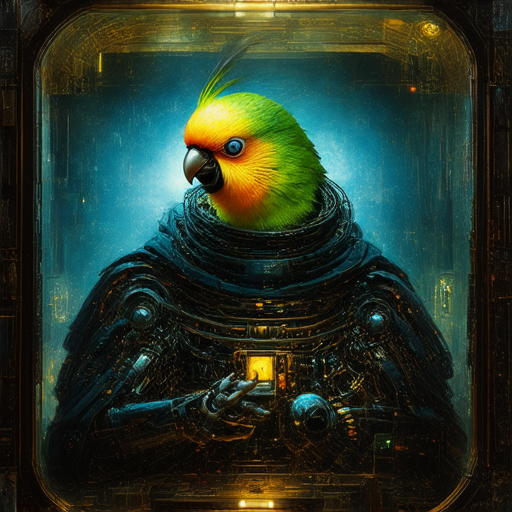}
  \end{minipage}
  \par\smallskip
  {\footnotesize\emph{Prompt.} an epic view of a demonic Rose-ringed parakeet
  cyborg inside an ironmaiden robot, wearing a noble robe, large view, a
  surrealist painting, aralan bean and Philippe Druillet, hiromu arakawa,
  volumetric lighting, detailed shadows}
  \caption{Held-out test prompt $10$. PickScore: Ours $=0.6661$, RatioNorm
  $=0.6675$, $\Delta=-0.0015$. Heavy compositional nesting with named artist
  styles; both methods score near the bottom of the distribution.}
  \label{fig:qual-pickscore-p0010}
\end{figure}

\clearpage

\clearpage
\section{Qualitative held-out GenEval comparisons}
\label{sec:qualitative-gallery-geneval}

This appendix illustrates the compositional GenEval regime with matched held-out
image comparisons between the analytic $\lambda$-Controlled GRPO method and
empirical RatioNorm on SD3.5. Each pair uses the same prompt and the same random
seed at each method's validation-selected checkpoint, and every example is drawn
from the disjoint held-out test complement that was not used for checkpoint
selection. Captions report the prompt, its GenEval category tag, and the raw
per-prompt GenEval and strict-accuracy scores, so the reader can judge each pair.
The first block shows GenEval wins for the analytic method, the second shows joint
successes where both methods pass, and the third shows cases where RatioNorm wins.

\subsection{GenEval hits (analytic method wins)}

\begin{figure}[h]
  \centering
  \begin{minipage}[t]{0.48\textwidth}
    \centering
    \textbf{$\lambda$-Controlled GRPO}\\[0.25em]
    \includegraphics[width=\linewidth]{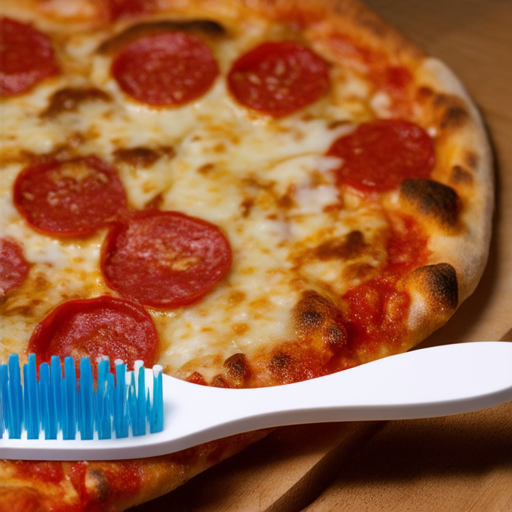}
  \end{minipage}\hfill
  \begin{minipage}[t]{0.48\textwidth}
    \centering
    \textbf{RatioNorm}\\[0.25em]
    \includegraphics[width=\linewidth]{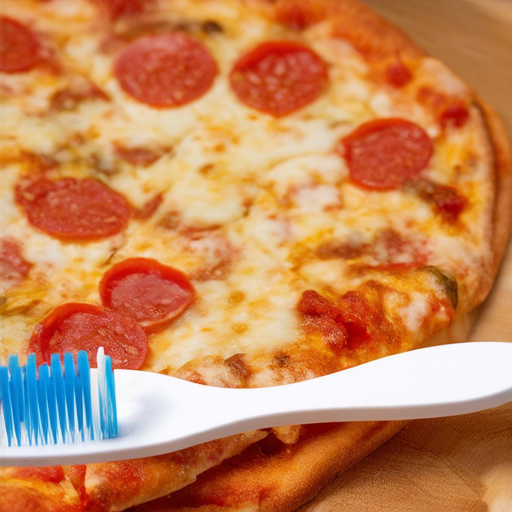}
  \end{minipage}
  \par\smallskip
  {\footnotesize\emph{Prompt.} a photo of a toothbrush below a pizza
  \quad(\texttt{position})}
  \caption{Held-out test prompt $117$. GenEval: Ours $=1.00$, RatioNorm
  $=0.00$. Strict accuracy: Ours $=1$, RatioNorm $=0$.}
  \label{fig:qual-geneval-p117}
\end{figure}

\begin{figure}[h]
  \centering
  \begin{minipage}[t]{0.48\textwidth}
    \centering
    \textbf{$\lambda$-Controlled GRPO}\\[0.25em]
    \includegraphics[width=\linewidth]{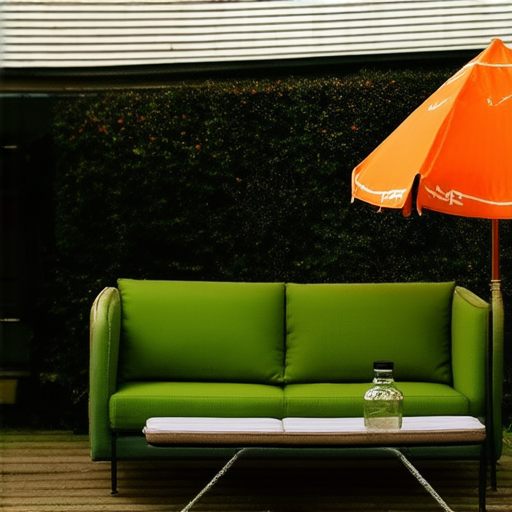}
  \end{minipage}\hfill
  \begin{minipage}[t]{0.48\textwidth}
    \centering
    \textbf{RatioNorm}\\[0.25em]
    \includegraphics[width=\linewidth]{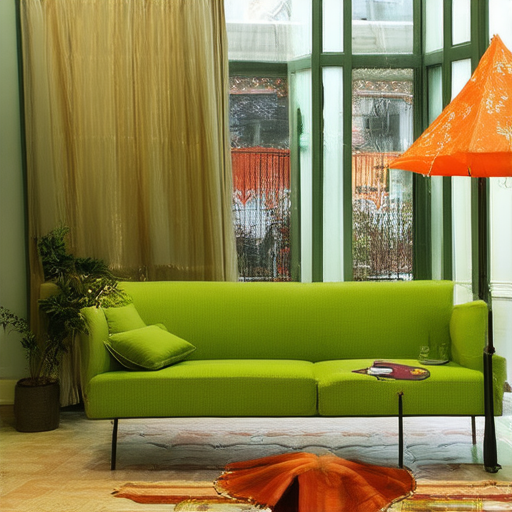}
  \end{minipage}
  \par\smallskip
  {\footnotesize\emph{Prompt.} a photo of a green couch and an orange
  umbrella \quad(\texttt{color\_attr})}
  \caption{Held-out test prompt $124$. GenEval: Ours $=1.00$, RatioNorm
  $=0.00$. Strict accuracy: Ours $=1$, RatioNorm $=0$.}
  \label{fig:qual-geneval-p124}
\end{figure}

\begin{figure}[h]
  \centering
  \begin{minipage}[t]{0.48\textwidth}
    \centering
    \textbf{$\lambda$-Controlled GRPO}\\[0.25em]
    \includegraphics[width=\linewidth]{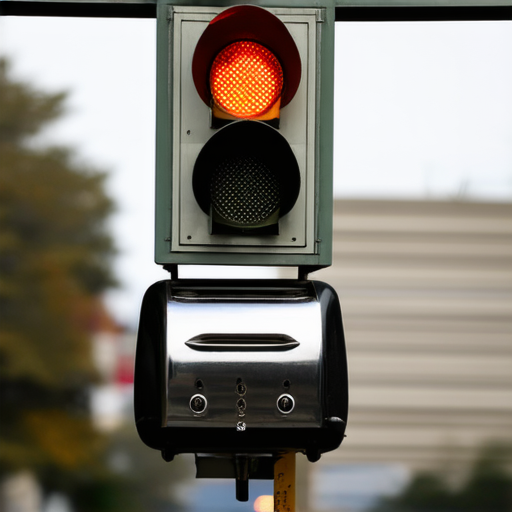}
  \end{minipage}\hfill
  \begin{minipage}[t]{0.48\textwidth}
    \centering
    \textbf{RatioNorm}\\[0.25em]
    \includegraphics[width=\linewidth]{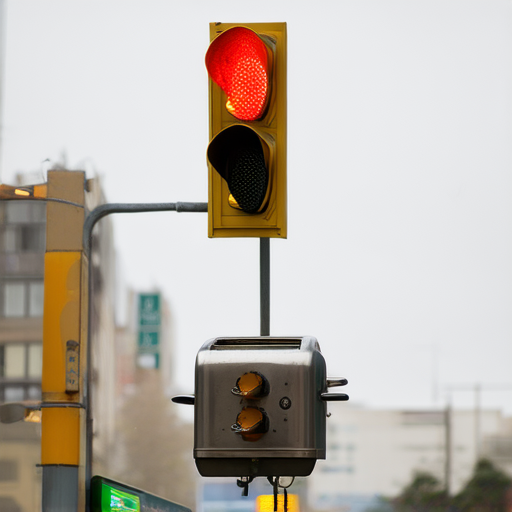}
  \end{minipage}
  \par\smallskip
  {\footnotesize\emph{Prompt.} a photo of a toaster below a traffic light
  \quad(\texttt{position})}
  \caption{Held-out test prompt $234$. GenEval: Ours $=1.00$, RatioNorm
  $=-0.33$. Strict accuracy: Ours $=1$, RatioNorm $=0$.}
  \label{fig:qual-geneval-p234}
\end{figure}

\begin{figure}[h]
  \centering
  \begin{minipage}[t]{0.48\textwidth}
    \centering
    \textbf{$\lambda$-Controlled GRPO}\\[0.25em]
    \includegraphics[width=\linewidth]{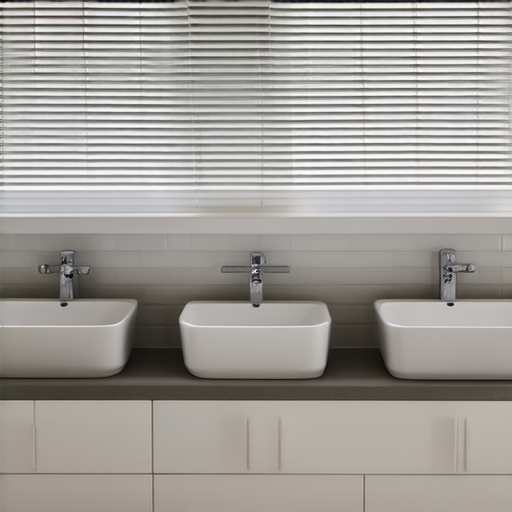}
  \end{minipage}\hfill
  \begin{minipage}[t]{0.48\textwidth}
    \centering
    \textbf{RatioNorm}\\[0.25em]
    \includegraphics[width=\linewidth]{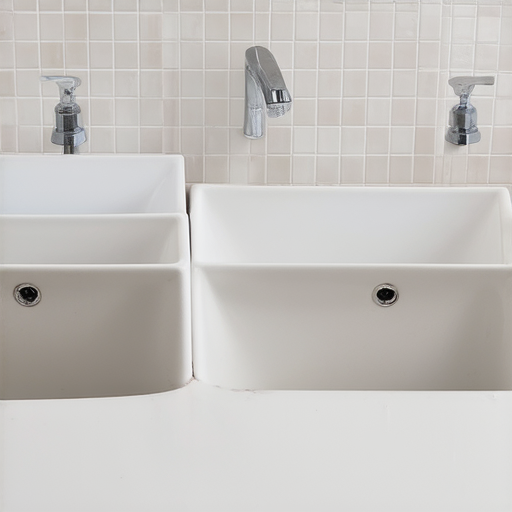}
  \end{minipage}
  \par\smallskip
  {\footnotesize\emph{Prompt.} a photo of three sinks
  \quad(\texttt{counting})}
  \caption{Held-out test prompt $423$. GenEval: Ours $=1.00$, RatioNorm
  $=0.33$. Strict accuracy: Ours $=1$, RatioNorm $=0$.}
  \label{fig:qual-geneval-p423}
\end{figure}

\clearpage
\subsection{GenEval joint successes (both methods pass)}

\begin{figure}[h]
  \centering
  \begin{minipage}[t]{0.48\textwidth}
    \centering
    \textbf{$\lambda$-Controlled GRPO}\\[0.25em]
    \includegraphics[width=\linewidth]{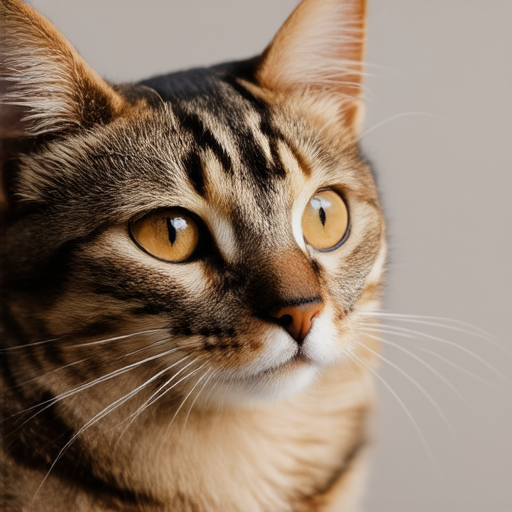}
  \end{minipage}\hfill
  \begin{minipage}[t]{0.48\textwidth}
    \centering
    \textbf{RatioNorm}\\[0.25em]
    \includegraphics[width=\linewidth]{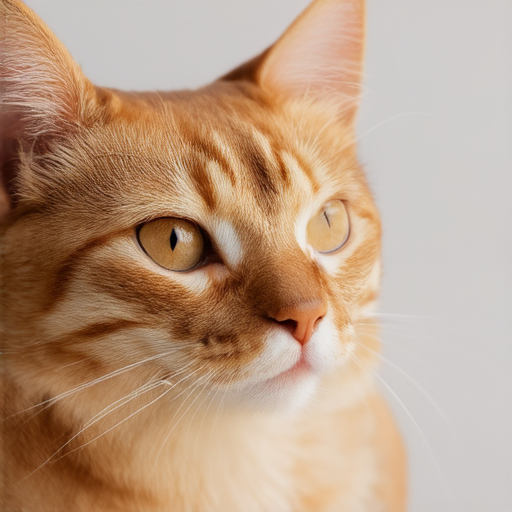}
  \end{minipage}
  \par\smallskip
  {\footnotesize\emph{Prompt.} a photo of a cat
  \quad(\texttt{single\_object})}
  \caption{Held-out test prompt $5$. GenEval: Ours $=1.00$, RatioNorm
  $=1.00$. Both methods produce a strict pass; included to show that the
  matched-seed protocol is not exploiting failure modes of one side.}
  \label{fig:qual-geneval-p5}
\end{figure}

\begin{figure}[h]
  \centering
  \begin{minipage}[t]{0.48\textwidth}
    \centering
    \textbf{$\lambda$-Controlled GRPO}\\[0.25em]
    \includegraphics[width=\linewidth]{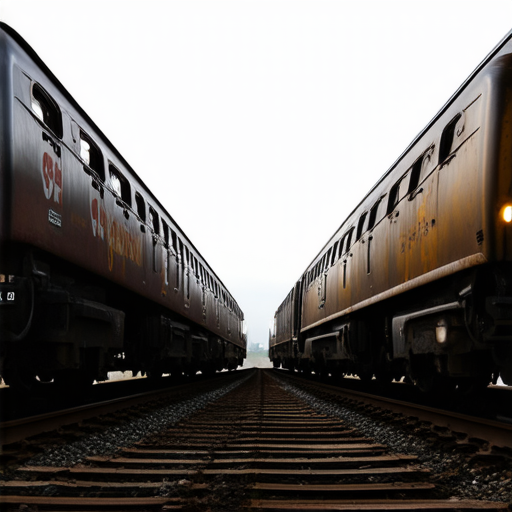}
  \end{minipage}\hfill
  \begin{minipage}[t]{0.48\textwidth}
    \centering
    \textbf{RatioNorm}\\[0.25em]
    \includegraphics[width=\linewidth]{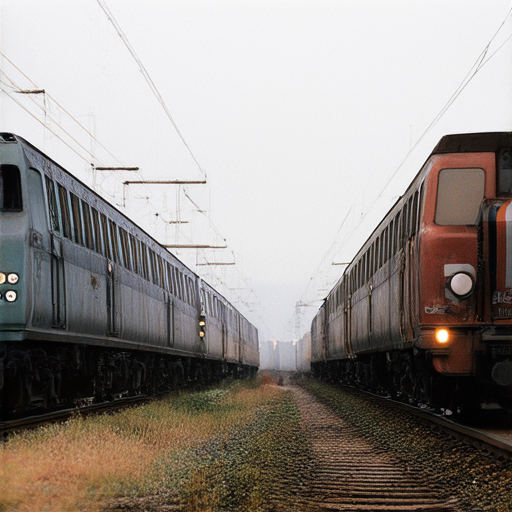}
  \end{minipage}
  \par\smallskip
  {\footnotesize\emph{Prompt.} a photo of two trains
  \quad(\texttt{counting})}
  \caption{Held-out test prompt $8$. GenEval: Ours $=1.00$, RatioNorm
  $=1.00$.}
  \label{fig:qual-geneval-p8}
\end{figure}

\begin{figure}[h]
  \centering
  \begin{minipage}[t]{0.48\textwidth}
    \centering
    \textbf{$\lambda$-Controlled GRPO}\\[0.25em]
    \includegraphics[width=\linewidth]{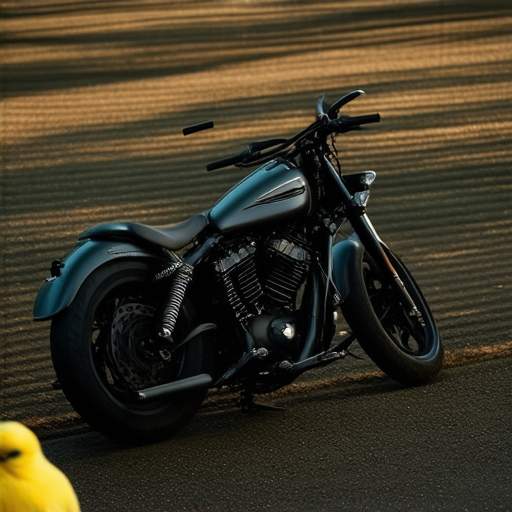}
  \end{minipage}\hfill
  \begin{minipage}[t]{0.48\textwidth}
    \centering
    \textbf{RatioNorm}\\[0.25em]
    \includegraphics[width=\linewidth]{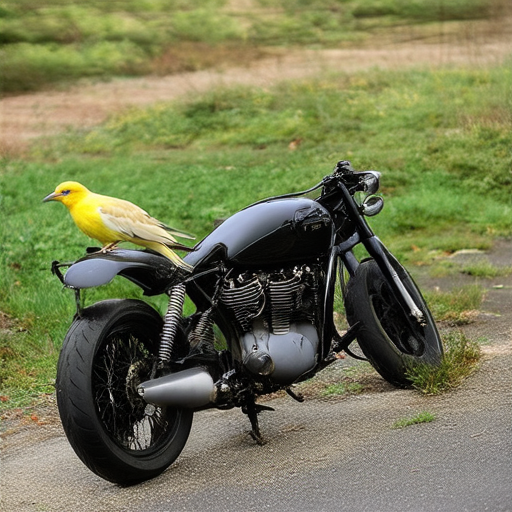}
  \end{minipage}
  \par\smallskip
  {\footnotesize\emph{Prompt.} a photo of a yellow bird and a black
  motorcycle \quad(\texttt{color\_attr})}
  \caption{Held-out test prompt $9$. GenEval: Ours $=1.00$, RatioNorm
  $=1.00$.}
  \label{fig:qual-geneval-p9}
\end{figure}

\begin{figure}[h]
  \centering
  \begin{minipage}[t]{0.48\textwidth}
    \centering
    \textbf{$\lambda$-Controlled GRPO}\\[0.25em]
    \includegraphics[width=\linewidth]{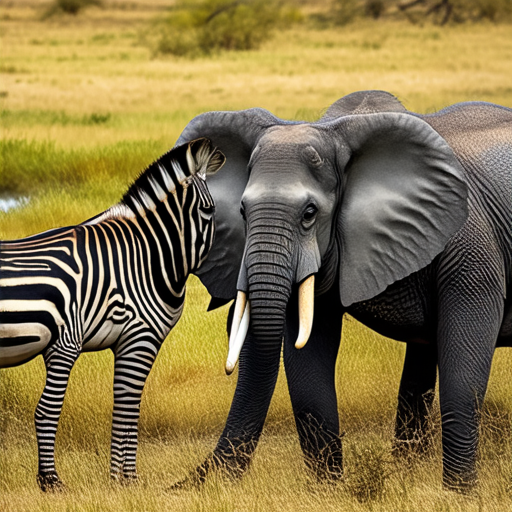}
  \end{minipage}\hfill
  \begin{minipage}[t]{0.48\textwidth}
    \centering
    \textbf{RatioNorm}\\[0.25em]
    \includegraphics[width=\linewidth]{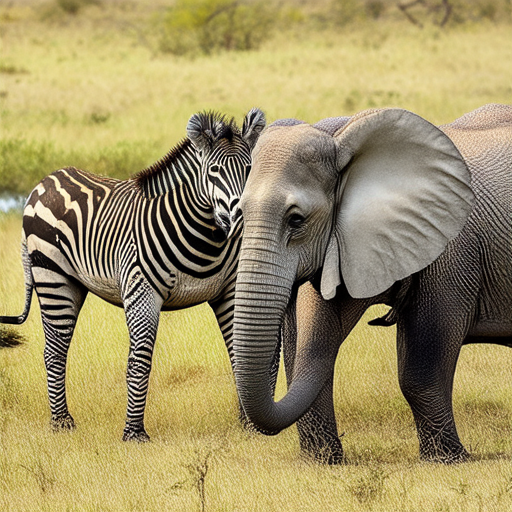}
  \end{minipage}
  \par\smallskip
  {\footnotesize\emph{Prompt.} a photo of a zebra left of an elephant
  \quad(\texttt{position})}
  \caption{Held-out test prompt $25$. GenEval: Ours $=1.00$, RatioNorm
  $=1.00$. A rare \texttt{position} joint success on a category that is
  otherwise near-floor for both methods.}
  \label{fig:qual-geneval-p25}
\end{figure}

\clearpage
\subsection{GenEval losses (RatioNorm wins)}

\begin{figure}[h]
  \centering
  \begin{minipage}[t]{0.48\textwidth}
    \centering
    \textbf{$\lambda$-Controlled GRPO}\\[0.25em]
    \includegraphics[width=\linewidth]{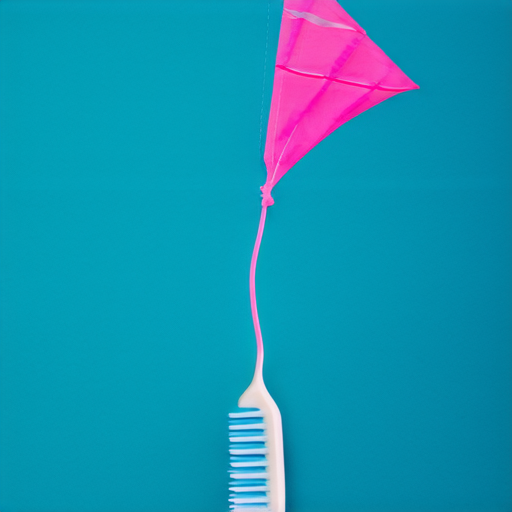}
  \end{minipage}\hfill
  \begin{minipage}[t]{0.48\textwidth}
    \centering
    \textbf{RatioNorm}\\[0.25em]
    \includegraphics[width=\linewidth]{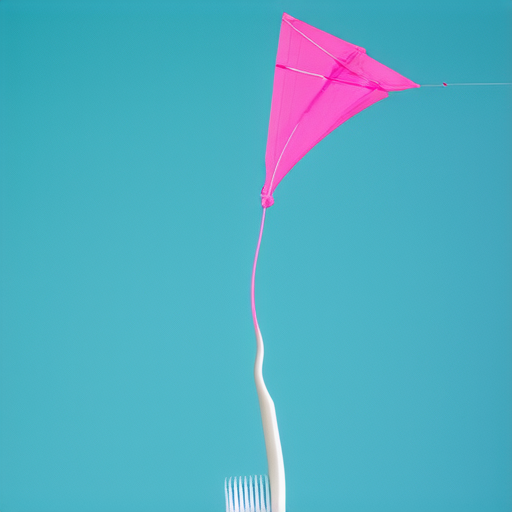}
  \end{minipage}
  \par\smallskip
  {\footnotesize\emph{Prompt.} a photo of a kite above a toothbrush
  \quad(\texttt{position})}
  \caption{Held-out test prompt $39$. GenEval: Ours $=0.33$, RatioNorm
  $=1.00$. Strict accuracy: Ours $=0$, RatioNorm $=1$.}
  \label{fig:qual-geneval-p39}
\end{figure}

\begin{figure}[h]
  \centering
  \begin{minipage}[t]{0.48\textwidth}
    \centering
    \textbf{$\lambda$-Controlled GRPO}\\[0.25em]
    \includegraphics[width=\linewidth]{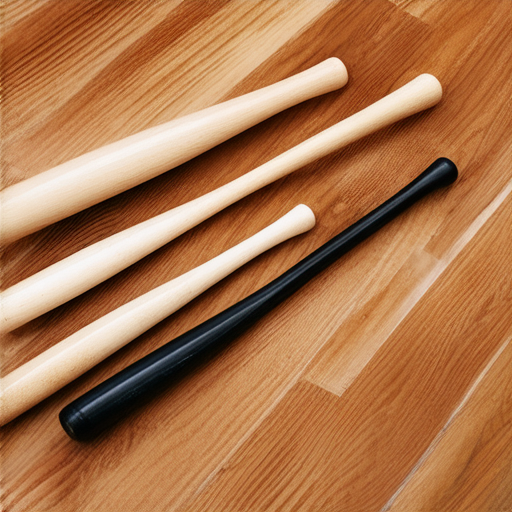}
  \end{minipage}\hfill
  \begin{minipage}[t]{0.48\textwidth}
    \centering
    \textbf{RatioNorm}\\[0.25em]
    \includegraphics[width=\linewidth]{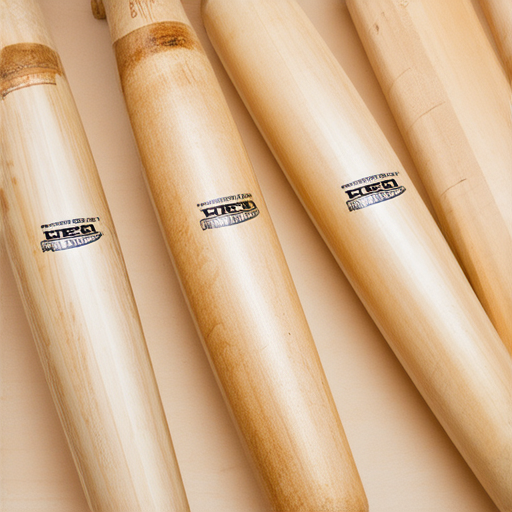}
  \end{minipage}
  \par\smallskip
  {\footnotesize\emph{Prompt.} a photo of three baseball bats
  \quad(\texttt{counting})}
  \caption{Held-out test prompt $62$. GenEval: Ours $=0.00$, RatioNorm
  $=1.00$. Strict accuracy: Ours $=0$, RatioNorm $=1$.}
  \label{fig:qual-geneval-p62}
\end{figure}

\begin{figure}[h]
  \centering
  \begin{minipage}[t]{0.48\textwidth}
    \centering
    \textbf{$\lambda$-Controlled GRPO}\\[0.25em]
    \includegraphics[width=\linewidth]{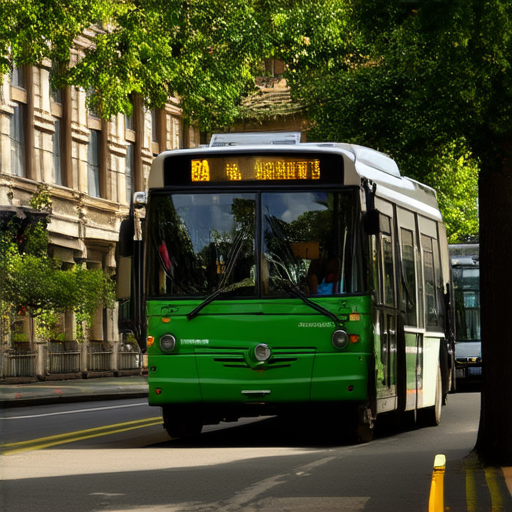}
  \end{minipage}\hfill
  \begin{minipage}[t]{0.48\textwidth}
    \centering
    \textbf{RatioNorm}\\[0.25em]
    \includegraphics[width=\linewidth]{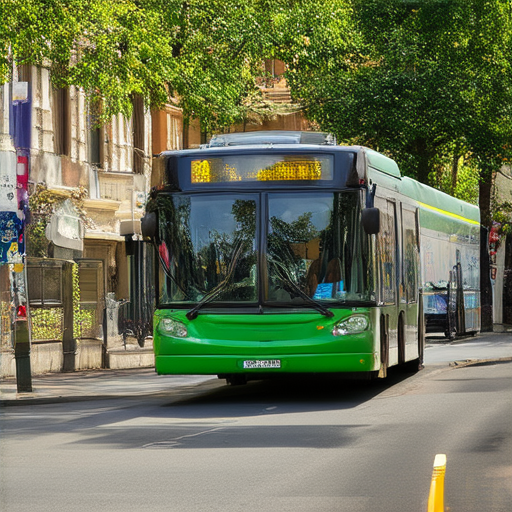}
  \end{minipage}
  \par\smallskip
  {\footnotesize\emph{Prompt.} a photo of a green bus
  \quad(\texttt{colors})}
  \caption{Held-out test prompt $105$. GenEval: Ours $=0.00$, RatioNorm
  $=1.00$. Strict accuracy: Ours $=0$, RatioNorm $=1$. The bus is
  recognized in both images, but only the RatioNorm rendering is judged
  strictly green by the GenEval color classifier.}
  \label{fig:qual-geneval-p105}
\end{figure}

\begin{figure}[h]
  \centering
  \begin{minipage}[t]{0.48\textwidth}
    \centering
    \textbf{$\lambda$-Controlled GRPO}\\[0.25em]
    \includegraphics[width=\linewidth]{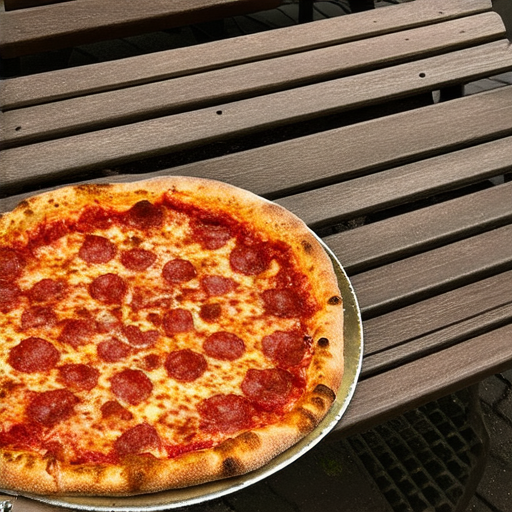}
  \end{minipage}\hfill
  \begin{minipage}[t]{0.48\textwidth}
    \centering
    \textbf{RatioNorm}\\[0.25em]
    \includegraphics[width=\linewidth]{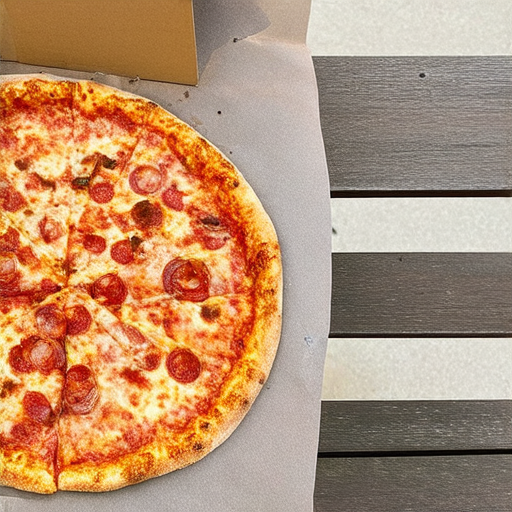}
  \end{minipage}
  \par\smallskip
  {\footnotesize\emph{Prompt.} a photo of a pizza and a bench
  \quad(\texttt{two\_object})}
  \caption{Held-out test prompt $136$. GenEval: Ours $=0.00$, RatioNorm
  $=1.00$. Strict accuracy: Ours $=0$, RatioNorm $=1$.}
  \label{fig:qual-geneval-p136}
\end{figure}

\end{document}